%% file: main.tex
\documentclass[opre]{informs3}

\makeatletter
\providecommand{\proofname}{Proof}
\def\proof{\@ifnextchar[{\@orbproof}{\@orbproof[\proofname]}}
\def\@orbproof[#1]{\Trivlist
  \item[\hspace*{1em}\hskip\labelsep{\itshape #1.\enskip}]\ignorespaces}

\def\appendix{\par\setcounter{section}{0}\setcounter{subsection}{0}\gdef\thesection{\@Alph\c@section}}
\makeatother

\usepackage{amsmath,amssymb,amsfonts,bm,mathtools}
\usepackage{xcolor}
\usepackage{enumitem}
\usepackage{microtype}
\usepackage{multirow}
\usepackage{booktabs}
\usepackage{makecell}
\usepackage{algorithm}
\usepackage{algorithmic}
\usepackage{pifont}          \usepackage{natbib}
\usepackage[colorlinks=true,linkcolor=blue,citecolor=blue,urlcolor=blue]{hyperref}

\bibpunct[, ]{(}{)}{,}{a}{}{,}\def\bibfont{\small}
\MANUSCRIPTNO{}
\allowdisplaybreaks[4]

\theoremstyle{TH}
\newtheorem{theorem}{Theorem}[section]
\newtheorem{lemma}[theorem]{Lemma}
\newtheorem{corollary}[theorem]{Corollary}
\newtheorem{proposition}[theorem]{Proposition}

\newtheorem{definition}[theorem]{Definition}

\newtheorem{remark}[theorem]{Remark}
\numberwithin{equation}{section}

\newcommand{\cmark}{\ding{51}}
\newcommand{\xmark}{\ding{55}}

\definecolor{tianqin}{RGB}{0, 139, 139}
\definecolor{luolan}{RGB}{199, 21, 133}
\definecolor{stanford}{HTML}{81221c}

\let\be\relax
\let\bi\relax
\let\argmin\relax

\newcommand{\bc}{\bm{c}}

\newcommand{\be}{\bm{e}}

\newcommand{\bg}{\bm{g}}

\newcommand{\bi}{\bm{i}}

\newcommand{\bq}{\bm{q}}
\newcommand{\br}{\bm{r}}

\newcommand{\bu}{\bm{u}}
\newcommand{\bv}{\bm{v}}
\newcommand{\bw}{\bm{w}}
\newcommand{\bx}{\bm{x}}

\newcommand{\bz}{\bm{z}}

\newcommand{\Ab}{\mathbf{A}}

\newcommand{\bA}{\bm{A}}
\newcommand{\bB}{\bm{B}}
\newcommand{\bC}{\bm{C}}
\newcommand{\bD}{\bm{D}}
\newcommand{\bE}{\bm{E}}
\newcommand{\bF}{\bm{F}}
\newcommand{\bG}{\bm{G}}
\newcommand{\bH}{\bm{H}}
\newcommand{\bI}{\bm{I}}

\newcommand{\bL}{\bm{L}}
\newcommand{\bM}{\bm{M}}

\newcommand{\bP}{\bm{P}}
\newcommand{\bQ}{\bm{Q}}
\newcommand{\bR}{\bm{R}}
\newcommand{\bS}{\bm{S}}

\newcommand{\bW}{\bm{W}}

\newcommand{\bZ}{\bm{Z}}

\newcommand{\cA}{\mathcal{A}}

\newcommand{\cC}{\mathcal{C}}
\newcommand{\cD}{\mathcal{D}}
\newcommand{\cE}{\mathcal{E}}
\newcommand{\cF}{\mathcal{F}}

\newcommand{\cI}{\mathcal{I}}
\newcommand{\cJ}{\mathcal{J}}
\newcommand{\cK}{\mathcal{K}}
\newcommand{\cL}{\mathcal{L}}
\newcommand{\cM}{\mathcal{M}}
\newcommand{\cN}{\mathcal{N}}
\newcommand{\cO}{\mathcal{O}}
\newcommand{\cP}{\mathcal{P}}
\newcommand{\cQ}{\mathcal{Q}}
\newcommand{\cR}{\mathcal{R}}

\newcommand{\cT}{{\mathcal{T}}}

\newcommand{\cV}{\mathcal{V}}

\newcommand{\cZ}{\mathcal{Z}}

\newcommand{\EE}{\mathbb{E}}

\newcommand{\PP}{\mathbb{P}}
\newcommand{\QQ}{\mathbb{Q}}
\newcommand{\RR}{\mathbb{R}}

\newcommand{\E}{{\mathbb E}}

\newcommand{\bgamma}{\bm{\gamma}}

\newcommand{\bmu}{\bm{\mu}}

\newcommand{\argmin}{\mathop{\mathrm{argmin}}}

\DeclareMathOperator{\Var}{{\rm Var}}

\DeclareMathOperator{\Cov}{\rm Cov}

\newcommand{\Tr}{\mathop{\text{tr}}\kern.2ex}

\mathtoolsset{showonlyrefs}

\TITLE{Towards Optimal Inventory Control under Censored Demand: A Biased Sample-Average Approximation Approach}
\RUNTITLE{Inventory Control under Censored Demand}
\RUNAUTHOR{Han, Fan, Zhang, and Zhou}

\ARTICLEAUTHORS{\AUTHOR{Yuxuan Han, \quad Xiaoyu Fan, \quad Jiawei Zhang, \quad Zhengyuan Zhou}
\AFF{Stern School of Business, New York University \\
\texttt{\{yh6061, fx2087, jz31, zzhou\}@stern.nyu.edu}}
}

\ABSTRACT{We study data-driven multi-period lost-sales inventory control under censored demand, where a stockout reveals only that demand exceeded the stocking level. We develop a unified, model-based framework for policy learning from censored data, built on a new cost decomposition for base-stock policies and a biased sample-average approximation (SAA) approach. The cost decomposition allows us to propose a new coverage condition under which censored observations are informative enough for sample-efficient policy learning. Guided by this coverage condition, we design two biased SAA algorithms: an upper-biased one that achieves near-optimal sample complexity under the offline coverage condition, and a lower-biased one that actively generates the required coverage and achieves near-optimal regret online. More broadly, this biased SAA approach provides a general principle for implementing pessimism and optimism under censored feedback, which may be of independent interest.}


\begin{document}
\maketitle

\section{Introduction}

Multi-period inventory control with possibly time-changing costs and demands is a classic problem in operations management and supply chain theory, with broad practical importance and a long research tradition~\citep{stevens1989integrating,zipkin2000foundations,chen2012pricing}.

Under standard demand-independence assumptions, inventory control problems can be formulated as Markov decision processes (MDP), and the optimal policies can be computed via the dynamic programming (DP) approach~\citep{scarf1960optimality,zipkin2008structure,bertsekas2012dynamic}.
In many real-world settings, however, the demand distribution is unknown and must be learned from data. This motivates the study of data-driven inventory control, which seeks to design near-optimal replenishment policies from historical data in which demand observations are often censored.

The underlying MDP structure of inventory control has led to substantial recent progress in data-driven inventory learning, enabling the adaptation of reinforcement learning techniques and the analysis of their sample-complexity guarantees~\citep{cheung2019sampling,halman2020provably,qin2023sailing,xie2024vc}.
In particular, the recent analysis of \citet{xie2024vc} establishes the sharpest known sample-complexity guarantees by exploiting structure specific to inventory problems that generic MDP arguments do not capture.
However, most existing works in this line focus on the \emph{uncensored} observation model, in which the learner always observes the full demand realization regardless of the corresponding order-up-to level.
This model is often unrealistic: stockouts typically make the realized demand unobservable, so that the informativeness of a dataset depends on the replenishment strategy under which it was collected. 

Switching from uncensored to censored observations in an offline dataset changes the statistical nature of the problem fundamentally. 
Under censored observations, a sample collected at censoring level $y$ reveals the realized demand only when demand is below $y$; otherwise, it reveals only that demand is at least $y$.
Each sample is thus informative only about demand below the level at which it was collected. Because different samples are collected at different censoring levels, the data inform policy learning only about the demand range these levels cover. This raises two challenges absent from the uncensored setting: characterizing when this range suffices for learning the optimal policy, and designing algorithms that exploit the available censored information efficiently. These challenges prevent existing uncensored analyses, including the sharp guarantees of \citet{xie2024vc}, from carrying over directly to the censored setting.

To the best of our knowledge, \citet{qin2023sailing} is the only prior work
that studies censored feedback in the \emph{offline} multi-period inventory
setting, where the learner must act on a fixed, pre-collected dataset. A central message of their analysis is an impossibility result: for any algorithm and non-trivial censoring level, there exist problem instances for which demand censoring induces an $\Omega(1)$ learning error.
While this negative result highlights the difficulty of policy learning under
censored observations, its worst-case perspective does not identify the
instance-dependent condition that actually governs policy learning. This
motivates two questions for the offline policy learning beyond the worst case:
\begin{quote}
\textit{Question~1: For a fixed problem instance, what coverage condition on the observed censoring levels suffices to consistently learn its optimal policy?}

\noindent\textit{Question~2: When this condition holds, which algorithm attains the
optimal sample complexity?}
\end{quote}

The two questions above concern a fixed, pre-collected dataset produced by a
behavior policy taken as given. Answering Question~1 would tell us which
censored datasets are informative enough for policy learning, and immediately
raises an \textit{online policy-design} counterpart: can such a condition be
satisfied by data collected from a cost-aware adaptive policy that actively
interacts with the inventory system? This is not a routine
extension of the offline problem, because a single adaptive data-collection
policy now plays two roles at once. Its trajectory forms the dataset from which
demand information is learned; but the same trajectory also incurs real
inventory cost as it unfolds. These two roles can pull in opposite directions:
collecting more informative observations may require stocking higher to reveal
demand at more levels, yet higher stocking incurs holding cost and worsens the
policy's own performance. This motivates the following question:
\begin{quote}
\textit{Question~3: Can a single adaptive policy generate a dataset
satisfying the condition of Question~1 while achieving low online regret?}
\end{quote}

In this work, we answer these three questions within a unified, model-based framework for policy learning under censored demand. For the offline setting, we identify an instance-dependent coverage condition on the censoring levels of the dataset that is sufficient for consistent learning, and we give an algorithm that attains the optimal sample complexity whenever this condition holds. For the online setting, we design an adaptive algorithm that achieves optimal regret while ensuring its collected data satisfy the same coverage condition.
Conceptually, these results rely on two new technical ingredients: an intentionally biased, one-sided sample-average approximation (SAA) of demand distributions, whose perturbation direction is chosen according to the task; and a structural cost decomposition tailored to base-stock policies, which serves as the common backbone of our analysis by decomposing the policy gap into local errors on the parts of the demand distribution that are decision-relevant.

\subsection{Contributions}

In this work, we study a $T$-period lost-sales inventory problem under bounded demand distributions, with the holding
and lost-sales penalty coefficients $h_t,b_t$. We denote by $C_1^\pi(x)$ the total cost of a policy $\pi$ with starting inventory level $x$ and $C_1^\star(x) = C_1^{\pi^\star}(x)$ the optimal cost benchmark. In most of this work, we assume a discrete demand setting for simplicity: at each period $t$, the
demand distribution $P_t$ with CDF $F_t$ is supported on $[M]_+ := \{0,1,\ldots,M\}$; we extend our analysis to the general demand setting in Section~\ref{sec: general-distribution}. We summarize our main contributions as follows.

\medskip\noindent\textbf{A derivative-based cost decomposition for base-stock policies.}
Our first contribution is a cost-decomposition identity, detailed in Theorem~\ref{thm: cost-decomposition}, tailored to the derivative structure of base-stock policies. 
The key structural fact is that a base-stock level is determined by the sign of the marginal order-up-to value, $D_t^\star(j):=W_t^\star(j+1)-W_t^\star(j)$, rather than by the value function itself. 
We therefore analyze how these marginal values, and their estimation errors, propagate through the demand-induced inventory dynamics. This leads to a deterministic decomposition of the policy gap into localized errors of the estimated demand model, instead of reducing it to the policy evaluation gap as in generic RL analysis \citep{li2024breaking,ren2021nearly,xiong2022nearly}.

A remarkable corollary of this cost-decomposition identity, detailed in Corollary~\ref{coro: uniform-convergence}, is the uniform-CDF-error reduction
\begin{equation}\label{eq: uniform-reduction-contribution}
C_1^{\hat\pi}(x) - C_1^{\pi^\star}(x) 
\lesssim 
(h_\infty + b_\infty) M T \cdot 
\sup_{1 \leq t \leq T}  \lVert \bF_t - \hat\bF_t\rVert_\infty,
\end{equation}
up to a lower-order term in the uniform CDF error. Although censored feedback is the main focus of this paper, this corollary alone, when applied to the uncensored setting, already illustrates why the decomposition is powerful.

First, under non-stationary demand, applying the standard period-wise SAA estimator to uncensored observations recovers the sharp rate $\tilde{\cO}(TM/\sqrt{N})$ of \citet{xie2024vc} up to logarithmic terms when $N$ trajectories of demands are observed. 
Our proof, however, obtains this rate by directly tracking error propagation through the model-based DP, rather than through a VC-dimension analysis of the policy class. 
To the best of our knowledge, this is the first direct model-based proof of the sharp sample-complexity rate for this multi-period inventory setting.

Second, under stationary demand, the estimator-agnostic nature of the decomposition allows us to aggregate samples across periods and obtain the sharper policy-learning rate $\tilde{\cO}(M\sqrt{T/N})$.
Notably, this rate is strictly below the $\Omega(MT/\sqrt N)$ lower bound for evaluating the optimal cost $C_1^{\pi^\star}(x)$.
This has two implications.
Conceptually, it shows that in this structured inventory class, learning a near-optimal policy can be \textit{strictly easier} than evaluating the optimal policy, which is not true for general RL \citep{ren2021nearly}.
Methodologically, it rules out previous proof strategies that reduce the policy-learning gap to policy evaluation error, as in \citet{qin2023sailing} or general RL works \citep{li2024breaking,ren2021nearly,xiong2022nearly,sidford2018near}: any such route inherits the $\Omega(MT/\sqrt N)$ policy evaluation barrier and therefore cannot attain the sharper $\tilde{\cO}(M\sqrt{T/N})$ policy learning rate.
Obtaining the sharp rate requires a different analytical machinery, which is precisely what our derivative-based cost-decomposition identity provides.

Finally, \citet{ganggang2024all} consider an asymptotic reduction result, as in~\eqref{eq: uniform-reduction-contribution}, for the \textit{optimal policy evaluation} task under possibly unbounded, continuous demand distributions with additional regularity assumptions. Under our notation, their result reads as
    $\cO_p\big(
T^2\sum_{t=1}^T
        \lVert\bF_t-\widehat{\bF}_t\rVert_\infty
    \big)$ 
 with cost factors and regularity factors treated as constants. 
In contrast, our result~\eqref{eq: uniform-reduction-contribution} provides a deterministic reduction for policy sub-optimality in the bounded-demand setting, with an improved dependence on $T$.

The uniform reduction~\eqref{eq: uniform-reduction-contribution} is, however, only a weakened form of the identity and is not the mechanism behind our censored guarantees. 
Under censored observations, estimation precision is highly non-uniform across demand levels, and the $\sup$-norm bound loses precisely the coverage information that drives the censored rate. 
The offline and online analyses below therefore work directly with Theorem~\ref{thm: cost-decomposition} instead of Corollary~\ref{coro: uniform-convergence}, and pair it with task-specific biased CDF estimators.

\begin{table}[t]
\centering
\small
\renewcommand{\arraystretch}{1.5}
\resizebox{0.9\textwidth}{!}{
\begin{tabular}{clccc}
\toprule
\textbf{Setup} &\textbf{Work} & \makecell{\textbf{Stationary}\\\textbf{demand}} & \makecell{\textbf{Censored}\\\textbf{observations}} & \makecell{\textbf{Leading-order}\\
\textbf{Result}} \\
\midrule
\multirow{4}{*}{\makecell{\textbf{Offline}\\\textbf{sub-optimality gap}\\
\textbf{with $N$ trajectories}}} & \citet{qin2023sailing} & \xmark & \xmark & $\cO(MT^{3/2}/\sqrt{N})$ \\
& \citet{xie2024vc} & \xmark & \xmark & $\cO({MT}/{\sqrt{N}})$ \\
& \textbf{This work} & \xmark & \textcolor{blue}{\cmark} & $\widetilde{\cO}(MT/\sqrt{\mathcal{N}^\star})$ \\
& \textbf{This work} & {\cmark} & \textcolor{blue}{\cmark} & $\widetilde{\cO}(MT/\sqrt{\mathcal{N}_{\mathsf{agg}}^\star})$ \\
\midrule
\multirow{2}{*}{\makecell{\textbf{Online}\\\textbf{episodic regret}\\\textbf{over $K$ episodes}}} & \textbf{This work} & \xmark & \textcolor{blue}{\cmark} & $\widetilde\Theta(MT\sqrt{K})$ \\
& \textbf{This work} &{\cmark} & \textcolor{blue}{\cmark} & $\widetilde\Theta(M\sqrt{KT})$ \\
\bottomrule
\end{tabular}}
\caption{Comparison of sub-optimality and regret guarantees for learning multi-period inventory policies.
$\cN^\star$ and $\cN_{\mathsf{agg}}^\star$ denote the effective sample sizes and its time-aggregated version under censored observations, where $\cN^\star \geq N$ and $\cN_{\mathsf{agg}}^\star \geq NT$ in the uncensored setting.}
\label{tab:offline-guarantees}
\end{table}

\medskip\noindent\textbf{Optimal offline policy learning with upper-biased SAA.}
We now turn to the offline censored setting, where the learner is given a fixed dataset $\cD:=\{(y_t^k, \bar d_t^k)\}_{t=1,k=1}^{T,N}$ of $N$ censored trajectories collected through any adaptive process. At each period $t$ of each trajectory $k$, $y_t^k$ is the realized order-up-to level, equivalently the post-ordering inventory level, and $\bar d_t^k = \min\{y_t^k, d_t^k\}$ is the censored demand observation.

For Question~1, we show that the learning error in the offline setting can be captured by an instance-dependent effective sample size
$$\cN^\star
:=\bigg(
\frac{1}{MT}{\sum\nolimits_{t \in [T],\, j \in [M-1]_+}\dfrac{\PP(y_t^\star > j \mid x_1 = 0)}{N_{t,j}\vee 1}} \bigg)^{-1},
\qquad
N_{t,j} := \sum\nolimits_{k=1}^N \bm{1}\{y_t^k > j\},
$$with the convention $a/0 = +\infty$. Here $N_{t,j}$ counts the trajectories whose period-$t$ order-up-to level exceeds $j$---precisely those that reveal whether demand falls below $j$---while $\PP(y_t^\star > j \mid x_1 = 0)$ measures how often the optimal policy requires information above level $j$. Thus $\cN^\star$ captures the relevant coverage condition for censored feedback: the dataset need not cover all demand levels uniformly, but it must cover the demand coordinates that matter for the optimal policy.

For Question~2, we propose DP-UCB, which constructs an \textit{upper-biased} SAA estimator of the demand CDF from $\cD$ and then solves the DP under this biased model. The upward CDF bias can be viewed as a model-level analogue of the \textit{pessimistic principle} in offline RL: by acting under a conservative model for poorly covered regions, the learner avoids relying on parts of the problem that the data cannot certify \citep{xiong2022nearly, rashidinejad2021bridging,li2024settling,jin2025pessimism}. 
In offline RL, pessimism is typically implemented at the level of value or Q-functions, by penalizing state-action pairs with limited coverage. 
In contrast, we propose a different implementation due to the nature of censored feedback: the informativeness of the dataset about $F_t(j)$ is non-increasing in $j$, since fewer trajectories satisfy $y_t^k > j$ at higher demand levels. 
Pessimism therefore acts on the demand model itself rather than on the value function: we show that inflating the CDF upward induces a downward bias in the resulting base-stock levels, pushing the policy toward lower demand coordinates where censored observations are more informative.

Technically, this shift in where pessimism is applied requires a different analysis. Since the bias is placed on the CDF rather than on a value or Q-function, standard pessimistic-value analyses do not directly translate. 
Our cost-decomposition identity provides such missing bridge: combined with the one-sided policy bias, it turns the model-level CDF errors into a cost-gap bound weighted by the optimal-policy coverage $\PP(y_t^\star > j \mid x_1 = 0)$. 
Thus, when the offline dataset covers the coordinates relevant to the optimal policy, the cost gap is controlled by estimation errors on precisely the coordinates measured by $\cN^\star$.
Up to logarithmic and burn-in terms, DP-UCB attains the  sub-optimality gap
\begin{equation}\label{intro-eq: offline-bound}
\max_x \big[C_1^{\hat\pi}(x) - C_1^{\pi^\star}(x)\big]
=
\tilde{\cO}\left(
\frac{(h_\infty + b_\infty)M T}{\sqrt{\cN^\star}}
\right).    
\end{equation}
 In particular, $\cN^\star \geq N$ in the uncensored setting, so this rate reduces to the uncensored non-stationary guarantee derived from~\eqref{eq: uniform-reduction-contribution} above.

When demand is stationary, the estimator-agnostic property of our cost decomposition allows us to aggregate censored observations across periods into a single CDF estimator. The aggregated effective sample size $\cN_{\mathsf{agg}}^\star$ is defined analogously to $\cN^\star$, with $N_{t,j}$ replaced by the period-summed count $\sum_{s=1}^T N_{s,j}$. The aggregated DP-UCB algorithm achieves the improved rate
$$
\tilde{\cO}\left(
\frac{(h_\infty + b_\infty)M T}{\sqrt{\cN_{\mathsf{agg}}^\star}}
\right).
$$
In the uncensored setting, $\cN_{\mathsf{agg}}^\star \geq NT$, so this result recovers the sharper aggregated rate $\tilde{\cO}(M\sqrt{T/N})$ obtained in the uncensored stationary setting above.

Regarding the optimality of the above results, it is tricky to state optimality directly with respect to $\cN^\star$, as it depends jointly on the underlying instance and the offline censoring levels. We therefore state optimality in a minimax sense over coverage classes: in Theorems~\ref{thm: lower-bound-independent} and~\ref{thm: lower-bound-independent-identical}, for a given coverage fraction $p \in (0,1]$, we construct classes of instances, together with deterministic censoring levels, satisfying $\cN^\star\geq \lfloor Np\rfloor$ (and $\cN_{\mathsf{agg}}^\star\geq \lfloor NTp\rfloor$, respectively) for all instances, and then show the $\Omega(MT(h_\infty + b_\infty)/\sqrt{Np})$ (and $\Omega(M(h_\infty + b_\infty)\sqrt{T/(Np)})$, respectively) minimax lower bounds over such classes. In particular, these results also cover the uncensored case when $p = 1$. Our lower bound results are, to the best of our knowledge, the first tight lower bounds for multi-period inventory policy learning even under the uncensored setting\footnote{Notably, \citet{qin2023sailing} also states an $\Omega(T/\sqrt{N})$ lower bound in the non-stationary demand setting. However, in their hard instance, the $T$ factor arises from the scaling of the cost coefficient, namely $b_\infty=\Theta(T)$. Under our notation, their result can therefore be reinterpreted as the horizon-free lower bound $\Omega(b_\infty/\sqrt{N})$, rather than as a horizon-dependent one.}. Previous lower bound results mostly focus on single-period newsvendor problems \citep{cheung2019sampling,fan2022sample}.

\medskip\noindent\textbf{Optimal online adaptive policy design with lower-biased SAA.}
We now turn to the online policy-design problem to resolve Question~3, where the learner adaptively interacts with the inventory system over $K$ episodes.

We resolve Question~3 by proposing DP-LCB, an \textit{optimistic} counterpart of DP-UCB with the bias direction reversed. 
At each episode $k$, DP-LCB executes a base-stock policy $\{s_t^{(k)}\}_{t=1}^T$ obtained by solving the DP under a \textit{lower-biased} SAA model fitted from the censored observations collected over all previous interactions. 
Just as the upper-biased CDFs in DP-UCB implement pessimism at the level of the demand model, the lower-biased CDFs in DP-LCB implement the standard principle of \textit{optimism in the face of uncertainty} at the level of the demand model itself~\citep{auer2002finite,auer2008near,azar2017minimax,jin2020provably}. 
Lowering the estimated CDFs makes the optimistic model expect higher demand, so the DP prescribes upward-biased base-stock levels.
Importantly, DP-LCB does not keep this upward bias fixed. 
As more censored observations are collected, the imposed bias shrinks across episodes, producing the monotone pattern
\begin{equation}\label{intro-eq: base-stock-sequence}
s_t^\star \leq s_t^{(K)} \leq s_t^{(K-1)} \leq \cdots \leq s_t^{(1)},\quad \forall t \in [T].
\end{equation}
Thus early episodes use more optimistic base-stock levels to reveal demand information at the coordinates relevant to the optimal policy, while later episodes, computed from richer data, rely only on coordinates that earlier episodes have already covered. 
Each episode therefore uses information from earlier episodes while providing information for later episodes.

The monotone pattern~\eqref{intro-eq: base-stock-sequence} yields two guarantees at once.

First, the left inequality in~\eqref{intro-eq: base-stock-sequence} ensures that every executed policy covers the optimal policy from above. 
Consequently, the data generated by DP-LCB satisfy the offline coverage condition identified in Question~1: after $K$ episodes, the effective sample size $\cN^\star$ of the collected dataset satisfies
$
\cN^\star \gtrsim K,
$
up to logarithmic and burn-in terms. 
In particular, if one runs the proposed offline algorithm on the $K$ trajectories collected by DP-LCB, the offline guarantee~\eqref{intro-eq: offline-bound} yields the same statistical order as in the uncensored setting. 
This establishes the data-collection optimality of DP-LCB: its censored trajectories are, for the purpose of offline policy learning, as informative as $K$ uncensored trajectories up to logarithmic and burn-in terms.

Second, the monotone decrease in~\eqref{intro-eq: base-stock-sequence} turns this self-coverage into regret control: at each episode $k$, the offline guarantee of Question~2 applies to the data accumulated so far, yielding a per-episode cost gap of order $\tilde{\cO}(MT/\sqrt{k})$. 
Summing over $k$ gives
$$
\mathrm{Regret}(K)
:=
\sum_{k=1}^K \max_x \big[C_1^{(k)}(x) - C_1^{\pi^\star}(x)\big]
\lesssim
(h_\infty + b_\infty) M T  \sqrt{K},
$$
up to logarithmic and burn-in terms. 
A matching lower bound holds even under uncensored demand observations, so this rate is optimal. 
Thus active coverage generation eliminates the leading-order cost of censoring in the online setting.

Finally, when demand is stationary, the same aggregation principle as in the offline setting applies: observations can be pooled across both periods and episodes into a single CDF estimator. 
The aggregated DP-LCB algorithm achieves the sharper optimal regret
$$
\mathrm{Regret}(K) = \tilde{\cO}\big(
(h_\infty + b_\infty) M \sqrt{KT}
\big),
$$
reflecting a factor-$T$ gain in effective sample size and hence a $\sqrt T$ improvement in the regret rate.

\subsection{Other Related Works}

\medskip\noindent\textbf{Data-driven Inventory Control.}
Data-driven inventory control has been extensively studied across a variety of settings, including the online setting where the decision maker sequentially observes sales or demand information and updates ordering decisions over time while simultaneously learning from the data~\citep{godfrey2001adaptive,huh2009nonparametric,chen2015nonparametric,shi2016nonparametric,zhang2018perishable,agrawal2019learning,zhang2018leadtimes,YLS2019,zhang2020closing,chen2022learning,lyu2024minibatch}, and the offline setting where the decision maker instead learns a replenishment policy from a fixed historical dataset before deploying the learned policy in future operations~\citep{levi2007provably,levi2015data,cheung2019sampling,besbes2021big,fan2022sample,zhang2021sampling,chen2023learning,xie2024vc,fan2024don,fan2025sample}.

The line most closely related to our work studies inventory control with possibly non-stationary demands and costs \citep{qin2023sailing,halman2020provably,xie2024vc,ganggang2024all}. To our knowledge, no prior work in this line provides consistent learning guarantees under censored demand feedback.

Another related line of work studies the repeated newsvendor problem with possibly censored feedback \citep{besbes2013implications,fan2022sample,lyu2024closing,hssaine2024data,chen2024survey,kumar2026value,chen2025learning}, as well as the multi-period setting with stationary demand and costs \citep{huh2009nonparametric,besbes2013implications,lyu2024minibatch}. Under both setups, the optimal policy is a base-stock policy determined by a one-stage quantile rather than by multi-period DP.

\medskip\noindent\textbf{General RL.}
Our results also connect to reinforcement learning for finite-horizon Markov decision processes (MDPs). The full-information results of Section~\ref{sec: uniform-convergence} are most directly comparable to RL under a generative model, where the transition kernel can be sampled freely \citep{sidford2018near,li2024breaking}. Viewing the inventory problem as an MDP with horizon $T$, and suppressing inventory-size and cost factors, \citet{qin2023sailing} obtain the rate $\tilde{\cO}(T^{3/2}/\sqrt N)$, matching the horizon dependence of a generic model-based analysis, whereas \citet{xie2024vc} exploit the inventory structure to obtain the sharper $\tilde{\cO}(T/\sqrt N)$ policy-learning rate. Our cost decomposition recovers this rate through a direct model-based argument and, under stationary demand, further sharpens it to $\tilde{\cO}(\sqrt{T/N})$. This improvement reflects one of the main differences between inventory learning and generic RL: reducing policy learning to policy evaluation is sharp for general MDPs \citep{li2024settling,li2024breaking,xiong2022nearly}, but is too coarse for the structured base-stock policy class, where policy learning can be strictly easier than evaluating the optimal cost \citep{ren2021nearly}.

Under censored feedback, our offline result is related to pessimism under partial data coverage in offline RL \citep{rashidinejad2021bridging,jin2025pessimism,xiong2022nearly,nguyen2023sample,li2024settling}. Compared with the standard $\tilde{\Theta}(T^{3/2}\sqrt{\cC/N})$ scaling under single-policy concentrability, our offline bound improves the horizon dependence to $\tilde{\Theta}(T\sqrt{\cC^\star/N})$; see the discussion following Corollary~\ref{corollary: C-star-sample-complexity}. The online result is analogous to optimism in episodic RL \citep{auer2008near,azar2017minimax,jin2020provably,dann2015sample}: DP-LCB uses a lower-biased CDF to generate exploration through higher base-stock levels. This yields the optimal regret $\tilde{\Theta}(T\sqrt K)$, improving the horizon dependence relative to the generic finite-horizon RL rate $\tilde{\Theta}(T^{3/2}\sqrt K)$.

Finally, our stationary-demand results complement the line of RL work that exploits stationary, time-homogeneous transitions to improve the leading horizon dependence from $\cO(T^{3/2})$ to $\cO(T)$~\citep{ren2021nearly}. In the structured inventory setting, aggregation sharpens the corresponding dependence further, from $\cO(T)$ to $\cO(\sqrt T)$, in both the offline and online bounds.

\section{Preliminaries}\label{sec: preliminaries}

\medskip\noindent\textbf{Notations.}
For any positive integer $n$, we denote $[n]=\{1,2,\dots,n\}$ and
$[n]_+=\{0,1,2,\dots,n\}$. For integers $a\le b$, we write
$[a,b]=\{a,a+1,\dots,b\}$ and $[a,b)=\{a,a+1,\dots,b-1\}$.
For any real number $r$, we denote $(r)_+=\max\{r,0\}$.
For a finite sum $\sum_{i=a}^b z_i$, we set it to be $0$ whenever
$a>b$. Similarly, an empty product is interpreted as the identity
operator of the appropriate dimension.
All vectors and matrices are indexed from $0$ unless otherwise stated.
In particular, an $M$-dimensional vector is indexed by $[M-1]_+$, while
an $(M+1)$-dimensional vector is indexed by $[M]_+$. For a discrete
function $f$ defined on $[M-1]_+$, we use its boldface version $\bm f$
to denote the corresponding vector, i.e., $[\bm f]_j=f(j)$. 
We adopt the conventions
$a/0:=+\infty$ for any $a>0$ and $0/0:=1$, applied throughout to all ratios. Throughout this work, we assume the confidence level $\delta \in (0,1/2)$ in all high probability statements.

\medskip\noindent\textbf{Multi-Period Inventory Control.}
We consider a finite-horizon multi-period inventory control problem over $T$ periods.
At each period $t \in [T]$:

\begin{enumerate}[nosep]
    \item The system starts at entering inventory level $x_t$, and the decision maker selects an order-up-to level $y_t \geq x_t$.
    \item Demand $d_t$ is realized from distribution $P_t$, and the one-period cost
    \begin{align*}
        c_t(y_t) := h_t (y_t - d_t)_+ + b_t (d_t - y_t)_+
    \end{align*}
    is incurred, where $h_t$ and $b_t$ denote the holding-cost and lost-sales penalty rates, respectively.
    \item The next entering inventory level is
    $
        x_{t+1} = (y_t - d_t)_+ .
    $
\end{enumerate}

The objective is to design a policy $\pi$ that adaptively selects the order-up-to levels $\{y^\pi_t\}_{t=1}^T$ to minimize the total expected cost starting from an initial inventory level $x$,
\begin{equation}\label{eq: multi-period-cost}
C_1^\pi(x):=\EE\big[\sum\nolimits_{t=1}^T c_t(y^\pi_t)\lvert x_1 = x\big],    
\end{equation}
where the expectation is taken over the demand realizations and the behavior of $\pi$.

Throughout most parts of this work, we assume the following facts:
\begin{enumerate}[nosep]
    \item The cost coefficients are uniformly bounded:
$ 0 \le h_t \le h_\infty,  0 \le b_t \le b_\infty, \forall t \in [T].$
\item Each demand distribution $P_t$ is supported on the finite set $[M]_+$, with masses $P_t(d_t = k) = \mu_{tk}$ for $k\in [M]_+$. And the initial inventory level $x_1$ is integer.
\item The demand realizations $\{d_t\}_{t=1}^T$ are independent across $t\in [T]$.
\end{enumerate}
And we generalize the discrete demand assumption to general demand distributions in Section~\ref{sec: general-distribution}.
\medskip\noindent\textbf{Optimal Policy and Dynamic Programming.}
Under the independent demand assumption, $\pi^\star$ can be solved by DP. More precisely, $\pi^\star$ can be obtained via
\begin{align*}
    \pi^\star_t(x) = \argmin_{y\geq x}  W^\star_t(y), \quad  W^\star_t(y):= \EE\big[c_t(y)+ C_{t+1}^\star \big((y-d_{t})_+\big)\big],
\end{align*}
with $C^\star_t(x)$ given by the Bellman equation
$$
C^\star_t(x) = \min_{y \ge x}
\EE\big[
c_t(y) +
C^\star_{t+1}\big((y - d_t)_+\big)
\big],
\qquad
C_{T+1}^\star(x) \equiv 0.
$$
In particular, the minimum over $y \geq x$ can always be attained within $[x, M]$, as demands are supported on $[M]_+$. We therefore restrict to $y \in [x, M]$ throughout.
In addition to the above general characterization, it has been shown in \citet{levi2007provably,cheung2019sampling} that $W_t^\star(\cdot)$ is discrete convex in the sense that its discrete derivative
$$
D^\star_t(y) := W^\star_t(y+1) - W^\star_t(y),
\qquad \forall y \in [M-1]_+, t\in [T],
$$
is always non-decreasing in $y$. This indicates that $\pi^\star$ can be a \textit{base-stock} policy given as \begin{align}\label{eq: true-DP}
    \pi_t^\star(x) = \begin{cases}
        s_t^\star, & \text{ if } x \le s_t^\star,\\
        x, & \text{otherwise.} 
    \end{cases},\quad
s_t^\star := \min\{y \in [M-1]_+ : D^\star_t(y) \ge 0\},
\end{align}
with the convention $s_t^\star = M$ if $D^\star_t(y) < 0$ for all $y \in [M-1]_+$.
Notably, restricting the order-up-to levels to integer values in~\eqref{eq: true-DP} is without loss of optimality: when the demands are integer-valued and $x_t \in [M]_+$, it is well-known that an integer-valued optimal base-stock policy remains optimal even when arbitrary order-up-to levels larger than $x_t$ are allowed \citep{zipkin2000foundations}.

\medskip
\noindent\textbf{Data-Driven Policy Learning.}
In this work, the demand distributions $\{P_t\}_{t=1}^T$ are unknown to the decision maker (the learner), who must instead learn a near-optimal policy from data. 
We focus throughout on \emph{censored} demand observations: whenever an order-up-to level $y$ is chosen and demand $d \sim P_t$ is realized, the learner observes only the truncated quantity
$
\bar d := \min\{y, d\}.
$
The realized demand is revealed only when it falls below the available inventory $y$; otherwise the learner learns only that demand is at least $y$.
The order-up-to levels at which data are collected therefore govern their informativeness: a period-$t$ observation collected at level $y_t^k$ reveals whether demand falls below a coordinate $j \in [M-1]_+$ only when $j < y_t^k$, so what matters is not just how many trajectories are collected, but which demand coordinates their censoring levels cover.

We study this censored problem in two settings. 
In the \emph{offline} setting, the learner is given a fixed dataset $\cD := \{(y_t^k, \bar d_t^k)\}_{t=1,k=1}^{T,N}$ of $N$ censored trajectories, collected in advance by a behavior policy $\pi^{(b)}$ that may itself be adaptive. 
Here $y_t^k$ is the realized order-up-to level in period $t$ of trajectory $k$, and $\bar d_t^k := \min\{y_t^k, d_t^k\}$ is the corresponding censored observation. 
Using only $\cD$, the learner outputs a policy $\hat\pi$, whose performance is measured by the worst-case sub-optimality gap
$$
\max_{x\in [M]_+} \underbrace{[C_1^{\hat\pi}(x) - C_1^{\star}(x)]}_{:= \Delta(x;\hat{\pi})}.
$$
In this offline setting, the coverage of $\cD$ is fixed before learning begins: the learner can exploit the demand coordinates revealed by the behavior policy, but cannot extend them.

In the \emph{online} setting, the learner instead interacts with the system over $K$ episodes: in each episode $k \in [K]$ it selects a policy $\pi^{(k)}$, executes it for all $T$ periods, and observes the resulting censored trajectory before proceeding to the next. 
Writing $C_1^{(k)} := C_1^{\pi^{(k)}}$ for the cost of the deployed policy, the learner is evaluated \emph{solely} by its cumulative regret
$$
\mathrm{Regret}(K) := \sum_{k=1}^K \max_x \Delta(x;{\pi}^{(k)}).
$$
Although regret is the only criterion against which the online learner is measured, the $K$ trajectories it generates themselves constitute a censored dataset of exactly the offline form $\cD$. 
The coverage of this dataset is now endogenous: each episode's order-up-to levels simultaneously determine the cost charged to regret and which demand coordinates that episode reveals for future learning. 
Thus the online problem is not regret minimization in isolation: the learner must control its current inventory cost while generating coverage rich enough to support the offline learning guarantees developed below. 
It is this coupling between immediate cost and self-generated coverage---rather than either concern alone---that distinguishes the online problem from a routine repetition of the offline one, and that the online algorithm we develop is designed to exploit.

\section{Cost Gap Decomposition of Base-Stock Policies}\label{sec: general-decomposition}

In this work, we adopt a model-based approach for data-driven inventory control, where a demand model with per-period CDF sequence $\{\hat{\bF}_t\}_{t=1}^T$ is first constructed, and then the estimated policy $\hat{\pi}$ is solved via performing DP under $\{\hat{\bF}_t\}_{t=1}^T.$ 
While the construction of $\{\hat{\bF}_t\}_{t=1}^T$ varies across different demand models and observation protocols, the performance of $\hat{\pi}$ usually relies on how accurately $\hat{\bF}$ approximates $\bF$. 
The goal of this section is to provide a general cost gap decomposition theorem that relates the cost difference between $\pi^\star$ and $\hat{\pi}$ to the gap between $\hat{\bF}$ and $\bF$. This serves as the backbone of our sample complexity analysis in later analyses under different constructions of $\{\hat{\bF}_{t}\}_{t=1}^T$.

To formally describe $\hat{\pi},$ we briefly replicate the DP procedure under $\{\hat{\bF}_t \}_{t=1}^T$ in this paragraph. For the estimated value sequence $\{\hat{C}_t\}_{t=1}^T$ given by the Bellman equation under $\{\hat{\bF}_t\}_{t=1}^T$
\begin{align*}
\hat{C}_t(x) = \min_{y \ge x}
\hat\EE\big[
c_t(y) +
\hat{C}_{t+1}\big((y - d_t)_+\big)
\big],
\quad
\hat{C}_{T+1}(x) \equiv 0, \quad \forall x \in [M]_+,
\end{align*}
and  $ \hat{D}_t(y):= \hat{W}_t(y+1) -\hat{W}_t(y), \hat{W}_t(y) = \hat\EE\big[c_t(y)+ \hat{C}_{t+1} \big((y-d_{t})_+\big)\big]$ for $ y \in [M-1]_+,$
with $\hat{\EE}$ the expectation taken under the demand distribution given by $\{\hat{\bF}_{t}\}_{t=1}^T$. The optimal $\hat{\pi}$ can be characterized via a base-stock sequence $\{s_t\}_{t=1}^T$: \begin{align}\label{eq: empirical-DP}
    \hat{\pi}_t(x) = 
         \begin{cases}
        s_t, & \text{ if } x \le s_t,\\
        x, & \text{otherwise.}
    \end{cases}\quad \text{ with } s_t:= \min\{y \in [M-1]_+ : \hat{D}_t(y) \ge 0\},
\end{align}
with the convention $s_t = M$ if $\hat{D}_t(y) < 0$ for all $y \in [M-1]_+$.

\begin{table}[t]
\centering
\renewcommand{\arraystretch}{1.15}
\setlength{\tabcolsep}{4pt}
\resizebox{\textwidth}{!}{\begin{tabular}{p{0.1\textwidth} p{0.40\textwidth} p{0.15\textwidth} p{0.45\textwidth}}
\toprule
\textbf{Notation} & \textbf{Interpretation} & \textbf{Notation} & \textbf{Interpretation} \\
\midrule
$\bF_t \in \mathbb{R}^{M}$ 
& $\big[\bF_t\big]_j = F_t(j)$ for $j \in [M-1]_+$
& $\bW_t^\star \in \mathbb{R}^{M+1}$ 
& Order-up-to value under the optimal policy \\

$\bD_t^\star \in \mathbb{R}^{M}$ 
& Discrete derivative of $\bW_t^\star$ 
& $\bA_t(s) \in \mathbb{R}^{M\times M}$ 
& $[\bA_t(s)]_{ij} = \mu_{t,i-j}\bm{1}\{i\ge j\ge s\}$ \\

$\bc_t \in \mathbb{R}^{M}$ 
& $\big[\bc_t\big]_j = (h_t+b_t)F_t(j) - b_t$ 
& $\bq_t \in \mathbb{R}^{M}$ 
& $[\bq_t]_j = \sigma_t \bm{1}\{j \in [\alpha_t,\beta_t)\}\PP(x_t \le j \mid x_1 = 0)$ \\

$\bu_t \in \mathbb{R}^{M}$
& $\big[\bu_t\big]_j = \sigma_t \bm{1}\{j \in [\alpha_t,\beta_t)\}\PP(x_t^\star \le j \mid x_1 = 0)$
& $\bv_t \in \mathbb{R}^{M}$ 
& $[\bv_t]_j = \PP(y_t^\star \le j \mid x_1 = 0) - \PP(y_t \le j \mid x_1 = 0)$ \\
\bottomrule
\end{tabular}}\caption{Summary of notations in Section~\ref{sec: general-decomposition}, where the estimated versions $\hat{\bW}_t$, $\hat{\bD}_t$, $\hat{\bA}_t(s)$, and $\hat{\bc}_t$ are defined analogously by replacing $\bF_t$ with $\hat{\bF}_t$ and $\bmu_t$ with $\hat{\bmu}_t$.}

\label{tab: section3-notation}
\end{table}

In the remainder of this section, we derive a cost-gap decomposition for the model-based policy $\hat\pi$ relative to the optimal policy $\pi^\star$, using the base-stock structure above. 
We adopt the vector convention introduced in Section~\ref{sec: preliminaries}, the main notations used throughout this section are summarized in Table~\ref{tab: section3-notation}.
Here we set the boundary value $F_t(M)=1$ as a scalar; the vector $\bF_t\in\mathbb{R}^M$ collects only the coordinates $\{F_t(j)\}_{j\in[M-1]_+}$.

\medskip\noindent\textbf{Derivative based performance difference lemma.} We start with the following performance difference lemma, which provides a decomposition of total cost gap as the summation of optimal value gap caused by immediate action mismatch along the trajectory induced by the policy $\hat{\pi}$:
\begin{lemma}\label{lem: performance-difference} Denote $\{x_t\}_{t=1}^T$ as the entering inventory level trajectory when executing policy $\hat{\pi}$, then
    \begin{align*}
   \Delta(x;\hat{\pi}):=  C_1^{\hat{\pi}}(x) - C_1^{\star}(x)  = \sum_{t=1}^T\EE[W_t^\star(\hat{\pi}_t(x_t)) - W_t^\star(\pi^\star_t(x_t))  \lvert x_1 = x].
\end{align*}
\end{lemma}

Lemma~\ref{lem: performance-difference} is a straightforward adaptation of standard performance-difference results for general MDP \citep{kakade2002approximately,nguyen2023sample,bhandari2024global}. In our setting, however, the structured transition dynamics of the inventory MDP and the base-stock form of $\pi^\star$ and $\hat{\pi}$ further imply the following monotonicity property of $\Delta(\cdot;\hat{\pi})$:
\begin{proposition}\label{prop-convexity} 
$\Delta(\cdot;\hat{\pi})$ is a decreasing function in $x$, i.e. $\Delta(x;\hat{\pi}) \geq \Delta(x+1;\hat{\pi}),\forall x\in [M-1]_+.$
\end{proposition}

As a consequence, bounding the maximum cost gap $\max_x \Delta(x;\hat{\pi})$ reduces to controlling $\Delta(0;\hat{\pi})$. 
For this purpose, noticing that if we define $\alpha_t:= \min\{s_t,s_t^\star\}, \beta_t = \max\{s_t,s_t^\star\},\sigma_t = \text{sgn}(s_t-s_t^\star),$ then 
\begin{align*}
    W_t^\star(\hat{\pi}_t(x_t)) - W_t^\star(\pi^\star_t(x_t)) = \sigma_t\sum_{j = \alpha_t}^{\beta_{t}-1} \bm{1}\{x_t \leq j\} D_t^\star(j)  = \begin{cases}
        0,& \text{ if } x_t \geq \beta_t,\\
        \sigma_t \sum_{j=x_t}^{\beta_t-1} D^\star_{t}(j),& \text{ if } \alpha_t\leq x_t < \beta_t\\
        \sigma_t \sum_{j=\alpha_t}^{\beta_t-1} D^\star_{t}(j),& \text{ otherwise.}
    \end{cases}
\end{align*}
Thus, if we introduce the weight vector $\bq_t \in \RR^{M}$ with \begin{align*}
    \big[\bq_t\big]_j := \sigma_t \cdot \bm{1}\{ \alpha_t \leq j < \beta_t\} \PP(x_t \le j \lvert x_1 = 0), 
\end{align*}
then we have the following derivative based representation of the cost gap:\begin{align}\label{eq: derivative-based-err}
    \Delta(0;\hat{\pi}) = \sum_{t=1}^T \EE\Big[\sigma_t\sum\nolimits_{j = \alpha_t}^{\beta_{t}-1} \bm{1}\{x_t \leq j\} D_t^\star(j) \Big\lvert x_1 = 0\Big] = \sum_{t=1}^T \langle \bq_t, \bD_t^\star \rangle.
\end{align}

To explicitly connect~\eqref{eq: derivative-based-err} with the model difference between $\{\bF_t\}_{t=1}^T$ and $\{\hat\bF_t\}_{t=1}^T$, observe that by the rule of determining $s_t^\star, s_t$ in~\eqref{eq: true-DP},~\eqref{eq: empirical-DP},  we have $\sigma_t \hat{\bD}_t(j) \leq 0, \forall \alpha_t\leq j < \beta_t$. As a consequence,
$\langle \bq_t,  \bD_t^\star \rangle  \leq  \langle \bq_t,   \bD^\star_t - \hat\bD_t \rangle$. This then motivates us to track the role of the model distance in propagation of derivatives.

\medskip\noindent\textbf{Derivative Propagation.}
To provide tight description on how derivatives propagate along $t$, we first introduce the transportation matrices $\bA_t(s), \hat{\bA}_t(s)\in \RR^{M\times M}$ for $t\in [T], s\in [M]_+$ as
\begin{align*}
     \big[\bA_t(s)\big]_{ij} = \mu_{t,i-j} \bm{1}\{i\geq j\geq s\},\quad \big[\hat{\bA}_t(s)\big]_{ij} = \hat{\mu}_{t,i-j} \bm{1}\{i\geq j\geq s\},\quad \forall i, j \in [M-1]_+,
\end{align*}
where $\hat{\mu}_{t,j}:= \hat{F}_{t}(j)-\hat{F}_{t}(j-1)$ is the induced mass from $\hat{\bF}_t$.

With this definition, we have the following recursive formula of derivatives:

\begin{proposition}\label{prop-matrix-notation} It holds for $t\in [T]$ that  \begin{align}
        \bD_t^\star =  \underbrace{h_t \bF_t - b_t (\bm{1} - \bF_t )}_{:= \bc_t}  + \bA_t(s_{t+1}^\star) \bD^\star_{t+1}, \label{eq: D-forward-equation}\\
        \hat{\bD}_t =  h_t \hat{\bF}_t - b_t (\bm{1} - \hat{\bF}_t )  + \hat{\bA}_t(s_{t+1}) \hat\bD_{t+1}, \label{eq: D-hat-forward-equation}
    \end{align}
where we set for convenience $\bD_{T+1}^\star = \hat{\bD}_{T+1} = \bm{0}$ and $s_{T+1} := s^\star_{T+1} := 0$.

As a consequence, with the notations $\Delta \bD_t :=  \bD^\star_t - \hat{\bD}_t$, $\Delta \bF_t := \bF_t - \hat{\bF}_t$, and $\Delta \bA_t (s) :=  \bA_t(s) - \hat{\bA}_t(s)$, we have \begin{align}\label{eq: recursion-D-diff}
       \Delta \bD_t  =   {\bA}_t(s_{t+1}) \Delta{\bD}_{t+1} + \Delta \bA_t(s_{t+1})\hat{\bD}_{t+1}+ \underbrace{ (h_t + b_t) \Delta \bF_t}_{:= \Delta \bc_t} + \big[\bA_t(s_{t+1}^\star) -\bA_t(s_{t+1}) \big] \bD^\star_{t+1}.
    \end{align}
\end{proposition}

In~\eqref{eq: recursion-D-diff}, the $\Delta \bD_t$ term is decomposed into three terms: (i) the propagation term $\bA_{t}(s_{t+1}) \Delta \bD_{t+1}$; (ii) the instantaneous error term due to the model distance at $t$-th period $\Delta\bA_t(s_{t+1}) \hat{\bD}_{t+1} + \Delta \bc_t$;  (iii) the policy mismatch term $\big[\bA_t(s_{t+1}^\star) - \bA_{t}(s_{t+1})\big]\bD_{t+1}^\star$.

In the following, we dualize the transportation rule of $\bD_t^\star$ to $\bq_t$ through~\eqref{eq: derivative-based-err} to make its effect more transparent. For this purpose, we introduce the following identity.

\begin{proposition}\label{prop: v-propagation}
    For $\{\bu_t\}_{t = 1}^T \subset \RR^{M}$ sequence defined as $\big[\bu_t\big]_j = \sigma_t \cdot \bm{1}\{\alpha_t \leq j < \beta_t\}\PP(x^\star_t \leq j \lvert x_1 = 0)$.
And for $\{\bv_t\}_{t=1}^T$ defined recursively through
    \begin{align}\label{eq: recursion-v-improved}
    \bv_{t+1}:=\bA_t(s_{t+1})^\top \bv_{t}+ \bu_{t+1},\forall t\in [T-1],\quad \bv_1:=  \bu_1.
    \end{align}
It holds that 
    \begin{align*}
     &\big(  \bA_t(s^\star_{t+1}) - \bA_t(s_{t+1}) \big)^\top \bv_t + \bq_{t+1}= \bu_{t+1},\quad \forall t \in [T-1].
    \end{align*}
\end{proposition}

\begin{remark}[Probabilistic Interpretation of Proposition~\ref{prop: v-propagation}]\label{remark: v-interpretation} We have the $\{\bv_t\}_{t=1}^T$ sequence defined in Proposition~\ref{prop: v-propagation} satisfies 
    $\big[\bv_t\big]_j = \PP(y_t^\star \leq j\lvert x_1 = 0) - \PP(y_t\leq j\lvert x_1 = 0)$.
And the last identity in Proposition~\ref{prop: v-propagation} is simply describing the transportation effect of $\bA_t(s_{t+1})$ as \begin{align*}
   &\big[ \bA_t^\top(s_{t+1}^\star) \bv_t \big]_j = \bm{1}\{j\geq s_{t+1}^\star\} \big[\PP({x}^\star_{t+1} \leq j\lvert x_1 = 0) - \PP({x}_{t+1} \leq j\lvert x_1 = 0) \big] ,\\
   &\big[ \bA^\top_t(s_{t+1}) \bv_t \big]_j   = \bm{1}\{j\geq s_{t+1}\} \big[\PP({x}^\star_{t+1} \leq j\lvert x_1 = 0) - \PP({x}_{t+1} \leq j\lvert x_1 = 0) \big].
\end{align*}
\end{remark}

Now with~\eqref{eq: derivative-based-err} and Proposition~\ref{prop: v-propagation}, we can show the following cost decomposition theorem:
\begin{theorem}\label{thm: cost-decomposition}
    With the notations $\bu_t, \bv_t$ introduced above, we have \begin{align}\label{eq: cost-decomposition-main-theorem}
    \max_{x\in [M]_+}\Delta(x;\hat{\pi}) = \sum_{t=1}^T \langle \bv_t, \Delta \bA_t(s_{t+1}) \hat{\bD}_{t+1} + \Delta \bc_t\rangle  + \sum_{t=1}^T \langle \bu_t, \hat{\bD}_t \rangle.
    \end{align}
\end{theorem}

Comparing~\eqref{eq: cost-decomposition-main-theorem} with the original representation~\eqref{eq: derivative-based-err}, the key transformation is an exchange of roles between the two policies in the inner-product terms: the weight vector $\bq_t$, generated by the entering-inventory trajectory
induced by the model-based policy $\hat\pi$, is replaced by the weight vector $\bu_t$, generated by the entering-inventory trajectory induced by the optimal policy $\pi^\star$. Meanwhile, the derivative sequence $\bD_t^\star$, whose signs determine the optimal base-stock level
$s_t^\star$, is replaced by $\hat{\bD}_t$, whose signs determine the model-based base-stock level $s_t$. The first term in~\eqref{eq: cost-decomposition-main-theorem} is precisely the residual
generated by this exchange.

The reason for passing from~\eqref{eq: derivative-based-err} to the exchanged term $\sum_{t=1}^T \langle \bu_t, \hat{\bD}_t \rangle$ is that, by construction, $\sigma_t \hat{\bD}_t$ is non-positive on the support of $\bu_t$, which lies on the index interval between $s_t$ and $s_t^\star$. Consequently, $\sum_{t=1}^T \langle \bu_t, \hat{\bD}_t\rangle$ is always non-positive, and therefore can be dropped when deriving an upper bound on the cost gap.

\begin{proof}[Proof of Theorem~\ref{thm: cost-decomposition}] From~\eqref{eq: recursion-D-diff} and Proposition~\ref{prop: v-propagation}, we have the following identity for every $t\in [T-1]$:
    \begin{align*}
    &\langle \bv_{t}, \Delta \bD_{t}  \rangle + \langle \bq_{t+1}, \bD^\star_{t+1}\rangle = \langle \bu_{t+1}, \bD^\star_{t+1}\rangle+ \langle \bA_t(s_{t+1})^\top\bv_{t
    }, \Delta \bD_{t+1} \rangle      + \big\langle \bv_{t} ,  \Delta \bA_{t}(s_{t +1}) \hat{\bD}_{t + 1} + \Delta \bc_t \big\rangle \\
    &= \langle \underbrace{\bA_{t}(s_{t+1})^\top \bv_{t}+ \bu_{t+1}}_{ = \bv_{t+1}}, \Delta \bD_{t+1}\rangle + \big\langle \bv_{t} ,  \Delta \bA_{t}(s_{t +1}) \hat{\bD}_{t + 1} + \Delta \bc_t \big\rangle + \langle \bu_{t+1}, \hat{\bD}_{t+1}\rangle.
    \end{align*}
    Then by $\bv_1 = \bq_1 = \bu_1$, $\langle \bq_{1},\bD_1^\star\rangle =  \langle \bv_{1},\Delta \bD_1\rangle + \langle \bu_1, \hat{\bD}_{1}\rangle,$ and $\hat{\bD}_{T+1} = \bm{0}, \Delta\bD_T = \Delta \bc_T,$ applying the above identity recursively for terms in the summation $\sum_{t=1}^T \langle \bq_t, \bD_t^\star \rangle$ finishes the proof.
\end{proof}

\section{Performance Bound with Uniform Convergence of CDF}\label{sec: uniform-convergence}

In this section, we introduce a corollary of the cost-decomposition identity in Theorem~\ref{thm: cost-decomposition}, which reduces the policy cost gap to \textit{uniform estimation error} of the underlying CDF functions $\{\bF_t\}_{t=1}^T$.

While this uniform reduction is coarse for censored feedback, which is the main focus of this paper, it already suffices to illustrate the power of Theorem~\ref{thm: cost-decomposition} in the uncensored setting considered in prior work~\citep{cheung2019sampling,halman2020provably,qin2023sailing,xie2024vc}. While in all these works the inventory dynamics are backlogged rather than lost-sales, the two dynamics are equivalent in our setup, as discussed in detail in Appendix~\ref{appendix-sec: backlog-LS-equivalence}. 
It not only recovers the sharp non-stationary SAA rate, but also yields a previously unknown aggregation-based rate under stationary demand and clarifies the separation between policy learning and policy evaluation in this structured inventory class.

More precisely, under a uniform error condition
\begin{equation}\label{eq: uniform-convergence}
   \sup\nolimits_{1\leq t\leq T} \lVert \bF_t - \hat{\bF}_t \rVert_\infty \leq \epsilon
\end{equation}
for $\epsilon>0$, we can show the following cost gap guarantee between $\hat{\pi}$ and $\pi^\star$.
\begin{corollary}\label{coro: uniform-convergence}
Under condition~\eqref{eq: uniform-convergence}, we have \begin{align*}
      \max_{x\in [M]_+} \Delta(x;\hat{\pi})   \leq c_0(h_\infty + b_\infty) M\big[T\epsilon + T^3 \epsilon^2 \big],
\end{align*}
for some absolute constant $c_0$. 
\end{corollary}

Corollary~\ref{coro: uniform-convergence} is a deterministic reduction from cost gap to model distance, and does not restrict how one obtains the model estimator $\hat{\bF}_t$. It can be directly applied to obtain the sample complexity results in the data-driven inventory control setting with uncensored observations, as we discuss in the following two subsections.

\subsection{Uncensored Non-Stationary Demand via Period-wise SAA}\label{subsec: SAA-independent-demand}

We first apply Corollary~\ref{coro: uniform-convergence} to the uncensored setting, where the demand random variables $\{d_t\}_{t=1}^T$ are independent across periods while their marginal distributions $\{P_t\}_{t=1}^T$ may vary with $t$.

\medskip\noindent\textbf{CDF Estimator.} With uncensored observations $\cD:= \{(y_t^k,d_t^k)\}_{t=1,k=1}^{T,N}$, we construct the empirical environment $\{\hat{\bF}_t\}_{t=1}^T$ via the standard period-wise SAA estimator for each $t:$
\begin{align}\label{eq: F-hat-SAA}
    \hat{F}_t(j) = \frac{1}{N} \sum_{k=1}^N \bm{1}\{d_t^k \leq j\}, \quad \forall j \in [M-1]_+, \qquad \hat F_t(M) := 1.
\end{align}
It is well known that this SAA estimator satisfies the uniform convergence result as the following:
\begin{lemma}[Dvoretzky–Kiefer–Wolfowitz inequality]\label{lem: uniform-convergence-independent} For $\hat{\bF}_t$ defined as in~\eqref{eq: F-hat-SAA}, the uniform convergence condition~\eqref{eq: uniform-convergence} holds with $\epsilon = \sqrt{\frac{\log(2T/\delta)}{2N}}$ with probability at least $1-\delta.$
\end{lemma}

\medskip\noindent\textbf{Sample Complexity Bound.} Applying Corollary~\ref{coro: uniform-convergence} with Lemma~\ref{lem: uniform-convergence-independent} leads to the following sample complexity result:
\begin{theorem}\label{thm: sample-complexity-independent}
With $N$ observations of uncensored trajectories $\{y^k_t,d_t^k\}_{t=1,k=1}^{T,N},$ the policy $\hat\pi$ obtained via DP under~\eqref{eq: F-hat-SAA} satisfies \begin{align*}
 \max_{x\in [M]_+} \Delta(x;\hat{\pi}) \leq c_0(h_\infty + b_\infty) M \big[T\sqrt{\frac{\log(T/\delta)}{N}} + \frac{T^3\log(T/\delta)}{N} \big] 
\end{align*}
for some absolute constant $c_0$ with probability at least $1-\delta.$
\end{theorem}
After a burn-in regime $N \gtrsim T^4$, Theorem~\ref{thm: sample-complexity-independent} states a $\tilde{\cO}(MT/\sqrt{N})$ error bound, which translates into a $\tilde{\cO}(M^2T^2/\varepsilon^2)$ sample complexity upper bound to obtain an $\varepsilon$-accurate policy.
This bound matches those first proposed in \citet{xie2024vc} up to logarithmic factors, which are obtained via VC theory and applied to the ERM estimator\footnote{In a recent revised version, \citet{xie2024vc} has extended their result to the SAA algorithm as ours, but it is still a VC-theory-based analysis}, and improves the $\tilde{\cO}(M^2T^3/\varepsilon^2)$ result in \citet{halman2020provably,qin2023sailing}.

\subsection{Uncensored Stationary Demand via Aggregated SAA}\label{subsec: SAA-independent--identical-demand}

We now apply Corollary~\ref{coro: uniform-convergence} to the uncensored stationary setting, where the demand random variables $\{d_t\}_{t=1}^T$ are i.i.d.\ with a common distribution $P_1=\dots = P_T$.

\medskip\noindent\textbf{CDF Estimator.} With observations $\cD:=\{(y_t^k,d_t^k)\}_{t=1,k=1}^{T,N},$ we construct the empirical environment $\{\hat{\bF}_t\}_{t=1}^T$ by aggregating all samples into a single sample-average approximation (SAA) estimator for each $t:$
\begin{align}\label{eq: F-hat-SAA-identical}
    \hat{F}_t(j) \equiv \hat{F}(j):=  \frac{1}{NT} \sum_{\ell=1}^T\sum_{k=1}^N \bm{1}\{d_\ell^k \leq j\}, \quad \forall j \in [M-1]_+, t\in[T], \qquad \hat F(M) := 1.
\end{align}
Similar to Lemma~\ref{lem: uniform-convergence-independent}, the following uniform convergence guarantee of~\eqref{eq: F-hat-SAA-identical} holds.
\begin{lemma}\label{lem: uniform-convergence-independent-and-identical} For $\hat{\bF}_t$ defined as in~\eqref{eq: F-hat-SAA-identical}, the uniform convergence condition~\eqref{eq: uniform-convergence} holds with $\epsilon = \sqrt{\frac{\log(2/\delta)}{2NT}}$ with probability at least $1-\delta.$
\end{lemma}

\medskip\noindent\textbf{Sample Complexity Bound.} Applying Corollary~\ref{coro: uniform-convergence} with Lemma~\ref{lem: uniform-convergence-independent-and-identical}, we arrive at the following sample complexity result:

\begin{theorem}\label{thm: sample-complexity-independent-identical}
With $N$ observations of uncensored trajectories $\{y^k_t,d_t^k\}_{t=1,k=1}^{T,N},$ the policy $\hat{\pi}$ obtained via DP under~\eqref{eq: F-hat-SAA-identical} satisfies \begin{align*}
 \max_{x\in [M]_+} \Delta(x;\hat{\pi})\leq c_0(h_\infty + b_\infty) M \big[\sqrt{\frac{T\log(2/\delta)}{N}} + \frac{T^2\log(2/\delta)}{N} \big]
\end{align*}
for some absolute constant $c_0$ with probability at least $1-\delta.$
\end{theorem}

After the burn-in regime $N\gtrsim T^3$, Theorem~\ref{thm: sample-complexity-independent-identical} yields the leading policy-learning error $\tilde{\cO}\big(M\sqrt{T/N}\big),$
or equivalently a sample-complexity upper bound of order $\tilde{\cO}(M^2T/\varepsilon^2)$ for obtaining an $\varepsilon$-accurate policy. To the best of our knowledge, this is the first such bound for the stationary-demand multi-period inventory setting.

The improvement comes from a feature that is not captured by existing VC-type analyses. In the non-stationary setting, each period has its own demand distribution, so the natural estimator uses only the $N$ samples from that period. Under stationary demand, however, the $NT$ demand observations are all drawn from the same distribution and can be aggregated into a single CDF estimator. Our model-based decomposition is estimator-agnostic and can directly exploit this aggregation. By contrast, VC-type analyses control a trajectory-level empirical process over policy classes; they average over trajectories rather than over the $NT$ period-level demand observations, and therefore do not directly yield the factor-$T$ gain from across-time aggregation.

We next show that this sharper policy-learning rate is not obtainable through a reduction to policy evaluation. The following lower bound shows that estimating the optimal cost remains statistically harder.

\begin{lemma}\label{lem: policy-evaluation}
There exist cost parameters $\{(h_t,b_t)\}_{t=1}^T$, an initial inventory level $x$, 
and a collection of demand distributions $\cP$ such that the following holds: 
For any algorithm $\cA$ that takes observations 
$\cD := \{(y_t^k,d_t^k)\}_{t=1,k=1}^{T,N}$
from some $\bP \in \cP$ and outputs an estimate of $C_1^\star(x)$,
\begin{align*}
  \max_{\bP \in \cP} \EE_{\cD \sim \bP^{\otimes NT}} 
  \big\lvert \cA(\cD) - C_1^\star(x) \big\rvert
  \ge c_0(h_\infty + b_\infty) M T / \sqrt{N}.
\end{align*}
\end{lemma}

Comparing Lemma~\ref{lem: policy-evaluation} with Theorem~\ref{thm: sample-complexity-independent-identical} gives a strict separation: in this structured inventory class, learning a near-optimal base-stock policy can require fewer samples than estimating the optimal cost $C_1^\star(x)$ to the same accuracy.

This separation is specific to the inventory structure. A common route in proving finite-horizon RL sample complexity bounds is to reduce policy learning to policy evaluation: one first proves a sharp bound for estimating value functions, and then converts it into a sharp policy-learning guarantee. This route is powerful enough to obtain optimal rates for general MDPs, and the corresponding lower-bound arguments can often transfer policy-evaluation hardness to policy-learning hardness by augmenting the instance with non-stationary rewards, as in \citet{ren2021nearly,li2024settling,li2024breaking,xiong2022nearly}. Theorem~\ref{thm: sample-complexity-independent-identical} shows that this otherwise sharp RL template is too coarse for multi-period inventory control. Even with non-stationary costs, the base-stock policy class and the demand-induced transition dynamics make policy learning strictly simpler than evaluating the optimal value. This complements recent work showing that inventory systems can be statistically simpler than generic MDPs \citep{fan2024don,xie2024vc,zhang2025reinforcement,fan2025sample}.

This observation also clarifies the lower-bound argument of \citet{qin2023sailing} for the stationary-demand setting. Their claimed $\Omega(T/\sqrt N)$ policy-learning lower bound is based on first proving a lower bound for evaluating $C_1^\star(\cdot)$ and then transferring it to policy optimization. Lemma~\ref{lem: policy-evaluation} confirms the evaluation lower bound, but Theorem~\ref{thm: sample-complexity-independent-identical} shows that the transfer step is not valid in this inventory setting. Together with the matching lower bound in Theorem~\ref{thm: lower-bound-independent-identical}, our results identify the correct optimal policy-learning rate as $\Theta(\sqrt{T/N})$ up to logarithmic and burn-in terms.

\section{Offline Policy Learning via Upper-Biased SAA}\label{sec: censored-demand}

In this section, we study offline policy learning from censored observations. Motivated by Questions~1 and~2, our goals are twofold: to identify the instance-dependent coverage condition under which a fixed censored dataset is informative enough for policy learning, and to design a model-based algorithm that attains the optimal rate under this condition.

The key algorithmic idea proposed in this section is to introduce a one-sided upper bias into the estimated demand CDF. This upper-biased CDF acts as a model-level form of pessimism: it induces downward-biased base-stock levels, thereby avoiding poorly covered high-demand coordinates. Combined with the cost-decomposition identity in Theorem~\ref{thm: cost-decomposition}, this bias helps localize the effect of estimation error to the coordinates that are relevant along the optimal-policy trajectory, leading to an optimal-policy-dependent coverage condition.

In the following subsections, we first establish the general construction and statistical properties of biased CDF estimators in Section~\ref{subsec: biased-SAA-properties}. We then show how the cost decomposition result in Theorem~\ref{thm: cost-decomposition} can be applied to obtain the sample complexity bounds in Section~\ref{subsec: biased-SAA-offline} and~\ref{subsec: biased-SAA-offline-identical}.

\subsection{Biased SAA: Construction and Statistical Properties}\label{subsec: biased-SAA-properties}

In this subsection, we introduce biased CDF estimators constructed from a
general censored, adaptively collected dataset $\{(y_k,\bar d_k)\}_{k=1}^n$. Although the offline algorithm below only uses the upper-biased estimator, we also define lower-biased versions here for consistency and later use in online policy design.

Here, by adaptively collected, we mean that for the filtration $\{\cF_k\}_{k=1}^n$ generated by all observations from $1$ to $k-1$, it holds that \begin{enumerate}[nosep]
    \item $\{y_k\}_{k=1}^n$ is predictable with respect to $\{\cF_k\}_{k=1}^n$
    \item Conditional on $\cF_k,$ the uncensored demand $d_k$ follows a fixed distribution $\bP$ with CDF $\bF.$
\end{enumerate}  

With such observations, we set $n_{j}:= \sum_{k=1}^n \bm{1}\{y_k >  j\},$ and the sub-sampled SAA estimator as \begin{align}\label{eq: partial-CDF-mean}
    \tilde{F}(j)&:= \begin{cases}
     n^{-1}_{j}\sum_{k = 1}^n \bm{1}\{y_k>j, \bar{d}_k \leq j \}  , &\text{if }n_j\geq 1,\\
        1, & \text{otherwise.}
    \end{cases}
\end{align}
We first recall the following standard Freedman-type inequality for~\eqref{eq: partial-CDF-mean}(cf. Lemma~3 in \citet{rakhlin2011making}):
\begin{proposition}\label{prop: uniform-concentration}
With the adaptively collected observations $\{(y_k, \bar{d}_k)\}_{k = 1}^{n}$ and $\tilde{F}$ defined as in~\eqref{eq: partial-CDF-mean}, it holds with probability at least $1-\delta$ that
\begin{equation}\label{eq: partial-CDF-confidence-bound-general}
    \lvert \tilde{F} (j) - F(j) \rvert \leq \underbrace{ 4\sqrt{\frac{F(j) \big(1-F(j)\big) \log (Mn/\delta)}{n_{j}\vee 1}}  + \frac{4\log(Mn/\delta)}{n_{j}\vee 1}}_{:= \cC_{j}},
\end{equation}
uniformly over all $j\in [M]_+.$
\end{proposition}

With~\eqref{eq: partial-CDF-mean} and~\eqref{eq: partial-CDF-confidence-bound-general}, we define the empirical confidence bound as
\begin{align}\label{eq: def-tilde-C}
     \tilde\cC_j := 8 \sqrt{\frac{\tilde F(j)\big(1 - \tilde F(j)\big) \log(Mn/\delta)}{n_j\vee 1}} + \dfrac{24 \log(Mn/\delta)}{n_j\vee 1}.
\end{align}
Using $\tilde{\cC}_{j}$, we define the upper- and lower-biased estimators $\bF^{\mathrm{UCB}}$ and $\bF^{\mathrm{LCB}}$ sequentially for each $j\in [M-1]_+$,
\begin{align}
    \textbf{Upper-Biased Estimator:} \qquad& F^{\mathrm{UCB}}(j):= \min\{1, \max\{ F^{\mathrm{UCB}}(j-1), \tilde{F}(j)+ \tilde{\cC}_{j} \}\}, \label{eq: def-UCB} \\
    \textbf{Lower-Biased Estimator:} \qquad & F^{\mathrm{LCB}}(j):= \max\{ F^{\mathrm{LCB}}(j-1), \tilde{F}(j)- \tilde{\cC}_{j} \}. \label{eq: def-LCB}
\end{align}
with boundary values $F^\mathrm{UCB}(-1) = F^\mathrm{LCB}(-1) = 0,  F^\mathrm{UCB}(M) = F^\mathrm{LCB}(M) = 1$.

The following bias and closeness result holds for the biased estimators.
\begin{proposition}\label{prop: biased-F}
    For each $j \in [M]_+,$ under the event~\eqref{eq: partial-CDF-confidence-bound-general}, it holds that \begin{align}
       F(j) \leq F^\mathrm{UCB}(j) \leq F(j) +  c_0\cC_j, \quad
        F(j) \geq F^\mathrm{LCB}(j) \geq F(j) - c_0\cC_j,
    \end{align}
    for an absolute constant $c_0>0$.
\end{proposition}
\subsection{Sample Complexity under Censored Feedback}\label{subsec: biased-SAA-offline}

In this subsection, we apply the construction introduced in Section~\ref{subsec: biased-SAA-properties} to the episodic censored observations $\cD:= \{(y_t^k, \bar{d}_{t}^k)\}_{k,t=1}^{N,T}$ to obtain the upper-biased SAA algorithm for offline policy learning, as presented in Algorithm~\ref{alg: censored-offline}.

In Algorithm~\ref{alg: censored-offline}, for each $t\in [T],$ we construct the period-wise filtered SAA estimator as \begin{align}\label{eq: partial-CDF-offline-mean}
    \tilde{F}_t(j)&:= \begin{cases}
     N^{-1}_{t,j}\sum_{k = 1}^N \bm{1}\{y_t^k>j, \bar{d}_t^k \leq j \}  , &\text{if }N_{t,j}\geq 1,\\
        1, & \text{otherwise,}
    \end{cases}\qquad N_{t,j}:= \sum_{k = 1}^N \bm{1}\{y_t^k > j\}.
\end{align}
The corresponding upper-biased estimators $\{\bF_t^\mathrm{UCB}\}_{t=1}^T$ are given as
\begin{equation}\label{eq: UCB-CDF-offline}
\begin{aligned}
   &F_t^{\mathrm{UCB}}(j):= \min\{1, \max\{ F_t^{\mathrm{UCB}}(j-1), \tilde{F}_t(j)+ \tilde{\cC}_{t,j} \}\}, \quad\forall j \in [M-1]_+, \\
   &F_t^{\mathrm{UCB}}(-1) := 0, \quad F_t^{\mathrm{UCB}}(M) := 1,
\end{aligned}
\end{equation}
with
\begin{align}\label{eq: CB-algorithm-offline}
 \tilde\cC_{t,j} := 8 \sqrt{\frac{\tilde F_t(j)\big(1 - \tilde F_t(j)\big) \log(MTN/\delta)}{N_{t,j} \vee 1}} + \dfrac{24 \log(MTN/\delta)}{N_{t,j}\vee 1}.
\end{align}

For each $t\in [T]$, the construction in~\eqref{eq: partial-CDF-offline-mean} and~\eqref{eq: UCB-CDF-offline} is the direct application of~\eqref{eq: partial-CDF-mean} and~\eqref{eq: def-UCB} to the $t$-period data subset $\{(y_t^k,\bar{d}_t^k)\}_{k=1}^N.$ In particular, by problem formulation it satisfies the adaptive requirements specified in Section~\ref{subsec: biased-SAA-properties}, thus the statistical guarantee provided in Proposition~\ref{prop: biased-F} can be applied directly to show that 
\begin{align}\label{eq: partial-CDF-confidence-bound-UCB}
    F_t(j)\leq F^{\mathrm{UCB}}_t(j)\leq F_t(j) + c_0 \bigg( 4\sqrt{\frac{F_t(j) \big(1-F_t(j)\big) \log (MTN/\delta)}{N_{t,j}\vee 1}}  + \frac{4\log(MTN/\delta)}{N_{t,j}\vee 1}\bigg),
\end{align}
holds uniformly for all $ t\in [T], j\in [M]_+$ with probability at least $1-\delta.$

With the biased estimators $\{\bF^{\mathrm{UCB}}_t\}_{t = 1}^T,$ we then define the base-stock policy $\pi = \{s_t\}_{t=1}^T$ induced by the upper-biased CDFs, i.e., the policy obtained by solving the DP under $\{\bF^{\mathrm{UCB}}_t\}_{t = 1}^T$.
An important consequence is that the bias in the CDF estimator induces a corresponding bias in the opposite direction on the resulting base-stock policies:

\begin{lemma}\label{lem: policy-ordering-UCB} Let $\pi = \{s_t\}_{t=1}^T$ be the base-stock policy obtained by solving the DP under the upper-biased CDF sequence $\{\bF_t^\mathrm{UCB}\}_{t=1}^T$. If the event in~\eqref{eq: partial-CDF-confidence-bound-UCB} holds, then $s_t \leq s_t^\star$ for all $t \in [T].$
\end{lemma}

\begin{algorithm}[t]
\caption{DP-UCB}
\label{alg: censored-offline}
\begin{algorithmic}[1]
\REQUIRE Censored dataset $\{y^k_t,\bar{d}_{t}^k\}_{k,t=1}^{N,T}$, cost coefficient sequence $\{(h_t,b_t)\}_{t=1}^T$, confidence level $\delta.$ 
\FOR{$t=1,2,\ldots,T$}
    \STATE \textbf{Step 1 (Compute the CDF estimator):} Compute $\tilde{\bF}_t$ as in~\eqref{eq: partial-CDF-offline-mean}.
    \STATE \textbf{Step 2 (Compute the upper-biased CDF estimator):} Construct $\bF^\mathrm{UCB}_t$ from $\tilde{\bF}_t$ as in~\eqref{eq: UCB-CDF-offline}.
\ENDFOR
\STATE Compute $\pi := \{s_t\}_{t=1}^T$ by solving the DP under $\{(h_t,b_t,\bF^\mathrm{UCB}_t)\}_{t=1}^T$.
\RETURN $\pi$.
\end{algorithmic}
\end{algorithm}

To see at a high level how the bias structure controls the cost gap, consider the terms $\langle \bv_t, \Delta \bA_t(s_{t+1}) \hat{\bD}_{t+1} + \Delta \bc_t\rangle$ in Theorem~\ref{thm: cost-decomposition}.
First, since $\bF_t^{\mathrm{UCB}} \geq \bF_t$, we have $\Delta \bc_t = (h_t + b_t)\Delta \bF_t \leq 0$ entry-wise; moreover, by the Abel-summation argument in the proof of Corollary~\ref{coro: uniform-convergence}, the upper bias $\Delta \bF_t \leq 0$ together with the monotonicity of $\hat{\bD}_{t+1}$ (convexity of the cost-to-go) yields $\Delta \bA_t(s_{t+1}) \hat{\bD}_{t+1} \leq 0$ entry-wise as well.
Second, Lemma~\ref{lem: policy-ordering-UCB} gives the downward-biased base-stock levels $s_t \leq s_t^\star$, so by Remark~\ref{remark: v-interpretation} we have $-\PP(y_t^\star > j\lvert x_1 = 0) \leq \big[\bv_t\big]_j \leq 0$.
Combining these two observations, \begin{align*}
        \langle \bv_t, \Delta \bA_t(s_{t+1}) \hat{\bD}_{t+1}  + \Delta \bc_t\rangle \leq -\sum_{j = 0}^{M-1} \PP(y^\star_t > j\lvert x_1 = 0) \big[\Delta \bA_t(s_{t+1}) \hat{\bD}_{t+1}  + \Delta \bc_t\big]_j,
\end{align*}
so only the estimation error weighted by the optimal-policy coverage $\PP(y_t^\star > j\lvert x_1 = 0)$ matters. To formalize this, we introduce the \emph{effective sample size}:
\begin{align}\label{eq: def-effective-sample-size}
    \cN^\star := \bigg(\frac{1}{MT}\sum_{t\in [T]}\sum_{j\in [M-1]_+}\dfrac{ \PP(y_t^\star > j \lvert x_1 = 0)}{N_{t,j}\vee 1}\bigg)^{-1},
\end{align}
which is the effective sample size of the censored dataset relative to the optimal policy: it aggregates over all coordinates the inverse count $1/(N_{t,j}\vee1)$, each weighted by the optimal-trajectory visiting probability $\PP(y_t^\star > j\lvert x_1=0)$, so that $\cN^\star$ is governed by the coverage ratio weighted by the optimal trajectory's visiting probabilities. Thus the dataset need not cover all demand coordinates uniformly---it only needs sufficient coverage on the coordinates visited by the optimal policy.

Now we are ready to state the sample complexity result with respect to $\cN^\star:$
\begin{theorem}\label{thm: censored-sample-complexity}
There exists an absolute constant $c_0$ so that  with probability at least $1-\delta$, the output policy $\pi$ of Algorithm~\ref{alg: censored-offline} satisfies
\begin{align*}
    \max_{x\in [M]_+} \Delta(x;\pi) \leq c_0 (h_\infty + b_\infty) M  \bigg[T \sqrt{\frac{\log(MTN/\delta)\log(eN)}{\cN^\star}} + \frac{T^2\log(MTN/\delta)\log(eN)}{\cN^\star}\bigg].
\end{align*}
\end{theorem}

\medskip\noindent\textbf{Fixed Behavior Policy Setting.} A special case of the adaptively collected setting is when all offline data are sampled by a fixed behavior policy $\pi^b$ repeatedly for $N$ episodes, starting at the initial inventory level $x_1 = 0$. Define the policy-level coverage ratio \begin{align*}
    \cC^\star:= \frac{1}{MT}\sum_{t=1}^T \sum_{j=0}^{M-1} \frac{\PP(y^\star_t > j\lvert x_1 = 0)}{\PP(y^b_t > j\lvert x_1 = 0)}
\end{align*}
for $y_t^b$ the post-ordering inventory level at each $t$ under $\pi^b$. We have the following corollary of Theorem~\ref{thm: censored-sample-complexity} stated with respect to $\cC^\star:$

\begin{corollary}\label{corollary: C-star-sample-complexity}
There exist absolute constants $c_0$ so that with probability at least $1-\delta$, the output policy $\pi$ of Algorithm~\ref{alg: censored-offline} satisfies
\begin{align*}
    \max_{x\in [M]_+} \Delta(x;\pi) \leq c_0 (h_\infty + b_\infty) M \cdot \bigg[T \sqrt{\frac{\cC^\star \log^2(MTN/\delta)\log(eN)}{N}} + \frac{\cC^\star T^2 \log^2(MTN/\delta)\log(eN)}{N}\bigg].
\end{align*}
\end{corollary}

\medskip\noindent\textbf{Comparison to Offline Reinforcement Learning.}
Compared with the familiar $\tilde{\Theta}(T^{3/2}\sqrt{\cC/N})$ scaling in offline reinforcement learning with single-policy concentrability \citep{xiong2022nearly,nguyen2023sample}, with $\cC^\star$ playing the role of the concentrability coefficient, Corollary~\ref{corollary: C-star-sample-complexity} achieves the sharper horizon dependence $\tilde{\Theta}(T\sqrt{\cC^\star/N})$, up to the inventory-specific factor $(h_\infty+b_\infty)M$. This improvement reflects the structural advantage of inventory systems over general MDPs, already visible in the uncensored analysis of Section~\ref{sec: uniform-convergence}.

A more fundamental distinction lies in the coverage definition itself. In offline RL, single-policy concentrability is defined via visitation measure ratios, involving the probability of visiting a specific state (an equality condition). In contrast, $\cC^\star$ involves survival probabilities (an inequality condition). The inequality-based version is more relaxed: it is easier for the behavior policy to cover the optimal policy under this measure, since the survival probability aggregates over multiple demand levels. This relaxation arises from the censored feedback structure---rather than observing exact state-action pairs as in RL, one only observes whether demand exceeds the inventory level, which naturally induces coverage through survival probabilities.

\subsection{Improved Sample Complexity under Stationary Demand}\label{subsec: biased-SAA-offline-identical}

We next consider the stationary-demand setting of Section~\ref{subsec: SAA-independent--identical-demand}, where the demand variables remain independent across periods but share a common marginal distribution $P_1=\cdots=P_T$. As in the uncensored case, stationarity lets us aggregate observations across periods; the only difference is that, under censoring, the number of informative observations still varies with the coordinate $j$. We show that an aggregated version of Algorithm~\ref{alg: censored-offline} achieves an improved sample complexity bound.

\medskip\noindent\textbf{Aggregated CDF Estimator.} With the observed dataset $\{(y_t^k,\bar{d}_t^k)\}_{k,t = 1}^{N,T}$ and aggregated sample size $N_{\mathsf{agg},j}:= \sum_{t=1}^T N_{t,j}$, we define the aggregated version of CDF estimator in~\eqref{eq: partial-CDF-offline-mean} as 
\begin{align}\label{eq: censored-CDF-mean-aggregated}
    \tilde{F}_{\mathsf{agg}}(j):= \begin{cases} N_{\mathsf{agg},j}^{-1} \sum_{k,t=1}^{N,T} \bm{1}\{y_t^k > j, \bar{d}_t^k \leq j \}, &\text{ if } N_{\mathsf{agg},j} \geq 1,\\
        1, &\text{ otherwise,}
    \end{cases}
\end{align}
and its biased version as
\begin{equation}\label{eq: UCB-CDF-offline-aggregated}
\begin{aligned}
    &F_{\mathsf{agg}}^{\mathrm{UCB}}(j):= \min\{1, \max\{ F_{\mathsf{agg}}^{\mathrm{UCB}}(j-1), \tilde{F}_{{\mathsf{agg}}}(j)+ \tilde{\cC}_{{\mathsf{agg}},j} \}\}, \forall j \in [M-1]_+, \\
    &F_{\mathsf{agg}}^{\mathrm{UCB}}(-1) := 0, \quad F_{\mathsf{agg}}^{\mathrm{UCB}}(M) := 1,
\end{aligned}
\end{equation}
with
\begin{align*}
\tilde{\cC}_{{\mathsf{agg}},j} := 8 \sqrt{\frac{\tilde{F}_\mathsf{agg}(j) \big(1 - \tilde{F}_\mathsf{agg}(j)\big) \log(MNT/\delta)}{N_{\mathsf{agg},j}\vee 1}} + \frac{24\log(MNT/\delta)}{N_{\mathsf{agg},j}\vee 1}.
\end{align*}

With the aggregated estimators above, we apply the modified version of Algorithm~\ref{alg: censored-offline} that sets $\bF_t^{\mathrm{UCB}} \equiv \bF_{\mathsf{agg}}^{\mathrm{UCB}}$ for all $t\in[T]$ (replacing Steps~1 and~2 by~\eqref{eq: censored-CDF-mean-aggregated} and~\eqref{eq: UCB-CDF-offline-aggregated}, respectively) and solves the DP under this common CDF sequence. This yields the following guarantee with respect to the \textit{aggregated effective sample size}:
\begin{align}\label{eq: def-aggregated-effective-sample-size}
    \cN_{\mathsf{agg}}^\star:= \bigg(\frac{1}{MT}{\sum_{t\in [T]}\sum_{j\in [M-1]_+}\dfrac{ \PP(y_t^\star > j \lvert x_1 = 0)}{N_{\mathsf{agg},j}\vee 1}}\bigg)^{-1}.
\end{align}
\begin{theorem}\label{thm: censored-sample-complexity-identical}
For the aggregated version of Algorithm~\ref{alg: censored-offline} that sets $\bF_t^{\mathrm{UCB}} \equiv \bF_{\mathsf{agg}}^{\mathrm{UCB}}$ for all $t\in[T]$, there exists an absolute constant $c_0 > 0$ so that with probability at least $1-\delta$ its output $\pi$ satisfies
\begin{align*}
    \max_{x\in [M]_+} \Delta(x;\pi) \leq c_0 (h_\infty + b_\infty) M  \bigg[T \sqrt{\frac{\log(MNT/\delta)\log(eNT)}{\cN_{\mathsf{agg}}^\star}} + \frac{T^2\log(MNT/\delta)\log(eNT)}{\cN_{\mathsf{agg}}^\star}\bigg].
\end{align*}
\end{theorem}

In the uncensored setting, every coordinate is revealed in every trajectory, so $N_{t,j} = N$ and $N_{\mathsf{agg},j} = NT$ for all $t\in [T], j \in [M-1]_+$; hence $\cN^\star \geq N$ and $\cN_{\mathsf{agg}}^\star \geq NT$, and Theorems~\ref{thm: censored-sample-complexity} and~\ref{thm: censored-sample-complexity-identical} recover the uncensored rates of Theorems~\ref{thm: sample-complexity-independent} and~\ref{thm: sample-complexity-independent-identical}. Thus stationarity provides a factor-$T$ gain in the effective sample size and a factor-$\sqrt{T}$ improvement in the leading policy-learning error.

\subsection{Sample Complexity Lower Bounds}

In this subsection, we establish sample-complexity lower bounds showing that the dependence on $M,T,h_\infty,b_\infty$ and the relevant effective sample sizes in Theorems~\ref{thm: censored-sample-complexity} and~\ref{thm: censored-sample-complexity-identical} is optimal in a minimax sense over coverage classes, up to logarithmic and burn-in terms. The constructions allow an arbitrary coverage fraction $p\in(0,1]$, with $p=1$ recovering the uncensored setting. To the best of our knowledge, these are the first lower bounds for multi-period inventory policy learning, even under uncensored observations.

In the non-stationary demand setting, we have the following result.

\begin{theorem}\label{thm: lower-bound-independent}
For any $c_\infty>0$, $M\geq 2$, $T \geq 3$, coverage fraction $p\in(0,1]$, and sufficiently large $N$ so that $Np \gtrsim T^2$, there exist cost parameters $\{(h_t,b_t)\}_{t=1}^T$, an initial inventory level $x$,
a collection of demand distribution sequences $\cP$, and a deterministic sequence of offline censoring levels $\{y^k_{t}\}_{t,k=1}^{T,N}$ such that the following hold:
\begin{enumerate}[nosep]
    \item $h_\infty = b_\infty = c_\infty$.
    \item $\cN^\star \geq \lfloor Np \rfloor$ for every instance in $\cP$. 
    \item For any algorithm $\cA$ that takes the censored dataset
$\cD := \{(y_t^k, \bar{d}_t^k)\}_{t=1,k=1}^{T,N}$ as input and outputs a policy $\pi = \cA(\cD)$,
\begin{align*}
  \max_{\{P_t\}_{t=1}^T \in \cP} \EE_{\{d_t^k\}_{t,k=1}^{T,N} \sim (\prod_{t=1}^T P_t)^{\otimes N}}
  \big[\Delta (x;\cA(\cD))\big]
  \geq c_0 c_\infty M T / \sqrt{Np}
\end{align*}
for some absolute constant $c_0$, where, for each $t\in[T]$, the demands $\{d_t^k\}_{k=1}^N$ are sampled i.i.d.\ from $P_t$, independently across periods, and the algorithm observes only the censored value $\bar d_t^k = \min\{y_t^k, d_t^k\}$.
\end{enumerate}
\end{theorem}
By combining points~2 and~3 of Theorem~\ref{thm: lower-bound-independent}, we have even when $\cN^\star \geq Np$ uniformly over the constructed class, the minimax sub-optimality gap is at least $\Omega(MT/\sqrt{Np})$, which indicates the optimality of Theorem~\ref{thm: censored-sample-complexity} up to logarithmic factors in minimax sense.

In the stationary demand setting, we have the following result.

\begin{theorem}\label{thm: lower-bound-independent-identical}
For any $c_\infty>0$, $M\geq 2$, $T\geq 3$, coverage fraction $p\in(0,1]$, and sample size $N$ sufficiently large so that $Np \gtrsim T$, there exist cost parameters $\{(h_t,b_t)\}_{t=1}^T$, an initial inventory level $x$, a collection of stationary demand distributions $\cP$, and a deterministic sequence of offline censoring levels $\{y_t^k\}_{t,k=1}^{T,N}$ such that the following hold:
\begin{enumerate}[nosep]
    \item $h_t=b_t=c_\infty$ for all $t\in[T]$.
    \item $\cN_{\mathsf{agg}}^\star \geq \lfloor NTp\rfloor$ for every instance in $\cP$.
    \item For any algorithm $\cA$ that takes the censored dataset $\cD:=\{(y_t^k,\bar d_t^k)\}_{t=1,k=1}^{T,N}$ as input and outputs a policy $\pi=\cA(\cD)$,
    \begin{align*}
      \max_{P\in\cP} \EE_{\{d_t^k\}_{t,k=1}^{T,N} \sim P^{\otimes NT}} \big[\Delta(x;\cA(\cD))\big] \geq c_0 c_\infty M \sqrt{\frac{T}{Np}},
    \end{align*}
    for some absolute constant $c_0>0$, where all demands $\{d_t^k\}_{t,k}$ are sampled i.i.d.\ from $P$, and the algorithm observes only the censored value $\bar d_t^k=\min\{y_t^k,d_t^k\}$.
\end{enumerate}
\end{theorem}
Similar to the non-stationary case, combining points~2 and~3 of Theorem~\ref{thm: lower-bound-independent-identical} shows that even when $\cN_{\mathsf{agg}}^\star \geq NTp$, the best achievable minimax sub-optimality gap is of order $\Omega(M\sqrt{T/(Np)})$, indicating the optimality of Theorem~\ref{thm: censored-sample-complexity-identical} up to logarithmic factors in minimax sense.

\section{Online Policy Design via Lower-Biased SAA}\label{sec: extensions-online}

In this section, we turn to the online counterpart of the offline censored problem studied in Section~\ref{sec: censored-demand}. Unlike the offline setting, where the censoring levels are fixed before learning begins, the online learner chooses its own policies over $K$ episodes, and the data collected by these policies determine which parts of the demand distribution become observable. The learner therefore faces the two coupled tasks described in Question~3: it must incur low inventory cost while actively generating the coverage needed for policy learning.

Our answer is a lower-biased counterpart of DP-UCB, which we call DP-LCB and summarize in Algorithm~\ref{alg: censored-online}. Whereas the upper-biased CDF in the offline algorithm implements pessimism and pushes the learned base-stock levels downward, the lower-biased CDF used here implements optimism: by underestimating the CDF, the model expects larger demand and prescribes larger base-stock levels. This optimistic over-ordering is useful for exploration, since higher inventory levels reveal demand at more coordinates. 

In the following subsections, we first provide a self-coverage-based regret analysis in Section~\ref{subsec: online-dp-lcb}. We then present the aggregated version with improved regret under stationary demand in Section~\ref{subsec: online-stationary}, and state matching regret lower bounds in Section~\ref{subsec: online-lower-bound}.

\subsection{Regret Analysis Through Self-Coverage Property}\label{subsec: online-dp-lcb}

\noindent\textbf{Episodic Regret.} The online learner is evaluated by cumulative regret. It executes a sequence of policies $\pi^{(1)}, \pi^{(2)}, \dots, \pi^{(K)}$ over $K$ episodes with possibly different initial inventory levels, where each $\pi^{(k)}$ is computed from the first $k-1$ episodes of data. Denoting by $C_1^{(k)}(x):=C_1^{\pi^{(k)}}(x)$ the cost of policy $\pi^{(k)}$ from initial inventory $x$, we define the \emph{episodic regret} as
\begin{align}\label{eq: def-online-regret}
    \mathrm{Regret}(K) := \sum_{k=1}^K \max_{x \in [M]_+} \Delta (x;\pi^{(k)}).
\end{align}

\medskip\noindent\textbf{Lower-biased SAA.}
For each episode $k$, let
\begin{align}\label{eq: def-online-counts}
    N_{t,j}^{(k)} := \sum_{\ell=1}^{k-1} \bm{1}\{y_t^{(\ell)} > j\}, \qquad t\in[T],\ j\in[M-1]_+,
\end{align}
be the number of observations collected up to episode $k$ that reveal coordinate $j$ at period $t$. From these data we construct the filtered SAA estimator and the lower-biased CDF estimator exactly as in Section~\ref{subsec: biased-SAA-properties}, with the UCB construction~\eqref{eq: def-UCB} replaced by the LCB construction~\eqref{eq: def-LCB}. More precisely, at the $k$-th episode, with the filtered SAA estimator\begin{align*}
     \qquad \tilde{F}_t^{(k)}(j):=\begin{cases}
    \dfrac{1}{N_{t,j}^{(k)}} \sum_{\ell=1}^{k-1} \bm{1}\{y_t^{(\ell)} > j, \bar{d}_t^{(\ell)} \leq j\}, \quad &\text{ if } N_{t,j}^{(k)} \geq 1,\\
    1, & \text{ otherwise,}
    \end{cases} 
\end{align*}
the LCB estimator is constructed as 
\begin{equation}\label{eq: LCB-online-construction}
\begin{aligned}
    &F_t^{\mathrm{LCB},(k)}(j):= \min\{1, \max\{F_t^{\mathrm{LCB},(k)}(j-1),\tilde{F}_t^{(k)}(j) - \tilde{\cC}_{t,j}^{(k)}\}\}, \forall j \in [M-1]_+,\\
    & F_t^{\mathrm{LCB},(k)}(-1):=0,\quad F_t^{\mathrm{LCB},(k)}(M):=1,
\end{aligned}    
\end{equation}
with \begin{align*}
    \tilde{\cC}_{t,j}^{(k)}:= 8\sqrt{\frac{\tilde{F}_t^{(k)}(j)(1- \tilde{F}_t^{(k)}(j) ) \log(MTK^2/\delta)}{N_{t,j}^{(k)}\vee 1}} + \frac{24\log(MTK^2/\delta)}{N_{t,j}^{(k)}\vee 1}.
\end{align*}
\begin{algorithm}[t]
\caption{DP-LCB}
\label{alg: censored-online}
\begin{algorithmic}[1]
\REQUIRE Horizon $T$, number of episodes $K$, confidence level $\delta$, cost coefficients $\{(h_t,b_t)\}_{t=1}^T$, and the LCB construction rule in~\eqref{eq: def-LCB}.
\STATE \textbf{Initialize:} set $\hat{F}^{(1)}_{t}(j) \equiv 0$ for all $t\in[T]$ and $j\in[M-1]_+$, and $\hat{F}^{(1)}_{t}(M) = 1$.
\FOR{$k=1,2,\ldots,K$}
    \STATE \textbf{Compute optimistic policy:} Solve the DP under $\{\hat{\bF}^{(k)}_{t}\}_{t=1}^T$ to obtain the base-stock policy $\pi^{(k)} = \{s_t^{(k)}\}_{t=1}^T$, where for $k=1$, we break ties by setting $s_t^{(1)} = M, \forall t \in [T].$
    \STATE \textbf{Execute policy:} execute $\pi^{(k)}$ for one episode and observe the censored trajectory $\{(y^{(k)}_t, \bar d_t^{(k)})\}_{t=1}^T$, where $\bar d_t^{(k)}$ is the censored demand observation.
    \STATE \textbf{Construct fresh LCB model:} using all data collected up to and including episode $k$, construct the lower-biased CDF sequence $\{\bF^{\mathrm{LCB},{(k+1)}}_{t}\}_{t=1}^T$ by applying~\eqref{eq: def-LCB} period-wise, with the confidence radius scaled by $\log(MTK^2/\delta)$ to union-bound over the $K$ episodes.
    \STATE \textbf{Monotone update:} for all $t\in[T]$ and $j\in[M-1]_+$, set
    \begin{equation}\label{eq: LCB-algorithm-modified}
        \hat{F}^{(k+1)}_t(j) := \max\big\{\hat{F}^{(k)}_t(j),\ F^{\mathrm{LCB},(k+1)}_t(j)\big\},
    \end{equation}
    and set $\hat{F}^{{(k+1)}}_t(M) = 1$.
\ENDFOR
\RETURN policies $\{\pi^{(k)}\}_{k=1}^K$ and the collected dataset $\{(y^{(k)}_t, \bar d_t^{(k)})\}_{t,k=1}^{T,K}$.
\end{algorithmic}
\end{algorithm}
To ensure the monotone property of the constructed sequence across $k$, we add a truncation step as in~\eqref{eq: LCB-algorithm-modified}, this monotonicity property ensures the following inductive bias on the sequence of base-stock policies.
\begin{lemma}[LCB optimism and monotonicity]\label{lem: online-policy-monotonicity}
For any $k\in [K]$, the CDF estimates generated by Algorithm~\ref{alg: censored-online} satisfy
\begin{align*}
    \bm{0} {= \hat{\bF}_t^{(1)} \leq \hat{\bF}_t^{(2)}} \leq \cdots \leq \hat{\bF}_t^{(k)} \leq \bF_t, \qquad \forall t\in[T]
\end{align*}
with probability at least $1-\delta.$
Consequently, the induced base-stock levels form the monotone sequence
\begin{align}\label{eq: online-base-stock-sequence}
    s_t^\star \leq s_t^{(k)} \leq s_t^{(k-1)} \leq \cdots \leq s_t^{(1)}{= M}, \qquad \forall t\in[T].
\end{align}
\end{lemma}
\noindent In particular, let
\begin{equation}\label{eq: LCB-confidence-set}
\cE^{(k)}:= \bigg\{ F_t(j) - c_0\,{\cC}_{t,j}^{(\ell)} \leq \hat{F}_t^{(\ell)}(j) \leq F_t(j), \quad \forall t\in[T], j\in[M-1]_+, 1\leq \ell \leq k \bigg\},    
\end{equation}
with $c_0$ the absolute constant of Proposition~\ref{prop: biased-F} and \begin{align*}
    {\cC}_{t,j}^{(\ell)} :=  8\sqrt{\frac{{F}_t(j)(1- {F}_t(j) ) \log(MTK^2/\delta)}{N_{t,j}^{(\ell)}\vee 1}} + \frac{24\log(MTK^2/\delta)}{N_{t,j}^{(\ell)}\vee 1}.
\end{align*}
By Proposition~\ref{prop: biased-F}, $\cE^{(K)} = \cap_{k=1}^K{\cE}^{(k)}$  holds with probability at least $1-\delta$.

To see how this optimistic bias enters the cost decomposition, fix an episode $k$ and denote $\cF_k$ the sigma-algebra generated by all information up to the beginning of $k$-th episode, and 
\begin{align*}
    q^{(k)}_{t,j}:= \PP(y_{t}^{(k)} > j  \lvert x_1 = 0)
\end{align*}
the corresponding visitation probability of $\pi^{(k)}$, where the probability is taken over the $k$-th episode demand distributions. In particular, $q_{t,j}^{(k)}$ is a random variable depending on all historical information in past $k-1$ episodes and is  measurable under $\cF_k$. Conditional on $\cF_k$, applying Theorem~\ref{thm: cost-decomposition} to $\pi^{(k)}$ leads to \begin{align*}
\max_{x\in [M]_+} \Delta(x;\pi^{(k)}) \leq  \sum_{t=1}^T \big\langle \bv_t^{(k)},\, \Delta\bA_t^{(k)}(s_{t+1}^{(k)})\hat{\bD}_{t+1}^{(k)} + \Delta\bc_t^{(k)} \big\rangle
\end{align*}
with $\hat{\bD}_t^{(k)}, \Delta \bc_t^{(k)}, \bA_t^{(k)}, \bv_{t}^{(k)}$ the corresponding quantities in Table~\ref{tab: section3-notation} associated with $\{\hat{\bF}_t^{(k)}\}_{t\in [T]}$.
Now suppose $\cE^{(k)}$ holds, by Lemma~\ref{lem: online-policy-monotonicity}, $\hat{\bF}_t^{(k)} \leq \bF_t$, and hence
$    \Delta\bc_t^{(k)} := (h_t+b_t)\big(\bF_t-\hat{\bF}_t^{(k)}\big) \geq \bm{0}$ entry-wise. 
Moreover, by the same summation-by-parts argument used in the proof of Corollary~\ref{coro: uniform-convergence}, the lower CDF bias together with the monotonicity of $\hat{\bD}_{t+1}^{(k)}$ yields $\Delta\bA_t^{(k)}(s_{t+1}^{(k)})\hat{\bD}_{t+1}^{(k)} \geq 0$ entry-wise, so the model-error residual in Theorem~\ref{thm: cost-decomposition} is entry-wise non-negative. On the other hand, Lemma~\ref{lem: online-policy-monotonicity} gives $s_t^{(k)} \geq s_t^\star$, so the deployed post-ordering inventory stochastically dominates the optimal one, and by Remark~\ref{remark: v-interpretation},
\begin{align}\label{eq: v-online-sign}
    [\bv_t^{(k)}]_j = \PP(y_t^\star \leq j\lvert x_1 = 0) - \PP(y_t^{(k)} \leq j\lvert x_1 = 0) = \PP(y_t^{(k)}>j\lvert x_1 = 0) - \PP(y_t^\star>j\lvert x_1 = 0),
\end{align}
so that $0 \leq [\bv_t^{(k)}]_j \leq \PP(y_t^{(k)}>j\lvert x_1 = 0)$. Combining the two sign relations,
\begin{align*}
    0 \leq \big\langle \bv_t^{(k)},\, \Delta\bA_t^{(k)}(s_{t+1}^{(k)})\hat{\bD}_{t+1}^{(k)} + \Delta\bc_t^{(k)} \big\rangle \leq \sum_{j=0}^{M-1} \PP(y_t^{(k)}>j\lvert x_1 = 0)\,\big[\Delta\bA_t^{(k)}(s_{t+1}^{(k)})\hat{\bD}_{t+1}^{(k)} + \Delta\bc_t^{(k)}\big]_j .
\end{align*}
Thus the estimation errors in episode $k$ are weighted by the visitation probabilities of the policy executed in that episode, rather than by the optimal-policy visitation probabilities in the offline setting, as in Section~\ref{subsec: biased-SAA-offline}.

With such correspondence, our key step then is to show the following online effective sample size
\begin{align}\label{eq: def-online-effective-sample-size}
    \cN^{(k)}_{\mathsf{cov}} := \bigg(\dfrac{1}{{MT}}{ \sum_{t=1}^T\sum_{j=0}^{M-1} \dfrac{\PP(y_t^{(k)} > j \lvert x_1 = 0)}{N_{t,j}^{(k)}\vee 1}}\bigg)^{-1},
\end{align}
dominates the performance of the $k$-th episode regret, in a form similar to that in Theorem~\ref{thm: censored-sample-complexity}:
\begin{proposition}\label{prop: online-k-th-episode-bound}
For every $k\in [K]$, suppose the event $\cE^{(k)}$ holds, then there exists an absolute constant $c_0$ so that the $k$-th episode cost gap in Algorithm~\ref{alg: censored-online} satisfies \begin{align*}
    \max_{x\in [M]_+} \Delta(x;\pi^{(k)})\leq   c_0 (h_\infty + b_\infty) M \cdot \bigg[T \sqrt{\frac{\log(KMT/\delta){\log(eK)}}{\cN^{(k)}_{\mathsf{cov}}}} + \frac{T^3\log(KMT/\delta){\log(eK)}}{\cN^{(k)}_{\mathsf{cov}}}\bigg].
\end{align*}
\end{proposition}
We leave the proof of Proposition~\ref{prop: online-k-th-episode-bound} to Appendix~\ref{appendix: online-censored-proofs}. In particular, this proposition reduces the regret control to the growth of online effective sample size in $k$, which is determined by the self-coverage property of the generated sequence of policies. In the following proposition, we show that this effective number grows at a nearly linear rate.

\begin{proposition}[Self-generated coverage]\label{prop: online-self-coverage}
There is an absolute constant $c_0>0$ such that, with probability at least $1-\delta$, the dataset generated by Algorithm~\ref{alg: censored-online} satisfies
\begin{align*}
    \cN^{(k)}_{\mathsf{cov}} \geq \frac{c_0\,{(k-1)}}{\log(KMT/\delta)}, \qquad \forall k\in [K]. 
\end{align*}
\end{proposition}

Besides providing a tool for controlling per-episode regrets when combined with Proposition~\ref{prop: online-k-th-episode-bound}, Proposition~\ref{prop: online-self-coverage} also reveals how effectively Algorithm~\ref{alg: censored-online} collects data. By the policy ordering fact~\eqref{eq: online-base-stock-sequence}, the lower bound of $\cN_\mathsf{cov}^{(k)}$ can be automatically transferred to that of $\cN^\star$. In particular, after $K$ episodes, it is guaranteed that $\cN^\star\gtrsim K$ up to logarithmic factors, which is as good as having uncensored observations when applied to offline policy learning algorithms.

Combining Proposition~\ref{prop: online-k-th-episode-bound}, Proposition~\ref{prop: online-self-coverage}, and the fact that $\PP(\cap_{\ell=1}^K\cE^{(\ell)}) \geq 1-\delta$, we have the following regret guarantee of Algorithm~\ref{alg: censored-online}.

\begin{theorem}\label{thm: online-regret}
With probability at least $1-\delta$, Algorithm~\ref{alg: censored-online} satisfies
\begin{align*}
    \mathrm{Regret}(K) \leq C_0 (h_\infty + b_\infty) M \bigg[T{\sqrt{K{\log(eK)}}}\log(KMT/\delta) + T^3\log^2(KMT/\delta){{\log^2 (eK)}}\bigg],
\end{align*}
for some absolute constant $C_0$.
\end{theorem}

\subsection{Stationary Demand: Aggregated DP-LCB}\label{subsec: online-stationary}

We next consider the stationary-demand setting of Section~\ref{subsec: biased-SAA-offline-identical}, where the demand variables remain independent across periods but share a common marginal distribution $P_1 = \cdots = P_T$. As in the offline stationary analysis, observations can be pooled across periods to estimate the shared demand distribution. We therefore replace the period-wise LCB model by an aggregated lower-biased CDF estimator and solve the DP under the aggregated model $\hat{\bF}_1^{(k)} = \cdots = \hat{\bF}_T^{(k)}$.
The same monotone maximum-update rule is applied to the aggregated estimator, so the policy sequence remains optimistic and non-increasing as in~\eqref{eq: online-base-stock-sequence}. The proof shows that this aggregation yields a factor-$T$ gain in the effective sample size, leading to the sharper regret bound below.
\begin{theorem}\label{thm: online-regret-stationary}
Under stationary demand, the aggregated version of Algorithm~\ref{alg: censored-online} satisfies, with probability at least $1-\delta$,
\begin{align*}
    \mathrm{Regret}(K) \leq c_0 (h_\infty + b_\infty) M \bigg[{\sqrt{KT{\log (KT)}}}\log(KMT/\delta) + T^2\log^2(KMT/\delta){{\log^2(KT)}}\bigg],
\end{align*}
for some absolute constant $c_0$.
\end{theorem}
Thus stationarity of the demand distribution yields the same aggregation gain in the online setting as in the offline setting: the effective sample size increases by a factor of $T$, and the leading regret improves by a factor of $\sqrt{T}$, from $\tilde{\cO}((h_\infty + b_\infty)MT\sqrt{K})$ to $\tilde{\cO}((h_\infty + b_\infty)M\sqrt{KT})$, as an analogue of the improvement in general RL setting.

\subsection{Regret Lower Bounds}\label{subsec: online-lower-bound}

We now show that the regret bounds in Section~\ref{subsec: online-dp-lcb},~\ref{subsec: online-stationary} are near-optimal up to logarithmic factors and higher order terms, as presented in the following lower bound result:
\begin{theorem}\label{thm: online-lower-bound}
For any $c_\infty > 0$, $M \geq 2$, $T \geq 3$, and $K \gtrsim T^2$, there exist cost parameters with $h_\infty = b_\infty = c_\infty$ and a collection $\cP$ of non-stationary demand distribution sequences such that, for any online algorithm that adaptively outputs the policy sequence $\{\pi^{(k)}\}_{k=1}^K$,
\begin{align*}
    \max_{\{P_t\}_{t=1}^T \in \cP}\sum_{k=1}^K \EE_{\{d_{t}^\ell\}_{t,\ell=1}^{T, k-1} \sim (\prod_{t=1}^T P_t)^{\otimes (k-1)}}[\Delta(0;\pi^{(k)})] \geq c_0 c_\infty MT\sqrt{K}
\end{align*}
for some absolute constant $c_0$, where for every $t$ the demands $\{d_t^\ell\}_{\ell=1}^{k-1}$ are sampled i.i.d. from $P_t$, independently across periods, and the algorithm observes only censored values.

Similarly, in the stationary demand setting, there exists a collection $\cP_{\mathsf{stat}}$ of stationary demand distributions with $h_\infty = b_\infty = c_\infty$ such that, for any online algorithm that adaptively outputs the policy sequence $\{\pi^{(k)}\}_{k=1}^K$,
\begin{align*}
    \max_{P \in \cP_{\mathsf{stat}}}  \sum_{k=1}^K\EE_{\{d_{t}^\ell\}_{t,\ell=1}^{T, k-1} \sim P^{\otimes T(k-1)}}[\Delta(0;\pi^{(k)})] \geq c_0 c_\infty M\sqrt{KT},
\end{align*}
for some absolute constant $c_0$, where for every $t$ the demands $\{d_t^\ell\}_{t,\ell=1}^{T,k-1}$ are sampled i.i.d. from $P$, and the algorithm observes only censored values.  
\end{theorem}

Theorem~\ref{thm: online-lower-bound} follows directly from the uncensored specializations of Theorems~\ref{thm: lower-bound-independent} and~\ref{thm: lower-bound-independent-identical} through the standard online-to-batch conversion. More precisely, with uncensored observations $\{d_{t}^\ell\}_{t,\ell=1}^{T,K}$, an offline learner can first simulate the online environment over $K$ episodes to obtain the online policy $\{\pi^{(k)}\}_{k=1}^K$, and then execute a policy through uniform randomization over $\{\pi^{(k)}\}_{k=1}^K$. Consequently, we have
    $$\max_{\{P_t\}_{t=1}^T \in \cP } \sum_{k=1}^K  \EE_{\{d_{t}^\ell\}_{t,\ell=1}^{T,k-1} \sim (\prod_{t=1}^T P_t)^{\otimes (k-1)}}\big[\Delta(0;\pi^{(k)})\big] \gtrsim K \inf_{\cA^{\mathsf{off}}} \max_{\{P_t\}_{t=1}^T \in \cP } \EE_{\{d_{t}^\ell\}_{t,\ell=1}^{T,K} \sim (\prod_{t=1}^T P_t)^{\otimes K}}\big[\Delta\big(0;\cA^{\mathsf{off}}\big)\big] $$
where the infimum is taken over all offline algorithms $\cA^{\mathsf{off}}$ that map the uncensored demand observations $\{d_{t}^\ell\}_{t,\ell=1}^{T,K}$ to a policy. The right-hand side can then be lower bounded through Theorem~\ref{thm: lower-bound-independent}. The same reduction to Theorem~\ref{thm: lower-bound-independent-identical} holds for the stationary demand setting.

\section{Discussion of General Bounded Demand Distributions}
\label{sec: general-distribution}

In this section, we extend our results from discrete demand distributions
to general demand distributions on a bounded interval $[0,B]$ for some $B\geq 1$ through a
discretization argument.
The discrete demand assumption maintained throughout this work is
therefore mainly a notational simplification rather than a real
restriction.
Beyond its role in the analysis, the discretization procedure
introduced here also matches how SAA-based methods are efficiently
implemented in practice: as discussed in \citet{cheung2019sampling}, such
a discretize-then-plan step is usually necessary for model-based methods
under general demand distributions, since solving the continuous DP
directly is computationally intractable. 
In the remainder of this section, we first describe the discretization
method and its implications for the uncensored setting, then present the extension of offline and online policy learning results under
censored feedback.

Throughout the discussion, we consider the $T$-period inventory control
problem with cost factors $\{{h}_t, {b}_t\}_{t=1}^T$, following
the same setup as in Section~\ref{sec: preliminaries}.  The only
difference is that each period's demand $D_t$ now has an arbitrary distribution
$P_t$ supported on $[0,B]$, and the policy is allowed to order up
to any level in $[0,B]$ rather than only on the integer grid.  We
write $C^\pi_1(x), C^\star_1(x), \Delta(x;\pi)$ for the expected cost of a policy $\pi$, the optimal cost, and the sub-optimality of $\pi$ with initial inventory level $x \in [0,B]$, respectively.

\subsection{Discretized Environments and Rounded Policies}
Choose $M\in\mathbb N$ and set $\eta:= B/M$, we define the $\eta$-rounded operation $Z^\eta(\cdot): [0, B] \to [M]_+, Z^\eta(d):= \lceil d/\eta \rceil$. For each $t \in [T]$ and $d_t \sim P_t$, we have $Z^\eta(d_t) \sim P_t^\eta$, with the CDF function $F_t^\eta(j):= \PP (Z^\eta(d_t) \leq j ) = F_t(\eta j)$ for $ j \in [M]_+$. 
With $h_{t}^\eta:= \eta h_t, b_t^\eta:= \eta b_t,$ we define the $\eta$-discretized problem of the original problem as the one with cost factors $\{(h_t^\eta,b_t^\eta)\}_{t=1}^T$ and demand distributions with CDF sequence $\{\bF_t^\eta\}_{t=1}^T$. Similarly, we write $C^\pi_{\eta,1}(z), C^\star_{\eta,1}(z), \Delta^\eta(z;\pi^\eta)$ for the policy cost, optimal cost, sub-optimality gap of $\pi^\eta$ of this discretized problem with initial inventory level $z \in [M]_+$, respectively.

With the possibly censored observation dataset $\cD = \{(y_t^k, \bar{d}_t^k)\}_{t=1,k=1}^{T,N}$, we say a policy $\pi$ is obtained via an \textit{$\eta$-rounded version} of a discrete demand setting algorithm $\cA$ output base-stock policy, if $\pi$ is obtained as the following: The learner first feeds the rounded dataset $\cD^\eta := \{(Z^\eta(y_t^k), Z^\eta(\bar{d}_t^k))\}_{t=1,k=1}^{T,N}$ to  $\cA$ to obtain the resulting policy $\hat{\pi}^\eta$ with base stock levels $\{s^\eta_t\}_{t=1}^T,$ and then set $\hat{\pi}$ as the base-stock policy of the original problem with base stock levels $\{\eta s^\eta_t\}_{t=1}^T.$

The following approximation result holds for the rounded environments. 

\begin{lemma}
\label{lem: general-gap-transfer}
For every base-stock policy $\pi^\eta$ with base-stock levels $\{s_t^\eta \}_{t=1}^T \subset [M]_+$ for the $\eta$-discretized problem and its induced $\eta$-grid base-stock policy $\pi$ of the original problem with base-stock levels $\{\eta s^\eta_t\}_{t\in [T]} \subset [0, B]$, it holds that
$\sup_{x\in[0,B]} \Delta(x;\pi) \leq
\max_{z\in[M]_+}
\Delta^\eta(z;\pi^\eta) + 4(h_\infty + b_\infty)  T^2 \eta$.
\end{lemma}

This structural result shows that for any base-stock policy $\pi^\eta$ for the $\eta$-discretized problem, its sub-optimality gap under such a discrete demand problem can be converted to a sub-optimality gap guarantee of the induced policy of the original problem up to an $\cO((h_\infty + b_\infty)T^2 \eta)$ term.

\medskip
\noindent\textbf{Extension of uncensored results.}
By $(h_\infty^\eta + b_\infty^\eta) M \leq (h_\infty + b_\infty) B,$ combining Lemma~\ref{lem: general-gap-transfer} together with Corollary~\ref{coro: uniform-convergence} directly gives the general demand extension of Corollary~\ref{coro: uniform-convergence}. More precisely, given any CDF estimator sequence $\{\hat{\bF}_t^\eta \}_{t=1}^T$ of the $\eta$-discretized problem and $\hat{\pi}^\eta$ the corresponding optimal base-stock policy, then the $\hat{\pi}^\eta$ induced $\eta$-grid base-stock policy $\hat{\pi}$ under the original problem satisfies\begin{align}\label{eq: general-demand-uncensored}
    \sup\nolimits_{x\in [0,B]} \Delta(x;\hat\pi) \leq c_0 (h_\infty + b_\infty) B \big( T \max_{t}\lVert \hat{\bF}^\eta_t - {\bF}_t^\eta\rVert_\infty + T^3 \max_{t}\lVert \hat{\bF}^\eta_t - {\bF}_t^\eta\rVert_\infty^2 + T^2\eta \big)
\end{align}
for some absolute constant $c_0.$
Notably, as $\eta \to 0$,~\eqref{eq: general-demand-uncensored} provides a policy performance counterpart of the evaluation-to-CDF-error reduction result in \citet{ganggang2024all}, whereas their result is asymptotic and its leading-order dependence on the CDF error is of order $\cO(T^2\sum_{t=1}^T \lVert \hat{\bF}_t - {\bF}_t\rVert_\infty)$.

In particular, as the  convergence rate SAA estimators~\eqref{eq: F-hat-SAA} and~\eqref{eq: F-hat-SAA-identical} in Lemma~\ref{lem: uniform-convergence-independent} and~\ref{lem: uniform-convergence-independent-and-identical} are independent of $\eta$,~\eqref{eq: general-demand-uncensored} directly yields policy learning guarantees under uncensored demand observations that parallel Theorems~\ref{thm: sample-complexity-independent} and~\ref{thm: sample-complexity-independent-identical} up to an additive $\cO((h_\infty+b_\infty)T^2\eta)$ term.
In particular, the resulting error bounds improve monotonically as $\eta$ decreases, so setting a smaller $\eta$ always yields a sharper error guarantee; the only trade-off is the increased computational cost of solving the finer $\eta$-grid DP.

Besides the SAA estimators, the deterministic reduction from the policy performance gap to the model error in~\eqref{eq: general-demand-uncensored} can also be applied to other CDF estimators for possibly covariate-dependent demand models, including the tree-based method of \citet{ban2019dynamic} and the log-spline or kernel-density based estimators of \citet{ganggang2024all,fan2024policy}.

\subsection{Extension of Offline Policy Learning Results}
In this section, we consider offline policy learning with censored observations
$\cD:= \{(y_t^k,\bar{d}_t^k)\}_{t,k=1}^{T,N}$ and provide the sample complexity guarantee of the policy $\hat{\pi}$ as the output of $\eta$-rounded version of Algorithm~\ref{alg: censored-offline}. 

Denoting $\cN^{\eta,\star}$ as the discrete effective sample size~\eqref{eq: def-effective-sample-size} of the rounded observations $\cD^\eta$ under the $\eta$-discretized environment, combining Lemma~\ref{lem: general-gap-transfer} and Theorem~\ref{thm: censored-sample-complexity} together yields 
\begin{equation}\label{eq: general-offline-via-discrete}
    \max_{x\in [0,B]} {\Delta}(x;\hat{\pi})\lesssim (h_\infty + b_\infty)B \bigg[T\sqrt{\frac{\log(\frac{BTN}{\eta \delta})\log(eN) }{\cN^{\eta,\star}}} + T^2 \big(\frac{\log(\frac{BTN}{\eta \delta}) \log(eN)}{\cN^{\eta,\star}} +\eta\big)\bigg],
\end{equation}
with probability at least $1-\delta$. 

To connect this result to a more transparent, problem-intrinsic description of the coverage, we introduce the following integral-form effective sample size
\begin{align}
\tilde{\cN}^\star
:=
\bigg(\frac{1}{BT}\sum_{t=1}^T\int_0^B
\frac{\PP (y_t^\star > u\lvert x_1 = 0)}{N_t(u)\vee1}du\bigg)^{-1}, \quad N_t(u)&:=\sum_{k=1}^N\mathbf1\{y_t^k>u\}.
\label{eq:general-continuous-effective-size}
\end{align}
Equation~\eqref{eq:general-continuous-effective-size} provides a natural generalization of the discrete effective sample size by replacing the summation over the demand space by the integral form. A Riemann-sum approximation argument, combined with a comparison of the optimal base-stock levels of the original and the discretized problems, leads to the following approximation result, as detailed in Appendix~\ref{appendix: general-demand-proof}:
\begin{equation}\label{eq: general-N-approximation}
\lvert 1/\cN^{\eta,\star} - 1/\tilde{\cN}^{\star} \rvert \leq T \eta.
\end{equation}
With~\eqref{eq: general-N-approximation}, the bound~\eqref{eq: general-offline-via-discrete} continues to hold with $\cN^{\eta,\star}$ replaced by $\tilde{\cN}^\star$, with an additional additive $\tilde{\cO}\big((h_\infty+b_\infty)B(T^{3/2}\sqrt{\eta} + T^3\eta)\big)$ term. This gives the offline policy learning guarantee under general demand distributions. Notably, unlike in the uncensored case, decreasing the discretization gap $\eta$ may now increase the $\log(\frac{T}{\delta \eta})$ terms, so $\eta$ cannot be set arbitrarily small. On the other hand, setting $\eta \asymp 1/(NT)^3$ is always a safe choice, which gives the same rate of convergence as in the discrete demand setting up to $\text{polylog}(NT)$ factors.

When all demand distributions are identical, both the aggregated version of~\eqref{eq: general-offline-via-discrete} (cf.\ Theorem~\ref{thm: censored-sample-complexity-identical}) and the approximation result in~\eqref{eq: general-N-approximation} extend analogously, with the aggregated integral-form effective sample size defined as
\begin{align}
\tilde{\cN}_{\sf agg}^\star
:=\bigg(\frac{1}{BT}\sum_{t=1}^T\int_0^B
\frac{\PP (y_t^\star > u\lvert x_1 = 0)}{N_{\rm agg}(u)\vee1}du \bigg)^{-1}, \quad N_{\rm agg}(u)&:=\sum_{r=1}^T\sum_{k=1}^N
\mathbf1\{y_r^k>u\}.
\label{eq:general-continuous-effective-size-agg}
\end{align}
Consequently, a result similar to Theorem~\ref{thm: censored-sample-complexity-identical} holds with respect to the new effective sample size $\tilde{\cN}_{\mathsf{agg}}^\star$, up to an additional additive $\tilde{\cO}\big((h_\infty+b_\infty)B(T^{3/2}\sqrt{ \eta} + T^3 \eta)\big)$ term.

\subsection{Extension of Online Policy Learning Results}

In this section, we consider the online learning extension over $K$ episodes. Unlike in the previous section, an $\eta$-rounded algorithm alone is not sufficient here, as the online policy design also involves the data collection procedure. To describe the details of the algorithm, given a base-stock policy $\{{s}_t^{(k)}\}_{t=1}^T$ over $\eta$-grids, we define two associated inventory-level processes over $[M]_+$ based on the demand realizations $\{({d}_t^{(k)})\}_{t=1}^T$ as follows:
\begin{align}\label{eq: online-discrete-process}
    Y_{t}^{\eta,(k)}= \max\{X_{t}^{\eta,(k)},Z^\eta(s_t^{(k)})\},\quad X_{t+1}^{\eta,(k)}= (Y_t^{\eta, (k)} - Z^\eta({d}_{t}^{(k)}) )_+ , \quad X_{1}^{\eta, (k)}= 0. 
\end{align}
The processes $X_t^{\eta,(k)},Y_t^{\eta,(k)}$ can be seen as the entering and post-ordering inventory levels under the discrete base-stock policy $\{Z^\eta(s_t^{(k)})\}_{t=1}^T$ in the $\eta$-discretized environment, with initial inventory level $0$. The following relation holds between these processes and the observed inventory levels under the original problem.
\begin{lemma}\label{lem: online-general-process}
Let $\{(x_t^{(k)},y_t^{(k)})\}_{t=1}^T$ denote the entering and post-ordering inventory levels of the original problem under the same $\eta$-grid base-stock policy $\{s_t^{(k)}\}_{t=1}^T$ and demand realizations $\{d_t^{(k)}\}_{t=1}^T$. Then we have $\eta Y_t^{\eta,(k)} \leq y_t^{(k)}$ and $ \eta X_{t}^{\eta,(k)} \leq x_t^{(k)}$ for all $t \in [T]$.
\end{lemma}

In particular, by $\eta Y_t^{\eta,(k)} \leq y_t^{(k)}$ and $\eta Z^\eta(d_t^{(k)}) \geq d_t^{(k)}$, we have $Z^\eta(y_t^{(k)}) \geq Y_t^{\eta,(k)}$ and $d_t^{(k)} \geq y_t^{(k)} \implies Z^\eta(d_t^{(k)}) \geq Y_t^{\eta,(k)}$.
As a result, $\min\{Y_t^{\eta,(k)}, Z^\eta(d_t^{(k)})\} = \min\{Y_t^{\eta,(k)}, Z^\eta(\bar{d}_t^{(k)})\}$: the augmented censored observations can be computed from the censored observations $\{(y_t^{(k)}, \bar{d}_t^{(k)})\}_{t=1}^T$ alone.

Now we can describe the online learning algorithm under general demand distributions: At each episode $k\in [K]$, the learner first reconstructs the rounded historical dataset $
\cD^\eta_k:= \{ (Y_t^{\eta,(\ell)}, \min\{Y_t^{\eta,(\ell)}, Z^\eta(\bar{d}_t^{(\ell)})\}) \}_{t,\ell=1}^{T,k-1}$ from the censored observations $\{ (y_t^{(\ell)},\bar{d}_t^{(\ell)}) \}_{t,\ell=1}^{T,k-1}$, computes the $k$-th episode's policy $\pi^{\eta, (k)}$ as in Algorithm~\ref{alg: censored-online} for the $\eta$-discretized problem, and executes its $\eta$-grid version in the original problem environment to obtain $\{y_t^{(k)}, \bar{d}_t^{(k)}\}_{t=1}^T.$

Since $\cD_k^\eta$ is exactly the observable dataset obtained by executing the historical policies $\hat{\pi}^{\eta,(\ell)}$ on the $\eta$-discretized problem with initial inventory level $X_1^\eta = 0$, the online regret guarantee in Theorem~\ref{thm: online-regret} holds. Combining this with the approximation result in Lemma~\ref{lem: general-gap-transfer}, we obtain
\begin{align}
    \mathrm{Regret}(K) \leq C_0 (h_\infty + b_\infty) B \bigg[T{\sqrt{K{\log(eK)}}}\log(\frac{KBT}{\eta \delta}) + T^3\big(\log^2(\frac{KBT}{\delta\eta}){{\log^2 (eK)}} + K \eta \big)\bigg].
\end{align}
As in the offline results, taking $\eta \asymp 1/K$ gives the same online regret result as in Theorem~\ref{thm: online-regret} up to additional $\log K$ factors. Moreover, the same extension of Theorem~\ref{thm: online-regret-stationary} also holds for the stationary demand setting, with the leading-order regret improving from $\tilde{\cO}(T\sqrt{K})$ to $\tilde{\cO}(\sqrt{TK}).$

\section{Numerical Illustrations}\label{sec: numerics}

In this section, we provide numerical results for both the offline and online settings. In the offline setting, we test DP-UCB against vanilla DP and \textsc{SAIL}~\citep{qin2023sailing}, and plot the sub-optimality gap $\max_x \Delta(x;\hat\pi)$. In the online setting, we test DP-LCB against vanilla DP and \textsc{SAIL-CE}, and plot cumulative regret $\mathrm{Regret}(K)$. Details of the instances and implementation are given in Appendix~\ref{appendix-sec: experiment-details}.
Throughout the experiments, we replace the conservative absolute constants $8$ and $24$ in~\eqref{eq: CB-algorithm-offline} and its online counterpart by $1$, and find that the empirical performance is consistently good. 
\subsection{Offline Policy Learning}

We conduct two offline experiments. Figure~\ref{fig: offline} shows that the sub-optimality gap is governed by the effective sample size $\cN^\star$, rather than by the raw number of trajectories. Figure~\ref{fig: pessimism} shows the effect of upper-biased estimates in DP-UCB, by comparing it with vanilla DP and \textsc{SAIL}~\citep{qin2023sailing}.

\begin{figure}[t]
\centering
\includegraphics[width=0.85\textwidth]{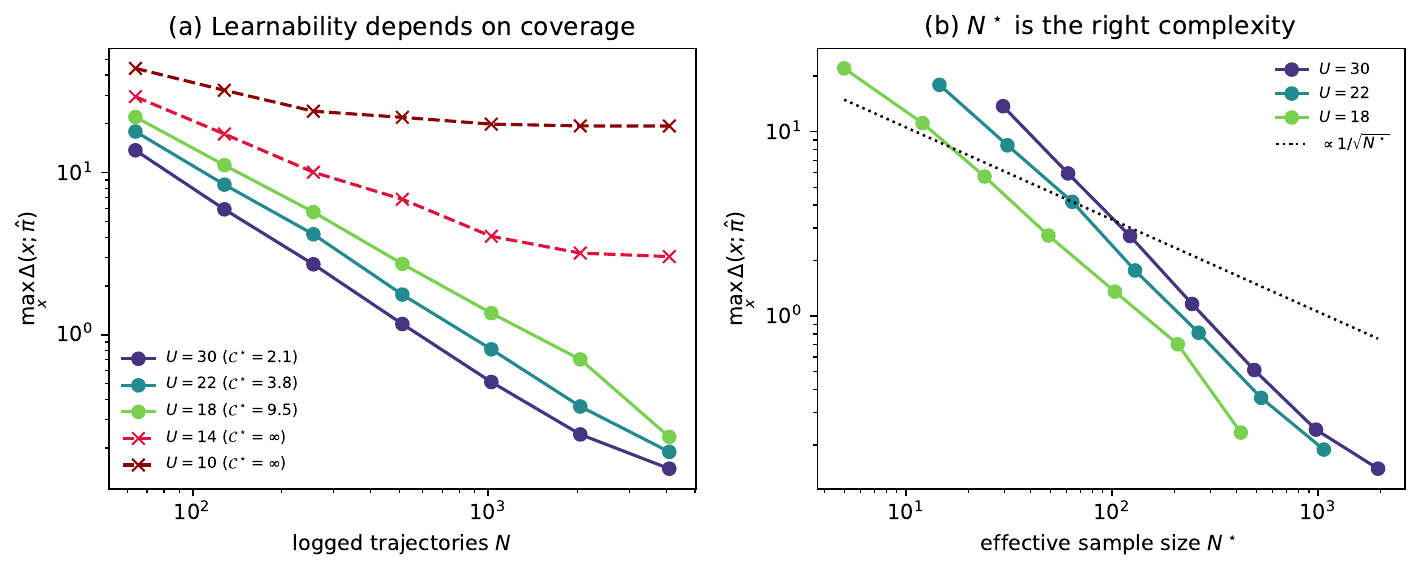}
\caption{
Panel (a) plots the DP-UCB gap versus the number of trajectories $N$; panel (b) plots the adequate-coverage curves versus the effective sample size $\cN^\star$. Experiments are conducted under a seasonal truncated-Poisson instance with $M=30, T = 6$ and $s_t^\star\in\{7,\dots,17\}$, with $150$ repetitions. In each trajectory, the logging policy draws $B_t\sim\mathrm{Uniform}\{0,\dots,U\}$ each period and orders up to $y_t^b=\max\{x_t,B_t\}$.}
\label{fig: offline}
\end{figure}

\smallskip\noindent\textbf{Effective Sample Size and Coverage.}
Theorem~\ref{thm: censored-sample-complexity} and Corollary~\ref{corollary: C-star-sample-complexity} show that offline policy learning is governed by the effective sample size $\cN^\star$ and the coverage ratio $\cC^\star$. Figure~\ref{fig: offline} supports this prediction. When the logging policy explores high enough to cover the coordinates used by the optimal policy, $\cC^\star<+\infty$ and the gap decreases as $N$ grows; smaller $\cC^\star$ leads to a smaller gap. When $\cC^\star=+\infty$, some needed coordinates are never revealed, and the gap stays flat even with more data. Panel~(b) further plots $\max_x\Delta(x;\pi)$ against $\cN^\star$ and shows that data collected under different behavior policies have very similar dependence on $\cN^\star$.

\begin{figure}[tt]
\centering
\includegraphics[width=0.85\textwidth]{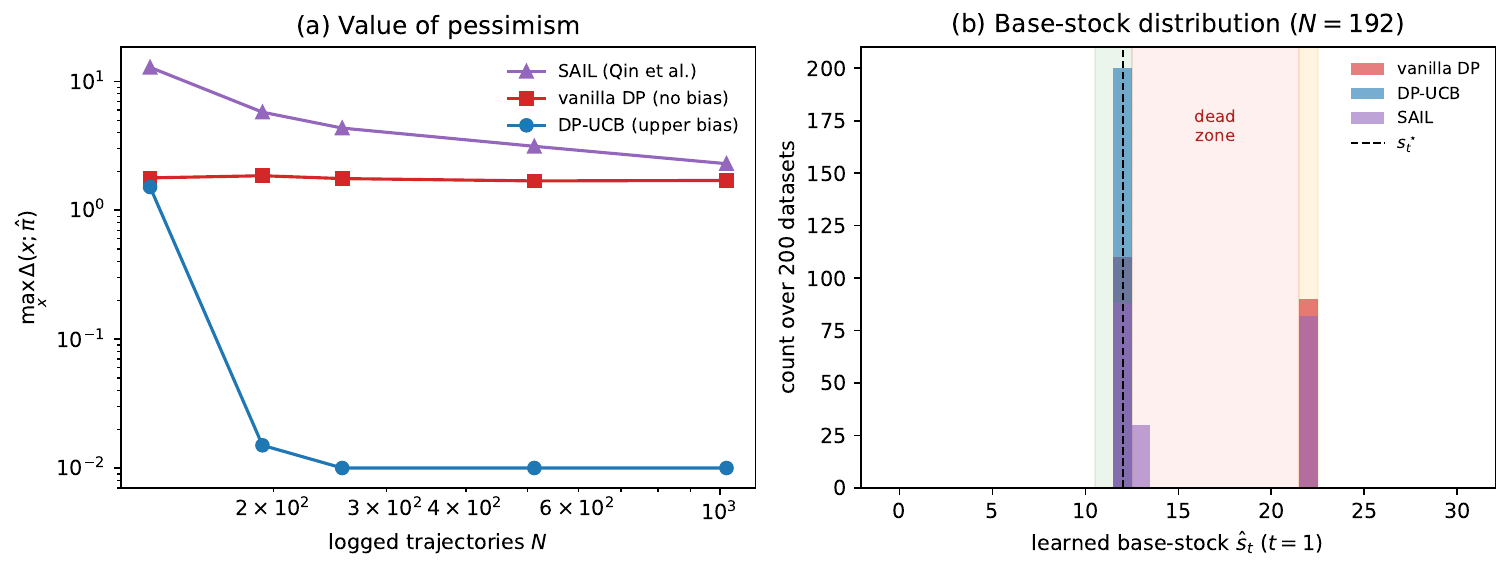}
\caption{Performance gap under partial coverage. DP-UCB is compared with vanilla DP and \textsc{SAIL} on an instance with $M=30$ and $T=12$ over $200$ repetitions; the demand distributions are detailed in Appendix~\ref{appendix-sec: experiment-details}. Panel~(a) plots the sub-optimality gap versus $N$ with $N \in \{128,192,256,512,1024\}$. Panel~(b) plots the distribution of learned first-period base-stock levels at $N = 192$ across $200$ repetitions.}

\label{fig: pessimism}
\end{figure}

\begin{figure}[t]
\centering
\includegraphics[width=0.85\textwidth]{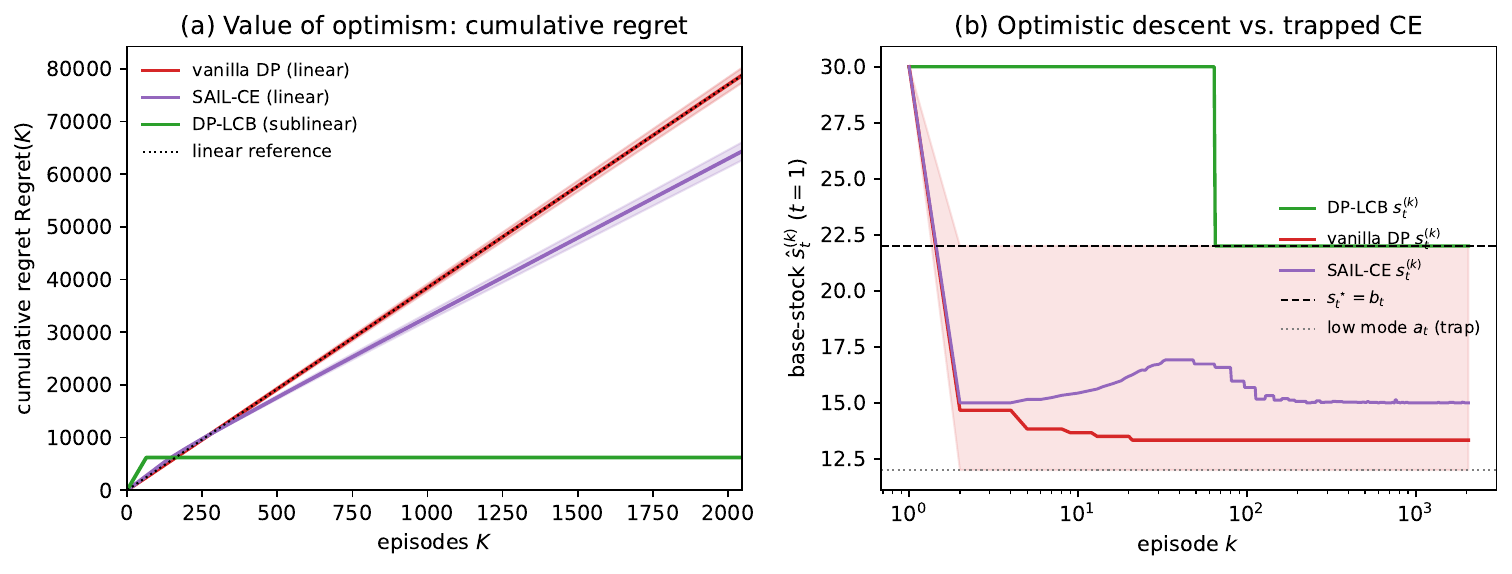}
\caption{Online regret under censored feedback. DP-LCB is compared with vanilla DP and \textsc{SAIL-CE} on an instance with $M=30$ and $T=12$ over $60$ repetitions, with demand distributions detailed in Appendix~\ref{appendix-sec: experiment-details}. Panel~(a) plots cumulative regret versus the number of episodes $K$, up to $K=2048$. Panel~(b) shows how the mean value of the learned first-period base-stock levels changes across episodes, over $60$ repetitions.}
\label{fig: online}
\end{figure}

\smallskip\noindent\textbf{Value of Pessimism.} Figure~\ref{fig: pessimism} compares DP-UCB with vanilla DP and \textsc{SAIL}. In this instance, the optimal base-stock level lies at a thinly covered CDF crossing followed by a \textit{dead zone}. Due to the lack of directional bias, vanilla DP and \textsc{SAIL} may underestimate this crossing and jump across the dead zone to a much higher stock level, causing a large and nearly flat gap. DP-UCB avoids this failure by pushing the CDF upward, which lowers the learned base-stock level and keeps the policy in the covered region. Panel~(b) shows illustrate this mechanism directly through the distribution of learned first-period order-up-to levels.

\subsection{Online Policy Learning}

In this section, we compare the DP-LCB algorithm with vanilla DP and \textsc{SAIL-CE} (the certainty-equivalent online deployment of \textsc{SAIL}~\citep{qin2023sailing}) in the online setting with censored feedback. We report the regret results in panel~(a) and the distribution of first-period base-stock levels in panel~(b). In the problem instance considered, as detailed in Appendix~\ref{appendix-sec: experiment-details}, the optimal base-stock level is high, but a policy that orders too low cannot observe the high-demand spike because of censoring. As shown in panel~(b), vanilla DP quickly gets trapped at a low stock level, so it keeps missing the spike and suffers nearly linear regret. \textsc{SAIL-CE} has the same issue because it is used greedily and does not add optimism. DP-LCB avoids this trap by pushing the CDF downward, which raises the learned base-stock level and generates the needed coverage. This leads to much smaller regret in panel~(a), matching the role of optimism in Theorem~\ref{thm: online-regret}.

\appendix
\newpage

\section{Equivalence Between Backlogging and Lost-Sales Dynamics}\label{appendix-sec: backlog-LS-equivalence}

In this section, we formalize the equivalence between the lost-sales setup introduced in Section~\ref{sec: preliminaries} and the backlogging dynamics commonly adopted in the uncensored inventory-learning literature \citep{qin2023sailing,xie2024vc}. For clarity of presentation, we use the superscripts $\mathsf{B}$ and $\mathsf{LS}$ to denote the corresponding quantities under the backlogging and the lost-sales dynamics, respectively. In particular, the quantities associated with the lost-sales dynamics corresponds to the notations in the main article.

Fix a demand trajectory $(d_1,\dots,d_T)\in [M]_+^T$ and a base-stock policy $\pi=\{s_t\}_{t=1}^T$ with $s_t\in [M]_+$ for every $t\in [T]$. 
Let $x_t^{\mathsf{B}}, y_t^{\mathsf{B}}$ be the corresponding entering inventory level and post-ordering inventory levels under the backlogged system, i.e.  $y_t^{\mathsf{B}} = \max\{x_t^{\mathsf{B}},s_t\},   x_{t+1}^{\mathsf{B}} = y_t^{\mathsf{B}}-d_t$. We have the following relation between the lost-sales and backlogging inventory levels.

\begin{proposition}\label{prop:backlog-LS-pathwise-equivalence}
Fix any initial inventory $x\in [M]_+$ and suppose $x_1^{\mathsf{LS}}=x_1^{\mathsf{B}}=x$, then
\begin{align}\label{eq:backlog-LS-state-action-coupling}
    x_t^{\mathsf{LS}}=(x_t^{\mathsf{B}})_+,
    \quad \forall t\in [T+1],
    \qquad
    y_t^{\mathsf{LS}}=y_t^{\mathsf{B}},
    \quad \forall t\in [T].
\end{align}
As a consequence, 
$    \sum_{t=1}^T c_t(y_t^{\mathsf{LS}})
    =
    \sum_{t=1}^T c_t(y_t^{\mathsf{B}})$.
\end{proposition}

\begin{proof}[Proof of Proposition~\ref{prop:backlog-LS-pathwise-equivalence}]
We prove~\eqref{eq:backlog-LS-state-action-coupling} by induction. The initial condition holds because $x\geq 0$. Suppose that $x_t^{\mathsf{LS}}=(x_t^{\mathsf{B}})_+$ for some $t\in [T]$. Since $s_t\geq 0$,
$    \max\{(x_t^{\mathsf{B}})_+,s_t\}
    =
    \max\{x_t^{\mathsf{B}},s_t\}$. It follows that
\begin{align*}
    y_t^{\mathsf{LS}}
    =\max\{x_t^{\mathsf{LS}},s_t\}
    =\max\{(x_t^{\mathsf{B}})_+,s_t\}
    =\max\{x_t^{\mathsf{B}},s_t\}
    =y_t^{\mathsf{B}}.
\end{align*}
Using the two transition equations then gives
$ x_{t+1}^{\mathsf{LS}}
    =(y_t^{\mathsf{LS}}-d_t)_+
    =(y_t^{\mathsf{B}}-d_t)_+
    =(x_{t+1}^{\mathsf{B}})_+,$
which completes the induction. 
\end{proof}

Now, let $ C_t^{\star,\mathsf{B}},C_t^{\star,\mathsf{LS}}$ denote the optimal cost-to-go in period $t$ under the backlogging and lost-sales dynamics respectively, with $C_{T+1}^{\star,\mathsf{LS}}=C_{T+1}^{\star, \mathsf{B}}\equiv 0$. We have the following equivalence result.
\begin{proposition}\label{prop:backlog-LS-value-equivalence}
For every $t\in[T], z\in\mathbb{Z}$,
$    C_t^{\star,\mathsf{B}}(z)
    =
    C_t^{\star, \mathsf{LS}}(z_+)$.
In particular, $C_1^{\star,\mathsf{B}}(x)=C_1^{\star}(x)$ for every initial inventory $x\in[M]_+$.
\end{proposition}

\begin{proof}
We proceed by backward induction. The claim is immediate at $t=T+1$. Suppose that it holds at period $t+1$. The Bellman equation of the backlogging system gives, for any $z\in\mathbb{Z}$,
\begin{align*}
    C_t^{\star,\mathsf{B}}(z)
    &=\min_{y\geq z}\EE\Big[c_t(y)+C_{t+1}^{\star,\mathsf{B}}(y-d_t)\Big] = \min_{y\geq z}\EE\Big[c_t(y)+C_{t+1}^{\star, \mathsf{LS}}((y-d_t)_+)\Big],
\end{align*}
where the second equality uses the induction hypothesis. On the other hand, since every feasible $y<0$ is weakly dominated by $0$: since $d_t\geq 0$, $(y-d_t)_+=(0-d_t)_+=0, c_t(y)=b_t(d_t-y)\geq b_t d_t=c_t(0).$
We may therefore restrict the minimization to $y\geq \max\{z,0\}=z_+$, and hence
\begin{align*}
    C_t^{\star,\mathsf{B}}(z)
    &=\min_{y\geq z_+}\EE\Big[c_t(y)+C_{t+1}^{\star}((y-d_t)_+)\Big]  =C_t^{\star}( z_+),
\end{align*}
which completes the induction.
\end{proof}

Let
$    C_1^{\pi,\mathsf{B}}(x), C_1^{\pi,\mathsf{LS}}(x)
$ denote the expected total cost under $\pi$ under the corresponding dynamics, by taking expectation in Proposition~\ref{prop:backlog-LS-pathwise-equivalence}, we obtain $C_1^{\pi,\mathsf{LS}}(x) =C_1^{\pi,\mathsf{B}}(x)$ for any $x\in [M]_+$ and base stock policy $\pi$. Applying Proposition~\ref{prop:backlog-LS-value-equivalence} then gives $ C_1^{\pi,\mathsf{LS}}(x)-C_1^{\star,\mathsf{LS}}(x)
    =
    C_1^{\pi,\mathsf{B}}(x)-C_1^{\star,\mathsf{B}}(x)$, as desired.

\section{Proof of General Demand Results}\label{appendix: general-demand-proof}

\subsection{Proof of Lemma~\ref{lem: general-gap-transfer}}

For any demand realizations $\{d_t\}_{t=1}^T$ of the original problem, we denote $\tilde{d}_t:= \eta Z^\eta(d_t)$ as its $\eta$-rounded version. Consider any base-stock policy $\pi$ with base-stock levels $\{s_t\}_{t=1}^T \subset [0,B]$, and denote the entering and post-ordering inventory levels under $\{d_t\}, \{\tilde{d}_t\}$ from any same initial inventory level by
$(x_t,y_t)$ and $(\tilde{x}_t,\tilde{y}_t)$ respectively. Then by $\tilde{d}_t \geq d_t$ for all $t\in[T]$ and monotonicity, 1-Lipschitz continuity of the order-up-to function $x\mapsto\max\{x,s_t\}$, we have
\begin{align*}
0\leq x_t- \tilde{x}_t\leq\sum\nolimits_{k=1}^{t-1}
(\tilde d_k - d_k) \leq t \eta,
\qquad
0\leq y_t-\tilde{y}_t\leq x_t-\tilde{x}_t\leq t\eta.
\end{align*}
Then, since each period's newsvendor cost is $(h_\infty + b_\infty)$-Lipschitz, taking expectation over $d$  leads to \begin{align}\label{eq:  appendix-proof-gap-transfer}
   \max_{x\in [0,B]}\lvert C_{1}^{\pi}(x) - \tilde{C}_{1}^{\pi}(x)\rvert \leq T^2 (h_\infty + b_\infty) \eta
\end{align} 
for every base-stock policy $\pi$ of the original problem, with $\tilde{C}_1^\pi$ the expected total cost under the distribution of $\{\tilde{d}_t\}$. In particular, for each base-stock policy $\pi$, we have \begin{align*}
    \tilde{C}_1^{\pi}(x)
    &= \sum\nolimits_{t=1}^T\EE[h^\eta_t(\frac{\tilde{y}_t}{\eta} - Z^\eta(d_t))_+  + b^\eta_t(Z^\eta(d_t)-\frac{\tilde{y}_t}{\eta} )_+\lvert \tilde{x}_1 = x].
\end{align*}
This corresponds to the total cost under the $\eta$-discretized environment of some policy $\pi'$ with possibly non-integer base-stock levels with the initial inventory level $x/\eta$, and, since integer order-up-to levels are without loss of optimality for the $\eta$-discretized problem (cf.\ Section~\ref{sec: preliminaries}), it holds that $\tilde{C}^{\pi}_1(\eta z) \geq C^{\star}_{\eta,1}(z)$ for every base-stock policy $\pi$ and $z \in [M]_+$. And if $\pi$ is an $\eta$-grid policy induced by some $\pi^\eta$, then $\tilde{C}^\pi_1(\eta z) = {C}^{\pi^\eta}_{\eta,1}(z)$ for all $z \in [M]_+.$
Now, for every fixed $x\in [0,B]$ and an $\eta$-grid base-stock policy $\pi$ induced by $\pi^\eta$,  there exists some $\eta$-grid point $r_x$ so that $\lvert r_x -x\rvert \leq \eta$. For such $r_x$, we have \begin{align}\label{eq:  general-demand-proof-tmp1}
\Delta(r_x; \pi) &= \underbrace{C_{1}^{\pi}(r_x) - \tilde{C}_1^{\pi}(r_x)}_{\leq T^2(h_\infty + b_\infty) \eta} + \underbrace{\tilde{C}_1^{\pi}(r_x) - \tilde{C}_1^{\pi^\star}(r_x)}_{ \leq \Delta^\eta (r_x/\eta;\pi^\eta)} +\underbrace{\tilde{C}_1^{\pi^\star}(r_x) - C_1^\star(r_x)}_{\leq T^2(h_\infty + b_\infty) \eta},
\end{align}
with $\pi^\star$ an optimal base-stock policy of the original problem.
By $r_x/\eta \in [M]_+,$ \eqref{eq:  general-demand-proof-tmp1} implies $\Delta(r_x;\pi) \leq \max_{z\in [M]_+} \Delta^\eta(z;\pi^\eta) + 2 T^2(h_\infty +b_\infty)\eta.$ By the $2T^2(h_\infty+b_\infty)$-Lipschitz continuity of $\Delta(x;\pi)$ in $x$, \begin{align*}
    \Delta(x;\pi)\leq \Delta(r_x;\pi) + 2T^2(h_\infty +b_\infty)\eta \leq  \max_{z\in [M]_+} \Delta^\eta(z;\pi^\eta) + 4 T^2(h_\infty +b_\infty)\eta.
\end{align*}
Taking maximum over $x\in [0,B]$ gives the desired result.

\subsection{Proof of Equation~(\ref{eq: general-N-approximation})}\label{appendix:general-demand-coverage-bridge}

In this section, we provide the proof of~\eqref{eq: general-N-approximation}. The aggregated version follows from the same argument with the count $N_t(u)$ replaced by the pooled count $N_{\mathsf{agg}}(u)$ in~\eqref{eq:general-continuous-effective-size-agg}.

Given any demand realization $d_t$, we define an $\eta$-rounded coupled demand $\tilde{d}_t = \eta Z^\eta(d_t)$, so that $d_t \leq \tilde{d}_t < d_t + \eta$. By definition, the $\tilde{d}_t$ follows the demand distribution of the $\eta$-discretized problem. Let $\{s_t^\star\}_{t=1}^T$ and $\{s_t^{\eta,\star}\}_{t=1}^T$ be the optimal base-stock levels of the original problem and the $\eta$-discretized problem respectively. We set $y_t^\star$ and $y_t^{\eta,\star}$ as the post-ordering inventory levels under the two optimal policies with $x_1 = 0$. With $N_t(u)$ defined in~\eqref{eq:general-continuous-effective-size}, we set $\lambda_t(u) := {1}/(N_t(u)\vee 1)$.

Since $Z^\eta(y) > j \Leftrightarrow y > j\eta$, the coordinate counts of the rounded dataset satisfy $\sum_{k=1}^N \mathbf{1}\{Z^\eta(y_t^k) > j\} = N_t(j\eta)$. Hence, by definitions~\eqref{eq:  def-effective-sample-size} and~\eqref{eq:general-continuous-effective-size}, together with $1/(MT) = \eta/(BT)$, it suffices to show, for every $t \in [T]$, the per-period bound
\begin{equation}
-t\eta \;\leq\; \eta\sum_{j=0}^{M-1} \PP(y_t^{\eta,\star} > j \lvert x_1 = 0)\lambda_t(j\eta) - \int_0^B \PP(y_t^\star > u \lvert x_1 = 0)\lambda_t(u)du \;\leq\; T\eta;
\label{eq: general-demand-per-period-gap}
\end{equation}
averaging~\eqref{eq: general-demand-per-period-gap} over $t$, dividing by $B$, and using $B \geq 1$ then yield $\lvert 1/\cN^{\eta,\star} - 1/\tilde{\cN}^\star \rvert \leq T\eta/B \leq T\eta$.

\noindent\textbf{Step~1: control the base-stock policy gap.} For this purpose, we first show the following base-stock level gap result:
\begin{equation}
    s_t^\star \leq \eta s_t^{\eta,\star} \leq \min\{B,\; s_t^\star + (T-t+1)\eta\}.
    \label{eq: general-demand-optimal-base-stock-localization}
\end{equation}
To see~\eqref{eq: general-demand-optimal-base-stock-localization}, let $D_t^\star$ denote the right derivative of the general-demand order-up-to value function $W_t^\star(y) := \EE[h_t(y-d_t)_+ + b_t(d_t-y)_+ + C^\star_{t+1}((y-d_t)_+)]$, so that $s_t^\star = \min\{y : D_t^\star(y) \geq 0\}$. Direct differentiation gives
\begin{equation}
    D_t^\star(y) = (h_t+b_t)F_t(y) - b_t + \EE\big[ \Gamma^\star_{t+1}(y - d_t)\big], \quad \forall t\in [T],\quad D^\star_{T+1} \equiv 0,
    \label{eq: general-demand-continuous-derivative}
\end{equation}
with $\Gamma^\star_{t}(z):= \bm{1}\{z\geq 0\} \big(D_t^\star(z)\big)_+$ the right derivative of the constrained value function $C_t^\star$, which is non-decreasing by convexity.
The same recursion holds for the $\eta$-discretized problem with $d_t$ replaced by $\tilde{d}_t$ and the cost factors $h_t, b_t$ unchanged, after a rescaling in $\eta$: if we denote $\tilde{D}_t^\star, \tilde{\Gamma}_t^\star$ as the related quantities given in this new recursion, it holds that  $\eta s_t^{\eta,\star} = \min\{y : \tilde{D}_t^\star(y) \geq 0\}$. (To see this fact, noticing that $ h_t^\eta (z - Z^\eta(d_t))_+ + b_t^\eta ( Z^\eta(d_t) - z)_+ = h_t(\eta z - \tilde{d}_t)_+ + b_t(\tilde{d}_t - \eta z)_+$.) 

By the rule of determining $s^\star_t,s^{\eta,\star}_t$, to prove~\eqref{eq: general-demand-optimal-base-stock-localization}, it suffices to show that for $\rho_t:= (T-t+1)\eta,$
\begin{equation}\label{eq: general-demand-proof-tmp2}
\tilde{D}_t^\star(y) \leq D_t^\star(y) \leq \tilde{D}_t^\star(y+\rho_t)   \quad\text{when } y + \rho_t \leq B.
\end{equation}
We prove~\eqref{eq: general-demand-proof-tmp2} by backward induction over $t$: at $t = T+1$ both inequalities hold directly by $D^\star_{T+1} = \tilde{D}^\star_{T+1} \equiv 0$. Now suppose the result holds at $t+1$ for some $t\in [T]$, taking positive parts in~\eqref{eq: general-demand-proof-tmp2} gives
$$\tilde{\Gamma}^\star_{t+1} \leq \Gamma^\star_{t+1},
\qquad
\tilde{\Gamma}^\star_{t+1}(z + \rho_{t+1}) \geq \Gamma^\star_{t+1}(z) \quad\text{when } z + \rho_{t+1} \leq B.$$
To see the first inequality in~\eqref{eq: general-demand-proof-tmp2} at $t$, denote $\tilde{F}^\eta_t$ as the CDF function of $\tilde{d}_t$, then $\tilde{d}_t \geq d_t$ implies $\tilde{F}^\eta_t(y) \leq F_{t}(y)$, and $\tilde{\Gamma}^\star_{t+1}(y - \tilde{d}_t) \leq \Gamma^\star_{t+1}(y - \tilde{d}_t) \leq \Gamma^\star_{t+1}(y - d_t)$ by the induction hypothesis and the monotonicity of $\Gamma^\star_{t+1}$, taking expectation over $d_t$ and $\tilde{d}_t$ then gives the desired inequality. 

To see the second inequality at $t$: $\tilde{d}_t < d_t + \eta$ implies $\tilde{F}^\eta_t(y + \rho_t) \geq \tilde{F}^\eta_t(y + \eta) \geq F_{t}(y)$, and $\tilde{\Gamma}^\star_{t+1}(y + \rho_t - \tilde{d}_t) \geq \tilde{\Gamma}^\star_{t+1}(y - d_t + \rho_{t+1}) \geq \Gamma^\star_{t+1}(y - d_t)$ by $\tilde{d}_t - d_t < \eta = \rho_t - \rho_{t+1}$, the monotonicity of $\tilde{\Gamma}^\star_{t+1}$, and the induction hypothesis. Taking expectations over  $d_t, $ and $\tilde{d}_t$ gives the second inequality. This finishes the proof of~\eqref{eq: general-demand-optimal-base-stock-localization}.

\noindent\textbf{Step 2: control the visiting probability gap.} Started from $x_1 = 0$, the period-$t$ post-ordering inventory level of a base-stock policy with levels $\{s_r\}_{r=1}^T$ under demands $\{d_r\}_{r=1}^T$ equals $\max_{1\leq r\leq t}\big(s_r - \sum_{\ell=r}^{t-1} d_\ell\big)_+$ pathwise. Combining this representation with~\eqref{eq: general-demand-optimal-base-stock-localization} and $0 \leq \tilde{d}_\ell - d_\ell < \eta$ yields 
$y_t^\star - (t-1)\eta \leq \eta y_t^{\eta,\star} \leq y_t^\star + T\eta$, and therefore $\PP(y_t^\star > u + (t-1)\eta \lvert x_1 = 0) \leq \PP(\eta y_t^{\eta,\star} > u \lvert x_1 = 0) \leq \PP(y_t^\star > u - T\eta \lvert x_1 = 0)$. Taking integration of $u$ over $[0,B]$ then gives
\begin{align}
    &\int_0^B \big(\PP(\eta y_t^{\eta,\star} > u \lvert x_1 = 0) - \PP(y_t^\star > u \lvert x_1 = 0)\big)_+ du \leq T\eta, \label{eq: general-demand-survival-signed-bounds-1}\\
    &\int_0^B \big(\PP(y_t^{\star} > u \lvert x_1 = 0) - \PP(\eta y_t^{\eta,\star} > u \lvert x_1 = 0)\big)_+ du \leq (t-1)\eta. \label{eq: general-demand-survival-signed-bounds-2}
\end{align}

\noindent\textbf{Step~3: putting all together.} Since $\eta y_t^{\eta,\star}$ takes values on the $\eta$-grid, we have $\PP(\eta y_t^{\eta,\star} > u \lvert x_1 = 0)$ is a constant function in $u$ over each interval $[j\eta, (j+1)\eta)$ with value $\PP(y_t^{\eta,\star} > j \lvert x_1 = 0)$,  while $\lambda_t$ is non-decreasing, $[0,1]$-valued. This gives
\begin{equation}
    0 \leq \int_0^B \PP(\eta y_t^{\eta,\star} > u \lvert x_1 = 0)\lambda_t(u)du - \eta\sum_{j=0}^{M-1} \PP(y_t^{\eta,\star} > j \lvert x_1 = 0)\lambda_t(j\eta)  \leq \eta.
    \label{eq: general-demand-weighted-riemann}
\end{equation}
To see~\eqref{eq: general-demand-weighted-riemann}, the constancy over each interval $[j\eta, (j+1)\eta)$ shown above implies that its middle term equals
$\sum\nolimits_{j=0}^{M-1} \PP(y_t^{\eta,\star} > j \lvert x_1 = 0) \int_{j\eta}^{(j+1)\eta} \big(\lambda_t(u) - \lambda_t(j\eta)\big)du$.
Each term in the sum is non-negative by the monotonicity of $\lambda_t$, which gives the first inequality in~\eqref{eq: general-demand-weighted-riemann}. For the second inequality, bounding the probabilities by one and using the monotonicity of $\lambda_t$ again, the sum above is at most
$\eta \sum_{j=0}^{M-1} \big(\lambda_t((j+1)\eta) - \lambda_t(j\eta)\big) = \eta \big(\lambda_t(B) - \lambda_t(0)\big) \leq \eta$, where the last step uses the telescoping structure of the sum and $0 \leq \lambda_t \leq 1$.
The desired bound in~\eqref{eq: general-demand-per-period-gap} now follows from $0\leq \lambda_t \leq 1$: combining the left inequality of~\eqref{eq: general-demand-weighted-riemann} with~\eqref{eq: general-demand-survival-signed-bounds-1} gives its upper-bound side, while combining the right inequality of~\eqref{eq: general-demand-weighted-riemann} with~\eqref{eq: general-demand-survival-signed-bounds-2} gives its lower side.

\subsection{Proof of Lemma~\ref{lem: online-general-process}}
We prove this result by induction: First at $t = 1$ the result holds directly by $\eta X_1^{\eta,(k)}  = 0 \leq x_1^{(k)}$ and $\eta Y_1^{\eta,(k)} = \eta Z^\eta(s_1^{(k)}) = s_1^{(k)} \leq y_1^{(k)}.$ Now suppose the result holds at some $t \in [T-1]$, then at $t+1$, we have 
    $\eta X_{t+1}^{\eta,(k)} = \big(\eta Y_t^{\eta, (k)} - \eta Z^\eta(d_t^{(k)}) \big)_+ \leq \big(\eta Y_t^{\eta, (k)} - d_t^{(k)} \big)_+ \leq  \big(y_t^{(k)} - d_t^{(k)} \big)_+  = x_{t+1}^{(k)},$
where the first inequality uses $\eta Z^\eta(d_t^{(k)}) \geq d_t^{(k)}$, and the second uses the induction hypothesis $\eta Y_t^{\eta,(k)} \leq y_t^{(k)}$. And then, $\eta Y_{t+1}^{\eta,(k)} = \max\{\eta X_{t+1}^{\eta,(k)}, s_{t+1}^{(k)}\} \leq y_{t+1}^{(k)}$ by $\eta Z^\eta(s_{t+1}^{(k)}) = s_{t+1}^{(k)}.$ Thus the result holds by induction.
\section{Details of the Numerical Experiments}\label{appendix-sec: experiment-details}

This appendix details the experiment setups of Section~\ref{sec: numerics}.

\noindent\textbf{Shared setup.} 
All experiments use a discrete demand support $[M]_+$ with $M=30$ and fixed costs $h\equiv 1$, $b\equiv 3$. The demand distributions are non-stationary and vary seasonally across periods. Each reported sub-optimality gap $\max_x\Delta(x;\hat\pi)$ is computed exactly by evaluating the learned policy under the true model via DP.

\noindent\textbf{Details of demand distributions.} 
Now we describe the detailed demand distributions.
\begin{itemize}[nosep]
\item In Figure~\ref{fig: offline}, we use a seasonal truncated Poisson distribution with mean
$
\lambda_t = 10\big(1+\frac{1}{2}\sin(2\pi t/T)\big)
$
on a horizon $T=6$, so $\lambda_t$ ranges over $[5,15]$ and the optimum $s_t^\star$ ranges over $\{7,\dots,17\}$. This smooth, fully supported instance is used to isolate the effect of coverage: the demand shape is fixed, while the logging policy's exploration range is varied.

\item In Figure~\ref{fig: pessimism}, we use a near-critical three-point distribution on a horizon $T=12$. Let
$
\ell_t = 12+\mathrm{round}(5\sin(2\pi t/T))\in\{7,\dots,17\}, 
H_t=\ell_t+10.
$
The demand places mass $0.50$, $0.26$, and $0.24$ on $\ell_t-1$, $\ell_t$, and $H_t$, respectively. Hence,
$
F_t(\ell_t-1)=0.50,\qquad F_t(\ell_t)=0.76,\qquad F_t(H_t)=1.
$
Solving the DP gives the optimal base-stock level $s_t^\star=\ell_t,\forall t\in [T]$ for this instance. 

\item In Figure~\ref{fig: online}, we use a bimodal two-point distribution on a horizon $T=12$. Let
$
a_t=12+\mathrm{round}(5\sin(2\pi t/T))\in\{7,\dots,17\},
H_t=a_t+10\in\{17,\dots,27\}.
$
The demand places mass $0.65$ on the low mode $a_t$ and mass $0.35$ on the high spike $H_t$. Thus,
$
F_t(a_t)=0.65, F_t(H_t)=1.
$
Solving the DP gives the high optimal base-stock level $s_t^\star=H_t,\forall t\in [T]$ for this instance. 
\end{itemize}

\noindent\textbf{Logging and data-generating policies.} 
For the coverage experiment, the logging policy draws $B_t\sim\mathrm{Uniform}\{0,\dots,U\}$ each period and orders up to $ y_t^b=\max\{x_t,B_t\}.$
The parameter $U$ controls the coverage level. For the pessimism experiment, the logging policy uses target $\ell_t$ with probability $0.95$ and target $M$ with probability $0.05$, and orders up to the maximum of the current inventory and the target. Thus $\ell_t$ and all higher coordinates are only thinly covered. For the online experiment, every method starts from the uninformed model $\hat{\bF}^{(1)}\equiv 0$, so the first episode orders up to $M$ and observes demand uncensored. 

\noindent\textbf{SAIL baseline.} 
We implement \textsc{SAIL}~\citep{qin2023sailing} following its main structure and run it on the same logged trajectories as our methods. From those trajectories, for each period $t$ and stock level $y$, we form the re-censored sample pool
$\{\min(\bar d_t^k,y): y_t^k\ge y\}.
$
We then apply variance-reduced value iteration to the observable virtual newsvendor costs $h(y-d)-b\,d$, with the per-period monotonicity bonus taken as a single period-level constant and the value function held flat above the observable boundary
$
\lambda_t=\max_k y_t^k.
$
\textsc{SAIL} returns a policy table, which is not necessarily a base-stock policy. We evaluate it by the same exact DP evaluation used for all other methods. Its online deployment, denoted by \textsc{SAIL-CE}, repeatedly re-solves this offline procedure on the data collected so far and then acts greedily, without extra exploration.

\bibliographystyle{informs2014}
\bibliography{main}

\newpage
\input{OR_EC.tex}

\end{document}

%% file: OR_EC.tex
\providecommand{\ECHowTheorems}{\setcounter{theorem}{0}}

\makeatletter
\if@BLINDREV
  \ECRUNAUTHOR{Authors' names blinded for peer review}
\fi
\makeatother

\ECSwitch

\renewcommand{\theequation}{\thesection.\arabic{equation}}

\renewcommand*{\theHsection}{EC.\arabic{section}}
\renewcommand*{\theHtheorem}{EC.\arabic{section}.\arabic{theorem}}
\ECHead{Electronic Companion}

\section{Proof of Results in Section~\ref{sec: general-decomposition}}

\subsection{Technical Lemmas and Notations.}

To present the proof of results in Section~\ref{sec: general-decomposition}, we first introduce several auxiliary stochastic dynamics and related events:

\begin{definition}[Time-reverse random walk]\label{def: Z-process}
For each index $j$ and $t'\in [T]$, we introduce the sequence of (time-reversed) random walks $\{Z^j_{t' \to t}\}_{1\leq t \leq t' }$ with starting level $j$ as \begin{equation}
\begin{aligned}
    Z_{t'\to t'}^j = j,&\quad Z^j_{t'\to t} = Z^j_{t'\to t+1} + d_{t}, \quad d_t \sim P_t,\quad \forall 1\leq t< t'. \end{aligned}
\end{equation}
\end{definition}
Noticing that by the demand independence across $t$, we always have for any $t'' \geq t',$
\begin{align}\label{eq: Z-markov-appendix}
   \PP\big( f\big( Z^j_{t' \to t'}, Z^j_{t' \to t'-1},\dots, Z^j_{t' \to 1}\big)  \big) = \PP\big( f\big( Z^k_{t'' \to t'},Z^k_{t'' \to t'-1},\dots, Z^k_{t'' \to 1}\big) \big\lvert Z_{t'' \to t'}^k = j \big)
\end{align}
for arbitrary measurable function $f$.
We also introduce the events
    $$\cE_{t, j}(\ell):= \big\{Z_{t \to \ell}^j \geq s_{\ell}\big\},\quad \cE_{t, j}^\star(\ell):= \big\{Z_{t \to \ell}^j \geq s^\star_{\ell}\big\},$$
for later reference. In addition, for integers $0 \leq a \leq b \leq M$, we denote by $\be(a,b) \in \RR^{M}$ the indicator vector of the index interval $[a,b)$, i.e. $[\be(a,b)]_j := \bm{1}\{a \leq j < b\}$ for $j \in [M-1]_+$.

This time-reverse random walk can provide a probabilistic interpretation of the $\bA_t$ as the following: 

\begin{proposition}[Probabilistic Interpretation of $\bA$]\label{remark: A-interpretation}
For each index $r,j\in [M-1]_+$ and $t\leq t'$, we have $\big[\bA_{t:t'-1}\big]_{rj} = \PP \big( Z^j_{t' \to t} = r, \cap_{t< \ell \leq t'} \cE_{t',j}(\ell)\big)$ for $\bA_{t:t'-1}:= \prod_{i = t}^{t'-1} \bA_i(s_{i+1})$.
\end{proposition}

\begin{proof}[Proof of Proposition~\ref{remark: A-interpretation}]
We prove this result by induction: When $t' = t$, we have the fact holds directly by $\bA_{t:t-1} =\bI$ and the definition of $Z^j_{t\to t'}$. Now for some $t' > t$, suppose the result holds for all $t\le k \leq t'$, then for $t'+1$, we have \begin{align*}
   &\big[ \bA_{t:t'}\big]_{rj} = \sum_{k = 0}^{M-1} \big[\bA_{t:t'-1}\big]_{rk} \big[\bA_{t'}(s_{t'+1})\big]_{kj} = \sum_{k = 0}^{M-1} \PP(Z_{t' \to t}^k = r, \cap_{t<\ell \leq t'} \cE_{t',k}(\ell) ) \PP(d_{t'} = k-j) \bm{1}\{j \geq s_{t'+1}\}\\
   &= \sum_{k = 0}^{M-1} \PP(Z_{t' \to t}^k = r, \cap_{t<\ell \leq t'} \cE_{t',k}(\ell) ) \PP(Z_{t' +1 \to t'}^j -  \underbrace{Z_{t'+1 \to t'+1}^j}_{ = j} = k-j)\bm{1}\{j \geq s_{t'+1}\}\\
   &= \sum_{k = 0}^{M-1} \PP(Z_{t' \to t}^k = r,  \cap_{t<\ell \leq t'} \cE_{t',k}(\ell) ) \PP(Z_{t' +1 \to t'}^j = k)\bm{1}\{j \geq s_{t'+1} \}.
\end{align*} 
Now by~\eqref{eq: Z-markov-appendix},\begin{align*}
    \PP(Z_{t' \to t}^k = r, \cap_{t<\ell \leq t'} \cE_{t',k}(\ell) ) =\PP(Z_{t'+1 \to t}^j = r, \cap_{t<\ell \leq t'} \cE_{t'+1,j}(\ell)\lvert Z^j_{t'+1 \to t'} = k )
\end{align*}
which then implies 
\begin{align*}
    &\sum_{k = 0}^{M-1} \PP(Z_{t' \to t}^k = r,  \cap_{t<\ell \leq t'} \cE_{t',k}(\ell) ) \PP(Z_{t' +1 \to t'}^j = k)\bm{1}\{j \geq s_{t'+1} \}\\
    &=\sum_{k = 0}^{M-1} \PP(Z_{t'+1 \to t}^j = r,Z^j_{t'+1 \to t'} = k,  \cap_{t<\ell \leq t'} \cE_{t'+1,j}(\ell))\bm{1}\{j \geq s_{t'+1} \}\\
    &=  \PP(Z_{t'+1 \to t}^j = r,  \cap_{t<\ell \leq t'} \cE_{t'+1,j}(\ell))\underbrace{\bm{1}\{j \geq s_{t'+1} \}}_{= \bm{1}\{Z_{t'+1\to t'+1}^{j} \geq s_{t'+1} \}} =  \PP(Z_{t'+1 \to t}^j = r,  \cap_{t<\ell \leq t'+1} \cE_{t'+1,j}(\ell)),
\end{align*}
as desired.
\end{proof}

\begin{lemma}\label{lem: J-closed-form}
    Given $x_1 = 0 $, under any base-stock policy determined by $\{s_t\}_{t=1}^T$, its entering inventory level $x_t$ and post-ordering inventory level $y_t$  satisfies \begin{align*}
        x_t &= \max_{1\leq \tau< t} \big\{\big(s_\tau - \sum_{r = \tau}^{t-1} d_r \big)_+ \big\},\qquad 
        y_t = \max_{1\leq \tau\leq t} \big\{\big(s_\tau - \sum_{r = \tau}^{t-1} d_r \big)_+ \big\}.
    \end{align*}
\end{lemma}
\begin{proof}[Proof of Lemma~\ref{lem: J-closed-form}]
    The result of $y_t$ can be directly obtained by noticing $y_t = \max\{s_t,x_t\}$ once we have the result for $x_t$. To prove the result for $x_t$, we can do this backward: \begin{align*}
        x_{t} 
        &= \max\{ \big(x_{t-1} - d_{t-1}\big)_+ , \big(s_{t-1}   - d_{t-1}\big)_+ \}\\
        &=\max\{ \big(x_{t-2} - d_{t-2} - d_{t-1}\big)_+, \big( s_{t-2} - d_{t-1}-d_{t-2}  \big)_+, \big(s_{t-1}   - d_{t-1}\big)_+ \}\\
        &= \dots = \max_{1\leq \tau\leq t-1} \big\{\big(s_\tau - \sum_{r = \tau}^{t-1} d_r \big)_+ \big\},
    \end{align*}
as desired.
\end{proof}

\subsection{Proof of Proposition~\ref{prop-convexity}}

To prove Proposition~\ref{prop-convexity}, we first introduce the following matrix representation of the trajectory visiting probability $x_t$: 

\begin{proposition}\label{prop: M-matrix}
    We have for $\bM_t:= \begin{bmatrix} \bm{m}_{t0}^\top & \bm{m}_{t1}^\top & \dots & \bm{m}_{tM}^\top \end{bmatrix}^\top \in \RR^{(M+1)\times (M+1)}$ with 
    \begin{align*} 
\bm{m}_{tj} &= \left(\begin{matrix}
    1 - F_t(a_{tj}-1) & \mu_{t,a_{tj} - 1} & \mu_{t,a_{tj} - 2} & \dots & \mu_{t 0} & \bm{0}_{M - a_{tj}}
    \end{matrix}\right), \quad\forall j \in [M]_+,
\end{align*}
$a_{tj}:= \max\{s_t,j\}$
It holds that $\big[\bM_{1:t-1} \big]_{rj} = \PP(x_t = j\lvert x_1 = r)$  for $\bM_{j:k}:= \prod_{i = j}^{k} \bM_{i}, t \in [T].$
\end{proposition}

\begin{proof}[Proof of Proposition~\ref{prop: M-matrix}]
First, for $t = 1$ the result holds directly since $[\bM_{1:0}]_{rj} = \bm{1}\{j = r\}.$ Now suppose the result holds for $t$, then for $t+1$, we have for $j\geq 1$,
\begin{align*}
     &\big[\bM_{1:t}\big]_{rj} = \big[\bM_{1:t-1} \bM_{t}\big]_{rj} = \sum_{k = 0}^{M} \big[\bM_{1:t-1} \big]_{rk} \big[\bM_{t}\big]_{kj}= \sum_{k = 0}^{M} \PP(x_t = k\lvert x_1 = r) (\bm{m}_{tk})_j\\
     & =\begin{cases}
         \sum_{k = 0}^{M} \PP(x_t = k\lvert x_1 =r) \mu_{t,a_{tk} - j} = \sum_{k = 0}^{M} \PP(x_t = k\lvert x_1 =r) \PP(d_t = a_{tk}-j) & \text{ if $j\geq 1$,} \\
         \sum_{k = 0}^{M} \PP(x_t = k\lvert x_1 =r) \big[1-F_{t}(a_{tk} -1 )\big] = \sum_{k = 0}^{M} \PP(x_t = k\lvert x_1 =r) \PP(d_t \geq a_{tk} ) & \text{ if $j = 0$.}
     \end{cases}
\end{align*} 
In both cases, we have the right-hand-side equals to $\PP(x_{t+1} = j\lvert x_1 =r),$ this finishes the proof.
\end{proof}

With Proposition~\ref{prop: M-matrix}, the decomposition in Lemma~\ref{lem: performance-difference} can be re-written as 
\begin{align*}
    \Delta(x;\hat{\pi}) = \sum_{t=1}^T \sum_{j = 0}^M \big[\bM_{1:t-1}\big]_{xj} \big[  W_t^\star\big(\hat{\pi}_t(j)\big) -  W_t^\star\big(\pi^\star_t(j)\big) \big]
\end{align*}
Noticing that by definition we have $W_t^\star\big(\hat{\pi}_t(j)\big) -  W_t^\star\big(\pi^\star_t(j)\big) \geq 0$ for all $j\in [M]_+$ (since $\pi_t^\star(j) = \arg\min_{y\geq j} W_t^\star(y)$ while $\hat\pi_t(j)\geq j$), denote $\bE_t \in \RR^{M+1}$ as the vector with $j$-th entry $E_t(j):= W_t^\star\big(\hat{\pi}_t(j)\big) -  W_t^\star\big(\pi^\star_t(j)\big)$ for $j \in [M]_+$. Then the monotonicity statement in Proposition~\ref{prop-convexity} can be covered by the following general property of $\bM:$
\begin{proposition}\label{prop-convexity-appendix} For every $1\leq t\leq T$, $\bM_{k:t-1} \bE_t$ is non-negative and non-increasing for every $1\le k\leq t$, in the sense that $
    \big[\bM_{k:t-1} \bE_t\big]_0 \geq \big[\bM_{k:t-1} \bE_t\big]_1 \geq \dots \geq \big[\bM_{k:t-1} \bE_t\big]_{M}\geq 0$.
\end{proposition}

\begin{proof}[Proof of Proposition~\ref{prop-convexity-appendix}]
First, by 
$\sigma_tD_t^\star(j)$ is non-negative for $\alpha_t\leq j < \beta_t,$ we have $E_t(j)$ is non-increasing by definition. Now it remains to show that for any non-negative, non-increasing vector $\bw$ and $1\leq t\leq T, \bM_t \bw$ is also a non-negative, non-increasing vector. To see this, we have for any $0<  j\leq M$ with $a_{t,j-1} = a_{tj} -1$ (otherwise $a_{t,j-1} = a_{tj}$ and the result holds directly)\begin{align*}
    (\bM_{t} \bm w)_{j} &=\big( 1-F_t(a_{tj}-1)\big)  w_0 + \sum_{k = 1}^{a_{tj}} \mu_{t,a_{tj} - k } w_k
     \leq  \big( 1-F_t(a_{tj}-1)\big)  w_0 + \sum_{k = 1}^{a_{tj}} \mu_{t,a_{tj} - k } w_{k-1} \\
    &= \big(1 - F_t(a_{t,j-1} - 1)\big) w_0 + \sum_{k = 1}^{a_{t,j-1}} \mu_{t,a_{t,j-1} - k} w_k = (\bM_t \bw)_{j-1},
\end{align*}
this finishes the proof.
\end{proof}

\subsection{Proof of Proposition~\ref{prop-matrix-notation}}

\begin{proof}[Proof of Proposition~\ref{prop-matrix-notation}]
To show the first claim, we have for every $m$ and $t$, by definition
\begin{align}
    D^\star_t(m) = \begin{cases}
        h_t F_t(m)  - b_{t }[1-F_t(m)] + \sum_{k = 0}^{m-s_{t+1}^\star} \mu_{tk}   D^\star_{t+1}(m-k)   &m \geq s_{t+1}^\star \\
         h_t F_t(m)  - b_t[1-F_t(m)]  &  m< s_{t+1}^\star
    \end{cases}.
\end{align}
This gives~\eqref{eq: D-forward-equation}, and~\eqref{eq: D-hat-forward-equation} can be shown in similar way.
By applying~\eqref{eq: D-forward-equation},~\eqref{eq: D-hat-forward-equation} directly, we can get the second claim. 
\end{proof}

\subsection{Proof of Remark~\ref{remark: v-interpretation}}

\begin{proof}[Proof of Remark~\ref{remark: v-interpretation}]
By definition, we have
\begin{align*}
    \bv_{t} =  \bu_t + \sum_{\ell = 1}^{t-1} \bA_{\ell: t-1}^\top \bu_\ell,
\end{align*}
and, observing that $$\sigma_t \bm{1}\{\alpha_t \leq j < \beta_t\} = \bm{1}\{j \geq s_t^\star\} - \bm{1}\{j \geq s_t\},\qquad \PP(x_t^\star \leq j \mid x_1 = 0) = \PP(Z_{t \to \tau}^j \geq s_\tau^\star \forall 1 \leq \tau < t),$$ $\bu_t$ can be equivalently written as the difference \begin{align*}
   [\bu_t]_j &= \PP\big( {\underbrace{Z_{t \to t}^j \geq s_{t}^\star}_{\iff j \geq s_{t}^\star},}  Z_{t\to \tau}^j \geq s_{\tau}^\star,  \forall 1 \leq  \tau < t \big)
    - \PP \big( {\underbrace{Z_{t \to t}^j \geq s_{t}}_{\iff j \geq s_{t}},}  Z_{t\to \tau}^j \geq s_{\tau}^\star,  \forall 1 \leq  \tau < t \big)\\
    &= \PP \big(\cap_{\tau = 1}^t \cE_{t,j}^\star(\tau) \big) - \PP \big( \cE_{t,j}(t) ,\cap_{\tau = 1}^{t-1} \cE_{t,j}^\star(\tau) \big).
\end{align*}
Then by 
\begin{align*}
    (\bu_t)_j =  \PP \big( \cap_{\ell = 1}^{t} \cE_{t,j}^\star(\ell)   \big) - \PP \big(\cE_{t,j}(t) , \cap_{\ell = 1}^{t-1} \cE_{t,j}^\star(\ell)   \big),
\end{align*}
it holds that \begin{align*}
    &\big(\bA_{t:t'-1}^\top \bu_t\big)_j = \sum_{r}(\bA_{t:t'-1})_{rj} (\bu_t)_r\\
    &=\sum_{r} \PP \big(Z_{t'\to t}^j = r, \cap_{t<\ell\leq t'} \cE_{t',j}(\ell) \big)\big[\PP \big( \cap_{\ell = 1}^{t} \cE_{t,r}^\star(\ell)   \big) - \PP \big(\cE_{t,r}(t)  \cap \cap_{\ell = 1}^{t-1} \cE_{t,r}^\star(\ell)   \big) \big].
\end{align*}
Noticing that by~\eqref{eq: Z-markov-appendix},
\begin{align*}
    &\PP \big(Z_{t'\to t}^j = r, \cap_{t<\ell\leq t'} \cE_{t',j}(\ell) \cap \cap_{\ell = 1}^{t} \cE_{t',j}^\star(\ell) \big)=\PP \big(Z_{t'\to t}^j = r, \cap_{t<\ell\leq t'} \cE_{t',j}(\ell) \cap \cap_{\ell = 1}^{t} \cE_{t',j}^\star(\ell) \big)\\
    &= \PP \big(Z_{t'\to t}^j = r, \cap_{t<\ell\leq t'} \cE_{t',j}(\ell)\big) \PP \big(    \cap_{\ell = 1}^{t} \cE_{t',j}^\star(\ell) \big\lvert Z_{t'\to t}^j = r, \cap_{t<\ell\leq t'} \cE_{t',j}(\ell) \big)\\
    &=  \PP \big(Z_{t'\to t}^j = r, \cap_{t<\ell\leq t'} \cE_{t',j}(\ell)\big)\PP \big( \cap_{\ell = 1}^{t} \cE_{t,r}^\star(\ell)   \big),
\end{align*}
and similarly,
\begin{align*}
   &\PP \big(Z_{t'\to t}^j = r, \cap_{t<\ell\leq t'} \cE_{t',j}(\ell) \cap \cE_{t',j}(t) \cap_{\ell = 1}^{t-1} \cE_{t',j}^\star(\ell) \big)=  \PP \big(Z_{t'\to t}^j = r, \cap_{t<\ell\leq t'} \cE_{t',j}(\ell)\big)\PP \big( \cE_{t,r}(t) \cap_{\ell = 1}^{t-1} \cE_{t,r}^\star(\ell)   \big),
\end{align*}
thus \begin{align*}
    \big(\bA_{t:t'-1}^\top \bu_t\big)_j  &= \sum_{r}\bigg[\PP \big(\{Z_{t'\to t}^j = r\}\cap \cap_{t<\ell\leq t'} \cE_{t',j}(\ell) \cap \cap_{\ell = 1}^{t} \cE_{t',j}^\star(\ell) \big)\\
    &- \PP \big(\{Z_{t'\to t}^j = r\}\cap \cap_{t<\ell\leq t'} \cE_{t',j}(\ell) \cap \cE_{t',j}(t) \cap_{\ell = 1}^{t-1} \cE_{t',j}^\star(\ell) \big)\bigg]\\
    &= \PP \big(\cap_{t<\ell\leq t'} \cE_{t',j}(\ell) \cap \cap_{\ell = 1}^{t} \cE_{t',j}^\star(\ell)\big) - \PP \big(\cap_{t\leq \ell\leq t'} \cE_{t',j}(\ell) \cap \cap_{\ell = 1}^{t-1} \cE_{t',j}^\star(\ell)\big).  
\end{align*}
Thus via telescoping summation, we can have 
\begin{equation}\label{eq: v-probabilistic-interpretation}
\begin{aligned}
    (\bv_t)_j &= \sum_{k = 1}^t \big(\PP \big(\cap_{k<\ell\leq t} \cE_{t,j}(\ell) \cap \cap_{\ell = 1}^{k} \cE_{t,j}^\star(\ell)\big) - \PP \big(\cap_{k\leq \ell\leq t} \cE_{t,j}(\ell) \cap \cap_{\ell = 1}^{k-1} \cE_{t,j}^\star(\ell)\big)\big)\\
    &= \PP \big(\cap_{1\leq \ell \leq t} \cE_{t,j}^\star(\ell) \big) - \PP \big(\cap_{1\leq \ell \leq t} \cE_{t,j}(\ell) \big) = \PP \big(y^\star_{t} \leq j \lvert x_1 = 0 \big) - \PP \big(y_t \leq j\lvert x_1 = 0 \big).
\end{aligned}    
\end{equation}
\end{proof}

\subsection{Proof of Proposition~\ref{prop: v-propagation}}

\begin{proof}[Proof of Proposition~\ref{prop: v-propagation}]

For every $1\leq t < T,$ if we denote \begin{align*}
    \bB_{t}&:=\bA_{t}(s_{t+1}^\star) - \bA_t(s_{t+1}) = \sigma_{t+1}\big[\bA_{t}(\alpha_{t+1}) - \bA_{t}(\beta_{t+1}) \big].
\end{align*}
Then, by definition of $\bv_t,$ it holds that \begin{align*}
    \bB_{t}^\top \bv_{t}= \sum_{\ell = 1}^{t} \underbrace{\bB_{t}^\top   \bA_{\ell: t-1}^\top}_{:= \bC^\top_\ell} \bu_\ell .
\end{align*}
It holds that by \begin{align*}
     \big[\bu_\ell\big]_j = \PP \big(\cap_{k = 1}^\ell \cE_{\ell,j}^\star(k) \big) - \PP \big(\cE_{\ell,j}(\ell) ,\cap_{k = 1}^{\ell-1} \cE_{\ell,j}^\star(k) \big),
\end{align*}
and Proposition~\ref{remark: A-interpretation},
\begin{align*}
   &\big[ \bC^\top_\ell \bu_\ell\big]_j = \sum_{r = 0}^{M-1} \big[\bu_{\ell}\big]_r  (\bC_\ell)_{rj}\\
   &\overset{\text{(a)}}{=} \sigma_{t+1} \bm{1}\{\alpha_{t+1}\leq j< \beta_{t+1} \} \sum_{r = 0}^{M -1} \big[\bu_{\ell}\big]_r \PP \big(Z^j_{t+1 \to \ell} = r,   Z^j_{t+1 \to k} \geq s_k, \ell < k<t+1 \big)\\
   &=\sigma_{t+1}  \sum_{r = 0}^{M-1} \bigg[\PP(\cap_{k = 1}^\ell \cE_{\ell,r}^\star(k)) - \PP( \cE_{\ell,r}(\ell) ,\cap_{k = 1}^{\ell-1} \cE_{\ell,r}^\star(k))\bigg] \PP \big(Z^j_{t+1 \to \ell} = r, \cap_{\ell < k < t+1} \cE_{t+1,j}(k)\big).
\end{align*}
Where in~(a) we have used $$[\bB_t]_{rj} =\begin{cases}
    \sigma_{t+1} \big[\bA_t\big]_{rj}  & \text{ if }  \alpha_{t+1}\leq j< \beta_{t+1},\\
    0 & \text{otherwise}.
\end{cases} $$

Noticing that as long as $\PP(Z_{t+1\to \ell}^j =r ) \neq 0$(otherwise both-sides are $0$ and the result holds directly), we have by the independence of demands across $t$, \begin{align*}
&\PP(\cap_{k = 1}^\ell \cE_{\ell,r}^\star(k))\PP \big(Z^j_{t+1 \to \ell} = r, \cap_{\ell < k < t+1} \cE_{t+1,j}(k)\big)  \\
&=  \PP(\cap_{k = 1}^\ell \cE_{t+1,j}^\star(k) \lvert Z_{t+1\to \ell}^j = r )\PP \big(Z^j_{t+1 \to \ell} = r, \cap_{\ell < k < t+1} \cE_{t+1,j}(k)\big) \\
&=  \PP(\cap_{k = 1}^\ell \cE_{t+1,j}^\star(k) \lvert Z_{t+1\to \ell}^j = r , \cap_{\ell < k < t+1}\cE_{t+1,j}(k))\PP \big(Z^j_{t+1 \to \ell} = r, \cap_{\ell < k < t+1} \cE_{t+1,j}(k)\big) \\
&= \PP(\cap_{k = 1}^\ell \cE_{t+1,j}^\star(k) , Z_{t+1\to \ell}^j = r , \cap_{\ell < k < t+1}\cE_{t+1,j}(k)).
\end{align*}
And similarly, we can show that\begin{align*}
    &\PP(\cE_{\ell, r}(\ell),\cap_{k = 1}^{\ell-1} \cE_{\ell,r}^\star(k)))\PP \big(Z^j_{t+1 \to \ell} = r, \cap_{\ell < k < t+1} \cE_{t+1,j}(k)\big)\\
&=\PP(\cE_{t+1,j}(\ell) ,\cap_{k = 1}^{\ell-1} \cE_{t+1,j}^\star(k) , Z_{t+1\to \ell}^j = r , \cap_{\ell < k < t+1}\cE_{t+1,j}(k)) .
\end{align*}
thus \begin{align*}
    &\big[\bC_\ell^\top \bu_\ell\big]_j = \sigma_{t+1} \sum_{r =0}^{M-1} \bigg[\PP(\cap_{k = 1}^\ell \cE_{t+1,j}^\star(k) , Z_{t+1\to \ell}^j = r , \cap_{\ell < k < t+1}\cE_{t+1,j}(k)) \\
    &-\PP(\cE_{t+1,j}(\ell) ,\cap_{k = 1}^{\ell-1} \cE_{t+1,j}^\star(k) , Z_{t+1\to \ell}^j = r , \cap_{\ell < k < t+1}\cE_{t+1,j}(k))\bigg]\\
    &= \sigma_{t+1}\bigg[\PP(\cap_{k = 1}^\ell \cE_{t+1,j}^\star(k) , \cap_{\ell < k < t+1}\cE_{t+1,j}(k)) - \PP( \cap_{k = 1}^{\ell-1} \cE_{t+1,j}^\star(k) , \cap_{\ell \leq k < t+1}\cE_{t+1,j}(k)) \bigg].
\end{align*}
Now taking summation of $\ell$ from $\ell = 1$ to $t$, we have \begin{align*}
    \big[\bB_t^\top \bv_t\big]_{j} &=  \sigma_{t+1}\bm{1}\{\alpha_{t+1} \leq j< \beta_{t+1} \} \sum_{\ell = 1}^t \bigg[\PP(\cap_{k = 1}^\ell \cE_{t+1,j}^\star(k) , \cap_{\ell < k < t+1}\cE_{t+1,j}(k)) \\
    &- \PP( \cap_{k = 1}^{\ell-1} \cE_{t+1,j}^\star(k) , \cap_{\ell \leq k < t+1}\cE_{t+1,j}(k)) \bigg]\\
    &=  \sigma_{t+1}\bm{1}\{\alpha_{t+1} \leq j< \beta_{t+1} \}\big[ \PP \big(\cap_{k = 1}^t \cE_{t+1,j}^\star(k) \big) - \PP\big(\cap_{k=1}^{t} \cE_{t+1,j}(k) \big)\big].
\end{align*}
Finally, by Lemma~\ref{lem: J-closed-form}, we have
\begin{align*}
[\bq_{t+1}]_j &= \bm{1}\{\alpha_{t+1} \leq j < \beta_{t+1}\} \sigma_{t+1} \PP(x_{t+1} \leq j)\\
&=  \bm{1}\{\alpha_{t+1} \leq j < \beta_{t+1}\} \sigma_{t+1} \PP\big(\max_{1\leq k \leq t}\big(s_{k} - \sum_{\ell = k}^{t} d_\ell\big)_+ \leq j \big)\\
&=   \bm{1}\{\alpha_{t+1} \leq j < \beta_{t+1}\} \sigma_{t+1} \PP\big(\max_{1\leq k\leq t}\big(s_{k} - \sum_{\ell = k}^{t} d_\ell\big) \leq j \big)\\
&=   \bm{1}\{\alpha_{t+1} \leq j < \beta_{t+1}\} \sigma_{t+1} \PP\big(\underbrace{ j+\sum_{\ell = k}^{t} d_\ell \geq s_k ,  \forall 1\leq  k\leq t }_{\cap_{k = 1}^t \cE_{t+1,j}(k)}\big).
\end{align*}
This gives \begin{align*}
     \big[\bB_t^\top \bv_t\big]_{j} + [\bq_{t+1}]_j = \sigma_{t+1}\bm{1}\{\alpha_{t+1}\leq j< \beta_{t+1}\} \PP(\cap_{k=1}^t \cE_{t+1,j}^\star(k)) = \big[\bu_{t+1}\big]_j,
\end{align*}
as desired.

\end{proof}

\section{Proof of Results in Section~\ref{sec: uniform-convergence}}\label{appendix-sec: proof-of-uniform-convergence-section}

In the proof of results in Section~4, we introduce the following notation for describing propagation of derivatives, which is a direct result of~\eqref{eq: D-forward-equation} and~\eqref{eq: D-hat-forward-equation}, for later reference: It holds for any $t \leq t'$ that \begin{align}\label{eq: D-recursion}
    \bD_t &=  \underbrace{\sum_{k = t}^{t'-1} \prod _{r = t}^{k-1} \bA_r(s_{r+1}^\star)  \big[ h_k \bF_k - b_k (\bm{1}- \bF_k) \big]}_{:= \bS_{t:t'-1}}   + \underbrace{\prod_{k = t}^{t'-1} \bA_k(s_{k+1}^\star)}_{= \bA_{t:t'-1}}\bD_{t'},\\
    \label{eq: D-hat-recursion}
    \hat\bD_t &=  \underbrace{\sum_{k = t}^{t'-1} \prod _{r = t}^{k-1} \hat\bA_r(s_{r+1})  \big[ h_k \hat\bF_k - b_k (\bm{1}- \hat\bF_k) \big]}_{:= \hat\bS_{t:t'-1}}   + \underbrace{\prod_{k = t}^{t'-1} \hat\bA_k(s_{k+1})}_{= \hat\bA_{t:t'-1}}\hat\bD_{t'}.
\end{align}

Introducing the notation $[\hat{\bw}_t]_j:= \hat{\PP}(y_t > j\lvert x_1 = 0)$ with $\hat{\PP}(\cdot)$ the probability taken under the environment induced by $\{\hat{\bF}_t\}_{t=1}^T,$ we provide the following basic perturbation lemma for later reference.

\begin{lemma}\label{lem: w-D-infty-bounds}
    Under~\eqref{eq: uniform-convergence}, we have \begin{align}\label{eq: bounds-for-w-D}
    \lVert \bw\rVert_\infty \leq 1, \quad \lVert \bD_t \rVert_\infty \leq (h_\infty +b_\infty)T      ,\quad \lVert \Delta \bw_t\rVert_\infty \leq T\epsilon, \quad \lVert \Delta  \bD_t\rVert_\infty \leq 4T^2(h_\infty +b_\infty) \epsilon.
    \end{align}
\end{lemma}
\begin{proof}[Proof of Lemma~\ref{lem: w-D-infty-bounds}]
The first bound is immediate since $[\bw_t]_j = \PP(y_t>j\lvert x_1=0)$. 

For the second bound, applying~\eqref{eq: D-recursion} with $t'=T+1$ gives 
\begin{align*}
\lVert \bD_t\rVert_\infty &= \big\lVert\sum_{k=t}^{T}\big(\prod_{r=t}^{k-1}\bA_r(s_{r+1}^\star)\big)\big[h_k\bF_k - b_k(\bm 1 - \bF_k)\big]\big\rVert_\infty \leq \sum_{k=t}^T \lVert h_k\bF_k -b_k(\bm{1}- \bF_k)\rVert_\infty
\end{align*}
by $\sum_{j}\lvert [\bA_r(s)]_{ij}\rvert = \PP(d_r\le i-s)\le 1$.

To see the third bound, first define the mixed distribution $\QQ_{t}$ of demands so that the demands for time period $1\leq \tau<t$ following the distribution with CDF $\hat{\bF}_\tau$, while demands for time period $t\leq \tau\leq T$ follows the distribution with CDF ${\bF}_\tau$. Then by Lemma~\ref{lem: J-closed-form}, and the fact that $\cap_{\tau = 1}^t \cE_{t,j}(\tau)$ depends only on $\{d_\tau\}_{1\leq \tau \leq t-1},$
\begin{align*}
\lvert [\bw_t]_j-[\hat\bw_t]_j\rvert &= \lvert \hat\PP(y_t\le j\lvert x_1 = 0)-\PP(y_t\le j\lvert x_1 = 0)    \rvert\\
&= \lvert \hat\PP(\cap_{\tau=1}^t\cE_{t,j}(\tau))-\PP(\cap_{\tau=1}^t\cE_{t,j}(\tau))    \rvert\\
&= \lvert \QQ_{t}(\cap_{\tau=1}^t\cE_{t,j}(\tau)) - \QQ_{1}(\cap_{\tau=1}^t\cE_{t,j}(\tau))\rvert\\
&\leq \sum_{k = 1}^{t-1} \lvert \QQ_{k+1}(\cap_{\tau=1}^t\cE_{t,j}(\tau))-\QQ_{k}(\cap_{\tau=1}^t\cE_{t,j}(\tau)) \rvert.
\end{align*}
On the other hand, for each $k$, $\cap_{\tau=1}^t \cE_{t,j}$ can be written as a threshold event of $\{d_k \geq g_{t,j}(k)\}$ for some $g_{t,j}(k)$ depending only on $\{d_{\tau}\}_{1\leq \tau \leq t-1,\tau\neq k},\{s_{\tau}\}_{1\leq \tau \leq t}$ and $j$, thus 
\begin{align*}
    \lvert \QQ_{k+1}(\cap_{\tau=1}^t\cE_{t,j}(\tau))-\QQ_{k}(\cap_{\tau=1}^t\cE_{t,j}(\tau)) \rvert \leq \EE_{g_{t,j}(k)}[\lvert \hat{F}_k(g_{t,j}(k)) - F_k(g_{t,j}(k))\rvert]\leq \epsilon.
\end{align*}
Taking summation over $k$ then leads to the desired bound.

Finally, to see the last bound, noticing that by~\eqref{eq: recursion-D-diff}, and the fact $$[\bA_{t}(s_{t+1}^\star)-\bA_t(s_{t+1})]\hat{\bD}_{t+1} \leq 0,\quad[\bA_{t}(s_{t+1}^\star)-\bA_t(s_{t+1})]\bD_{t+1}^\star \geq 0,$$ which hold because the difference $\bA_t(s_{t+1}^\star)-\bA_t(s_{t+1})$ is supported on the index range between $s_{t+1}$ and $s_{t+1}^\star$, on which $\hat{\bD}_{t+1}$ and $\bD_{t+1}^\star$ carry opposite signs,
we have
\begin{align*}
    \Delta \bD_t &\leq \bA_t(s_{t+1}^\star) \Delta \bD_{t+1} + \Delta\bA_t(s_{t+1})\hat{\bD}_{t+1} + (h_t+b_t)\Delta \bF_t,\\
    \Delta \bD_t &\geq \bA_t(s_{t+1}) \Delta \bD_{t+1} + \Delta\bA_t(s_{t+1})\hat{\bD}_{t+1} + (h_t+b_t)\Delta \bF_t.
\end{align*}
This gives \begin{align*}
    \lVert \Delta \bD_t\rVert_\infty\leq \max\big(\lVert \bA_t(s_{t+1}) \Delta \bD_{t+1}\rVert_\infty,\lVert \bA_t(s_{t+1}^\star) \Delta \bD_{t+1}\rVert_\infty\big) + \lVert \Delta\bA_t(s_{t+1})\hat{\bD}_{t+1}\rVert_\infty + (h_\infty+b_\infty) \epsilon.
\end{align*}
For the first term, we have
\begin{align*}
    \lVert \bA_t(s_{t+1}^\star) \Delta \bD_{t+1}\rVert_\infty &\leq  \max_{k}\sum_{j=0}^{M-1} \lvert[ \bA_t(s_{t+1}^\star)]_{kj} \rvert  \lVert \Delta \bD_{t+1}\rVert_\infty\leq \lVert \Delta \bD_{t+1}\rVert_\infty,\\
    \lVert \bA_t(s_{t+1}) \Delta \bD_{t+1}\rVert_\infty &\leq  \max_{k}\sum_{j=0}^{M-1} \lvert[ \bA_t(s_{t+1})]_{kj} \rvert  \lVert \Delta \bD_{t+1}\rVert_\infty\leq \lVert \Delta \bD_{t+1}\rVert_\infty.
\end{align*}
For the second term,  we have by~\eqref{eq: AD-to-GH},\begin{align*}
    \lVert \Delta \bA_t (s_{t+1})\hat{\bD}_{t+1}\rVert_\infty \leq \epsilon \lVert \bH_{t}\rVert_1 \leq \epsilon \lVert \hat{\bD}_{t+1} \rVert_\infty \leq \epsilon(h_\infty +b_\infty)T, 
\end{align*}
where in the last inequality we have used $\lVert \hat{\bD}_{t+1} \rVert_\infty
\leq T(h_\infty+b_\infty)$ by the same reason in getting the second bound.

Combining these bounds together leads to \begin{align*}
    \lVert \Delta\bD_{t}\rVert_\infty \leq \lVert \Delta\bD_{t+1}\rVert_\infty + 2\epsilon (h_\infty + b_\infty) T.
\end{align*}
Applying this inequality recursively then gives $\lVert \Delta\bD_{t}\rVert_\infty\leq 4 T^2\epsilon (h_\infty + b_\infty) $

\end{proof}

\subsection{Proof of Corollary~\ref{coro: uniform-convergence}}
Throughout the proof, we suppose $\varepsilon\leq 1/8,$ otherwise the result holds directly by the trivial bound $\Delta(0;\hat{\pi})\lesssim (h_\infty + b_\infty) M T$.
By Theorem~\ref{thm: cost-decomposition} and the non-positivity of the policy-mismatch term $\sum_{t=1}^T \langle \bu_t, \hat{\bD}_t \rangle \leq 0$ (cf. the discussion after Theorem~\ref{thm: cost-decomposition}), it suffices to bound
\begin{align*}
    S:= \sum_{t=1}^T \langle \bv_t, \Delta \bA_t(s_{t+1}) \hat{\bD}_{t+1} + \Delta \bc_t \rangle.
\end{align*}
The instantaneous cost term is easy: the CDF error guarantee~\eqref{eq: uniform-convergence} directly gives $\lVert \Delta \bc_t \rVert_\infty \leq (h_\infty+b_\infty) \epsilon,$ and since $\lVert \bv_t \rVert_\infty \leq 1,$
\begin{align*}
    \sum_{t=1}^T \langle \bv_t, \Delta \bc_t \rangle \leq (h_\infty +b_\infty) MT\epsilon.
\end{align*}
It remains to control the propagation term $\sum_{t=1}^T \langle \bv_t, \Delta \bA_t(s_{t+1}) \hat{\bD}_{t+1} \rangle$, which we show is of order $M(h_\infty + b_\infty)(T\epsilon + T^3\epsilon^2).$ The naive bound $\hat{\bD}_{t+1} = \cO(T-t)$ via~Proposition~\ref{prop-matrix-notation} would only give $\cO(MT^2\epsilon)$ result, so a more careful argument is needed. We do so by a discrete summation-by-parts transform followed by a bulk-tail decomposition. Detailed proofs are deferred to Appendix~\ref{appendix-sec: proof-of-uniform-convergence-section}.

\noindent\textbf{A Bulk-Tail Decomposition.} A summation-by-parts identity yields the factorization
\begin{equation}\label{eq: AD-to-GH}
\begin{aligned}
    &\Delta \bA_t(s_{t+1}) \hat{\bm D}_{t+1} = \bG_t \bH_t
\end{aligned}
\end{equation}
with \begin{align}
    [\bG_t]_{ij} &= \begin{cases}
    F_t(i-j) - \hat{F}_t(i-j), \text{ if } s_{t+1} \leq j \leq i \leq M-1,\\
    0, \text{ otherwise.}
\end{cases}\\
 [\bH_t]_j &= \begin{cases}
    0 & \text{ if } 0 \leq j < s_{t+1},\\
    \hat{D}_{t+1}(s_{t+1}) & \text{ if } j = s_{t+1},\\
    \hat{D}_{t+1}(j) - \hat{D}_{t+1}(j-1) & \text{ if } M-1 \geq j > s_{t+1}.\\
\end{cases}\label{eq: H-def}
\end{align}

We then consider the index $k_t:= \inf\{k \geq 0: F_t(k) \geq \frac{1}{2}\}$ and decompose $\bG_{t} = \bL_t + \bR_t$, where the bulk part $\bL_t$ contains the sub-diagonal lines of $\bG_t$ corresponding to large CDF coordinate indices $\ell \geq k_t$, and the tail part $\bR_t$ contains the remaining elements.

\noindent\textbf{Bounding the Bulk Part.} We show the summed bulk bound:
\begin{align*}
    \sum_{t=1}^T \lvert \langle \bv_t, \bL_t \bH_t \rangle\rvert \leq  MT(h_\infty  +b_\infty)\big[16\epsilon  + 20 T^2 \epsilon^2 \big].
\end{align*}
The key ingredient is a self-bounding property for the bulk part: since all non-zero elements in $\bL_t$ are of the form $F_t(j) - \hat{F}_t(j)$ for some $j\geq k_t,$ we have  
\begin{align}\label{eq: G-self-bounding}
     F_t(k_{t})\geq \frac{1}{2} \text{ and } \epsilon \leq \frac{1}{4} \implies \inf_{j\geq k_t}F_t(j) \geq \frac{1}{2} \implies \lvert F_t(j) -\hat{F}_t(j) \rvert \leq 2F_t(j) \epsilon \leq 4 \hat{F}_t(j) \epsilon, \quad \forall  j\geq k_{t}.
\end{align}
 This gives the element-wise bound $\big\lvert [\bL_{t}]_{ij} \big\rvert \leq 4\epsilon \hat{F}_t(i-j)\,\bm{1}\{s_{t+1} \leq j \leq i \leq M-1\}$, and thus
\begin{equation}\label{eq: L-bound}
\lvert \bL_{t} \bH_{t} \rvert  \leq \lvert \bL_{t}  \rvert \bH_{t} \leq 4\epsilon \hat{\bA}_t(s_{t+1}) \hat{\bD}_{t+1}.      
\end{equation} 
entry-wisely. This then gives
\begin{align*}
   \lvert  \langle \bv_t, \bL_t\bH_t \rangle \rvert  &\leq 4\epsilon \langle \lvert \bv_t \rvert, \hat{\bA}_t(s_{t+1})\hat{\bD}_{t+1} \rangle.
\end{align*}
Noticing that by the probabilistic interpretation in Remark~\ref{remark: v-interpretation}, we have \begin{align*}
    \lvert \bv_t \rvert_j &= \lvert \PP(y_t^\star \leq j\lvert x_1 = 0) - \PP(y_t \leq j\lvert x_1 = 0) \rvert
    \\
    &= \lvert \PP(y_t^\star > j\lvert x_1 = 0) -\PP(y_t > j\lvert x_1 = 0)\rvert\\
    &\leq  \underbrace{\PP(y_t^\star > j\lvert x_1 = 0)}_{:= [\bw^\star]_j} + \underbrace{\PP(y_t > j\lvert x_1 = 0)}_{[\bw_t ]_j}.
\end{align*}
We further have
\begin{align*}
 4\epsilon \langle \lvert \bv_t \rvert, \hat{\bA}_t(s_{t+1})\hat{\bD}_{t+1} \rangle &\leq 4\epsilon \langle \bw_t + \bw_t^\star, \hat{\bA}_t(s_{t+1})\hat{\bD}_{t+1} \rangle\\
 &= 4 \epsilon \langle \bw_t + \bw_t^\star, \hat{\bD}_{t} + b_t (\bm{1}- \hat{\bF}_t ) - h_t \hat{\bF}_t\rangle\\
   &\leq 4\epsilon \langle \bw_t + \bw_t^\star, \hat{\bD}_{t}\rangle + 8\epsilon M(h_\infty +b_\infty),
\end{align*}
where the first inequality uses $\lvert \bv_t \rvert \leq \bw_t + \bw_t^\star$  and $\hat{\bA}_t(s_{t+1})\hat{\bD}_{t+1} \geq \bm{0}$.

For the first term, we have\begin{align*}
     &\langle \bw_t + \bw_t^\star, \hat{\bD}_{t}\rangle = \langle \hat{\bw}_t, \hat{\bD}_t \rangle + \langle \bw^\star_t, \bD^\star_t\rangle +  \langle \underbrace{\bw_t - \hat{\bw}_t}_{:= \Delta \bw_t}, \hat{\bD}_t \rangle + \langle \bw_t^\star, \hat\bD_t-\bD_t^\star \rangle\\
     &\leq  \sum_{j= 0}^{M-1} \hat{\PP}(y_t > j\lvert x_1 = 0)[\hat{\bD}_t]_j + \sum_{j= 0}^{M-1} {\PP}(y^\star_t > j\lvert x_1 = 0)[{\bD}_t^\star]_j + \lVert \Delta \bw_t \rVert_\infty \lVert \hat{\bD}_t \rVert_1 + \lVert \bw_t^\star\rVert_1 \lVert \Delta \bD_t\rVert_\infty\\
     &\overset{\text{(a)}}{\leq}  \sum_{j= s_{t}}^{M} \hat{\PP}(y_t =  j\lvert x_1 = 0)[\hat{C}_t(j) - \hat{C}_t(s_t)] + \sum_{j= s_t^\star}^{M}{\PP}(y^\star_t =  j\lvert x_1 = 0)[{C}^\star_t(j) - {C}^\star_t(s_t^\star)]\\
     &+ \lVert \Delta \bw_t \rVert_\infty \lVert \hat{\bD}_t \rVert_1 + \lVert \bw_t^\star\rVert_1 \lVert \Delta \bD_t\rVert_\infty.\end{align*}
Here (a) is by summation-by-parts: for the first term we have,
\begin{align*}
    \sum_{j=0}^{M-1}\hat\PP(y_t>j\lvert x_1 = 0)[\hat\bD_t]_j = \underbrace{\big[\hat W_t(s_t) - \hat W_t(0)\big]}_{\leq 0} + \sum_{j=s_t}^{M}\hat\PP(y_t = j\lvert x_1 = 0)[\hat C_t(j) - \hat C_t(s_t)],
\end{align*}
and the leading term is non-positive because $s_t = \arg\min_{y} \hat W_t(y)$. The second term follows the same argument.

For the last two terms, Lemma~\ref{lem: w-D-infty-bounds} gives $ \lVert \Delta \bw_t \rVert_\infty \lVert \hat{\bD}_t \rVert_1 + \lVert \bw_t^\star\rVert_1 \lVert \Delta \bD_t\rVert_\infty  \leq 5T^2(h_\infty+b_\infty) M\epsilon$, so it remains to control the first two terms, noticing that they are of the form of expected over-shooting gap induced by optimal policy under each ($\bF$ or $\hat{\bF}$ induced) environments. This can be bounded by $\cO(1)$ as in the following statement. 
\begin{proposition}\label{prop: over-shooting-price}
    Given any environment with independent demand distributions $\bP:= (P_1,\dots,P_T)$ denote $\{s_{t}^\star\}_{1 \leq t\leq T}$ the optimal policy under $\bP$ and $y^\star_{t}$ the post-ordering inventory level process driven by the optimal policy, we have then \begin{align*}
        \sum_{j= s_{t}^\star+1}^M \PP (y_{t}^\star = j \lvert x_1 = 0 ) \big[C_{t}^\star(j) - C_{t}^\star(s_t^\star) \big] \leq (h_\infty + b_\infty)M.
    \end{align*}
\end{proposition}
Combining all bounds together gives the per-period bulk bound $\lvert \langle \bv_t, \bL_t \bH_t \rangle\rvert \leq  M(h_\infty  +b_\infty)\big[16\epsilon  + 20 T^2 \epsilon^2 \big],$ and summing over $t$ yields the claimed bound.

\noindent\textbf{Bounding the Tail Part.}
Assuming $\epsilon \leq 1/8$, we show the summed tail bound :
\begin{align*}
    \sum_{t=1}^T \lvert \langle  \bv_t, \bR_t \bH_t\rangle \rvert \leq 6\epsilon (T+1) M (h_\infty +b_\infty).
\end{align*}
Let $\bQ(k_t)$ denote the matrix with only the $k_t$-th sub-diagonal equal to $1$, and all other elements $0$. We claim that the tail part satisfies the entry-wise bound
\begin{align}\label{eq: tail-bound-inequality}
\lvert \bR_t \bH_t \rvert_j \leq \begin{cases}
    0 & \text{ if } 0\leq j<s_{t+1},\\
    \epsilon[\hat{D}_{t+1}]_j & \text{ if } s_{t+1}\leq j<\min\{M,s_{t+1} + k_t\},\\
    \epsilon\big([\hat{D}_{t+1}]_j - [\hat{D}_{t+1}]_{j-k_t}\big) & \text{ if } s_{t+1} + k_t\leq j\leq M-1.
\end{cases}
\end{align}

\begin{proof}[Proof of~\eqref{eq: tail-bound-inequality}]
\textbf{Case 1 ($0 \leq j < s_{t+1}$).} Since $\bR_t$ inherits the support structure of $\bG_t$ (non-zero entries only on rows and columns indexed by $\geq s_{t+1}$), row $j$ of $\bR_t$ is identically zero, giving $(\bR_t \bH_t)_j = 0$.

\textbf{Case 2 \& 3 ($j \geq s_{t+1}$).} Recall that $\bR_t$ contains the diagonal and the first $k_t-1$ sub-diagonals of $\bG_t$, with the $\ell$-th sub-diagonal entry equal to $F_t(\ell) - \hat{F}_t(\ell)$ for $\ell = 0, 1, \dots, k_t-1$. Each such entry is bounded by $\epsilon$ in absolute value via the uniform perturbation~\eqref{eq: uniform-convergence}. Together with $\bH_t \geq 0$, this gives
\begin{align*}
    \lvert \bR_t \bH_t \rvert_j \leq \epsilon \sum\nolimits_{i=\max(s_{t+1}, j-k_t+1)}^{j} \bH_t(i).
\end{align*}
Using the definition of $\bH_t$, the right-hand-side summation telescopes:
\begin{itemize}
    \item If $s_{t+1} \leq j < s_{t+1} + k_t$, the lower limit of the sum is $s_{t+1}$, giving
    $$\hat{D}_{t+1}(s_{t+1}) + \sum\nolimits_{i = s_{t+1}+1}^{j}\big(\hat{D}_{t+1}(i) - \hat{D}_{t+1}(i-1)\big) = \hat{D}_{t+1}(j),$$
    matching the second case.
    \item If $j \geq s_{t+1} + k_t$, the lower limit is $j - k_t + 1 > s_{t+1}$, and all terms are consecutive differences, giving
    $$\sum\nolimits_{i = j-k_t+1}^{j}\big(\hat{D}_{t+1}(i) - \hat{D}_{t+1}(i-1)\big) = \hat{D}_{t+1}(j) - \hat{D}_{t+1}(j-k_t),$$
    matching the third case.
\end{itemize}
This establishes~\eqref{eq: tail-bound-inequality}.
\end{proof}

Using $\lvert \bv_t\rvert \leq \bw_t + \bw_t^\star \leq 2\cdot\bm{1}$ entry-wise (cf.~\eqref{eq: bounds-for-w-D}), this gives \begin{align}\label{eq-appendix: tail-tmp1}
    \sum_{t=1}^T \lvert \langle \bv_t, \bR_t \bH_t\rangle \rvert & \leq 2\sum_{t=1}^T  \langle \bm{1}, \lvert \bR_t \bH_t \rvert \rangle \leq 2\epsilon \sum_{t=1}^T \langle \be(M-k_t,M), (\hat{\bD}_{t+1})_+ \rangle.
\end{align}

To bound the right-hand-side summation, the key idea is then to construct a non-negative \textit{dominating sequence} $\{\bg_t\}_{1\leq t\leq T}$ satisfying $(\hat{\bD}_{t})_+ \leq \bg_t$ coordinate-wise. We construct this sequence recursively:
\begin{align*}
   \bg_t = (h_\infty +b_\infty) \bm{1} + \underbrace{\big((\frac{1}{2}+ \epsilon)\bI + (\frac{1}{2}- \epsilon)\bQ(k_t) \big)}_{:= \bZ_{t}}\bg_{t+1},\quad \bg_{T+1} = 0
\end{align*}
with $k_{T+1} = 0$ for convenience.

The following claim holds for this sequence $\{\bg_t\}_{t=1}^T$:
\begin{lemma}\label{lem: dominating-sequence}
It holds that $\bg_t \geq (\hat{\bD}_{t})_+ \geq \bm{0}$ entry-wise for every $t \in [T+1]$.
\end{lemma}
\begin{proof}[Proof of Lemma~\ref{lem: dominating-sequence}]
    The result holds for $t = T+1$ directly by $\bg_{T+1} = \hat{\bD}_{T+1} = \bm{0}.$ Now suppose the result holds for $t+1$ with some $t\leq T,$ we now show it holds also at $t$ by induction: by~\eqref{eq: D-hat-forward-equation}  \begin{align*}
    \lvert \hat{\bD}_{t} \rvert &\leq (h_\infty + b_\infty)\bm{1} + \lvert \hat{\bA}_t(s_{t+1}) \hat{\bD}_{t+1}\rvert  \overset{\text{(a)}}{=} (h_\infty + b_\infty)\bm{1} +  \hat{\bA}_t(s_{t+1}) (\hat{\bD}_{t+1})_+\\
    & \overset{\text{(b)}}{\leq} (h_\infty + b_\infty)\bm{1} +  \big((\frac{1}{2}+ \epsilon)\bI + (\frac{1}{2}- \epsilon)\bQ(k_t) \big) (\hat{\bD}_{t+1})_+\\
    & {\leq} (h_\infty + b_\infty)\bm{1} +  \big((\frac{1}{2}+ \epsilon)\bI + (\frac{1}{2}- \epsilon)\bQ(k_t) \big) \bg_{t+1} = \bg_{t}.
    \end{align*}
Here the last line is by induction hypothesis; (a) uses $(\hat{\bD}_{t+1})_j \geq 0$ if and only if $j\geq s_{t+1}$; (b) uses, with the convention $[(\hat{\bD}_{t+1})_+]_{r}:=0$ for $r<0$,\begin{align*}
    \big[\hat{\bA}_t(s_{t+1}) (\hat{\bD}_{t+1})_+ \big]_j &= \sum_{\ell = s_{t+1}}^{j} \hat{\mu}_{j-\ell} \big[(\hat{\bD}_{t+1})_+\big]_\ell= \sum_{\ell = s_{t+1}}^{j-k_t} \hat{\mu}_{j-\ell} \big[(\hat{\bD}_{t+1})_+\big]_\ell + \sum_{\ell = j-k_t+1}^{j} \hat{\mu}_{j-\ell} \big[(\hat{\bD}_{t+1})_+\big]_\ell\\
    &\leq (1-\hat{F}_t(k_t-1)) \big[(\hat{\bD}_{t+1})_+\big]_{j-k_t} + \hat{F}_t(k_t-1) \big[(\hat{\bD}_{t+1})_+\big]_j\\
    &= \big[(\hat{\bD}_{t+1})_+\big]_{j-k_t} + \hat{F}_t(k_t-1) \underbrace{\big(\big[(\hat{\bD}_{t+1})_+\big]_j- \big[(\hat{\bD}_{t+1})_+\big]_{j-k_t} \big)}_{\geq 0}\\
    &\leq \big[(\hat{\bD}_{t+1})_+\big]_{j-k_t} + \big(\frac{1}{2}+\epsilon\big)\big(\big[(\hat{\bD}_{t+1})_+\big]_j- \big[(\hat{\bD}_{t+1})_+\big]_{j-k_t} \big)\\
    &=(\frac{1}{2}-\epsilon)\big[(\hat{\bD}_{t+1})_+\big]_{j-k_t}+\big(\frac{1}{2}+\epsilon\big)\big[(\hat{\bD}_{t+1})_+\big]_j
\end{align*}
where the second line is by $\hat{\bD}_{t+1}$ is non-decreasing, the last second line is by $ \hat{F}_t(k_{t}-1)\leq \frac{1}{2}+\epsilon$.
\end{proof}

Then we have
\begin{align*}
    & \epsilon \sum_{t=1}^T \langle \be(M-k_{t} ,M) , (\hat{\bD}_{t+1})_+\rangle \leq \epsilon \sum_{t=1}^T \langle \be(M-k_{t} ,M) , \bg_{t+1}\rangle\\
    &\overset{\text{(a)}}{=} \frac{2\epsilon}{1-2\epsilon} \sum_{t=1}^T \langle (\bI - \bZ_t)^\top \bm{1}, \bg_{t+1} \rangle \overset{\text{(b)}}{=} \frac{2\epsilon}{1-2\epsilon} \sum_{t=1}^T \big(\langle \bm{1}, \bg_{t+1}- \bg_{t}\rangle + (h_\infty+b_\infty) M  \big) \\
    &= \frac{2\epsilon}{1-2\epsilon} \big( \langle \bm{1}, \bg_{T+1} - \bg_1\rangle + TM(h_\infty + b_\infty) \big) \leq \frac{2\epsilon T M(h_\infty + b_\infty)}{1-2\epsilon} \leq \frac{8\epsilon TM(h_\infty + b_\infty)}{3},
\end{align*}
where the last inequality uses $\epsilon \leq 1/8$. Here (a) uses the identity $\bZ_t^\top \bm{1} = (\tfrac{1}{2}+\epsilon)\bm{1} + (\tfrac{1}{2}-\epsilon)(\bm{1}-\be(M-k_t,M)) = \bm{1} - (\tfrac{1}{2}-\epsilon)\be(M-k_t,M)$, obtained from the column sums of $\bQ(k_t)$, namely $\bQ(k_t)^\top \bm 1 = \be(0, M-k_t)$; rearranging gives $\be(M-k_t,M) = \tfrac{2}{1-2\epsilon}(\bI - \bZ_t)^\top \bm 1$. Step~(b) combines the adjoint relation $\langle (\bI - \bZ_t)^\top \bm 1, \bg_{t+1}\rangle = \langle \bm 1, (\bI - \bZ_t)\bg_{t+1}\rangle$ with the recursion $\bZ_t \bg_{t+1} = \bg_t - (h_\infty + b_\infty)\bm 1$. The penultimate inequality uses $\bg_{T+1} = \bm 0$ and $\bg_1 \geq \bm 0$, giving $\langle \bm 1, \bg_{T+1} - \bg_1\rangle \leq 0$.

Combining with~\eqref{eq-appendix: tail-tmp1}, we arrive at
$$\sum_{t=1}^T |\langle \bv_t, \bR_t\bH_t\rangle| \leq 2 \cdot \frac{8\epsilon TM(h_\infty+b_\infty)}{3} = \frac{16}{3}\epsilon T M(h_\infty + b_\infty) \leq 6\epsilon(T+1)M(h_\infty+b_\infty),$$ as claimed.

\subsection{Proof of Proposition~\ref{prop: over-shooting-price}}\label{appendix-sec-proof: prop: over-shooting-price}

For each $t\in [T],$ we first consider the following time index set  $\cT_{t}$ constructed via looking backward from $t:$\footnote{We define as supremum of an empty set as $0$ for convenience.} 
\begin{itemize}
    \item  Initialize $\cT_t = \{t\}$.
    \item \textbf{For $\tau = t-1,\dots, 1:$} If $s_{\tau}^\star > \max\{\sup_{l \in \cT_t} s_l^\star, s_{t}^\star \},$ add $\tau$ to $\cT_t.$ 
\end{itemize}
Then if we denote $\cT_t:= \{t_1<t_2<\dots<t_L = t\}$ (with $L\geq 1$ since $t\in\cT_t$), it holds that
\begin{align}\label{eq: backward-decomposition}
\PP(y^\star_{t} = j\lvert x_1 = 0) = \sum_{l=1}^L \PP\bigg(y^\star_{t} = j, y_{t_l}^{\star} = s_{t_l}^\star , y^\star_{t_\ell} > s_{t_\ell}^\star  ,\; \forall \ell > l \bigg\lvert x_1 = 0 \bigg),\quad \forall j > s_t^\star. 
\end{align}
\begin{proof}[Proof of~\eqref{eq: backward-decomposition}]
When $\cT_t = \{t\},$ we have $y_t^\star \leq s_t^\star$ a.s. by construction, and as a result $\PP(y_t^\star = j) = 0, \forall j > s_t^\star$, as desired.

When $\cT_t\neq\{t\},$ noticing that for any $\ell \in [L]$ and $\cA_\ell:=\{y_{t_\ell}^\star = s_{t_\ell}^\star \}$, we have by $\cA_\ell^c = \{y_{t_\ell}^\star > s_{t_\ell}^\star \}$, it holds for any $l\in [L]$ that
$\PP(y_t^\star = j,  \cap_{\ell > l}\cA_{\ell}^c)  = \PP(y_t^\star = j, \cA_l, \cap_{\ell > l}\cA_{\ell}^c)  + \PP(y_t^\star = j, \cap_{\ell \geq l}\cA_{\ell}^c)$.
Applying this rule recursively from $l = L$ leads to \begin{align*}
   \PP(y_t^\star = j)  = \PP(y_t^\star = j,  \cap_{\ell > L}\cA_{\ell}^c) =\sum_{l=1}^L \PP(y_t^\star = j, \cA_l, \cap_{\ell > l}\cA_{\ell}^c)  + \PP(y_t^\star=j, \cap_{\ell \geq 1} \cA_\ell^c).
\end{align*}
Finally, by definition of ${t_1},$ we have $\PP(\cA_1^c) = 0,$ we have the second term in the right-hand-side is $0$, thus the desired result holds.
\end{proof}

Based on the decomposition in~\eqref{eq: backward-decomposition}, we have  \begin{align*}
 & \sum_{j> s_{t}^\star}  \PP(y^\star_t = j)   \big[ C^\star_t(j) - C^\star_t(s_t^\star)  \big] = \sum_{l=1}^L\sum_{j> s_t^\star} \PP\bigg(y^\star_{t} = j, y_{t_l}^{\star} = s_{t_l}^\star , y^\star_{t_\ell} > s_{t_\ell}^\star  ,\; \forall \ell > l \bigg)  \bigg[ C^\star_t(j) - C^\star_t(s_t^\star)  \bigg]\\
 &= \sum_{l=1}^L \sum_{j> s_{t}^\star} \PP\bigg(y^\star_{t} \geq  j, y_{t_l}^{\star} = s_{t_l}^\star , y^\star_{t_\ell} > s_{t_\ell}^\star  ,\; \forall \ell > l \bigg)  \bigg[ C^\star_t(j) - C^\star_t(j-1)  \bigg]\\
 &= \sum_{l=1}^L \sum_{j= s_{t}^\star}^{M-1}  \PP\bigg(y^\star_{t} >  j, y_{t_l}^{\star} = s_{t_l}^\star , y^\star_{t_\ell} > s_{t_\ell}^\star  ,\; \forall \ell > l \bigg)   D_{t,j}^\star.
\end{align*}
Now we bound the above summation by induction from $l = L$ to $l = 1$, with the additional notation  
    $y^\star_{t \to t'} = \max_{ t\leq \tau \leq t'}\{ (s_\tau^\star - \sum_{r=\tau}^{t'-1}d_r )_+ \}$
denoting the post-ordering inventory level at $t'$ under $\pi^\star$ when starting at time $t$ with inventory level $s_t^\star$.
With this notation, we can rewrite the summation as
\begin{align*}
      \sum_{l=1}^L \sum_{j= s_{t}^\star}^{M-1}  \PP\bigg(y^\star_{t_l \to t} >  j, y_{t_l}^{\star} = s_{t_l}^\star , y^\star_{t_\ell} > s_{t_\ell}^\star  ,\; \forall \ell > l \bigg)   D_{t,j}^\star.
\end{align*}
For $l = L$, we have
\begin{align*}
  & \sum_{j \geq s_{t}^\star} \PP(y^\star_t > j, y_{t_L}^\star = s_{t_L}^\star) D^\star_{t j}
   = \sum_{j \geq s_{t}^\star} \PP(y^\star_{t_L \to t} > j, y_{t_L}^\star = s_{t_L}^\star) D^\star_{t j}  \\
  &=  \sum_{j \geq s_t^\star}  \big[\PP(y_{t_L\to t}^\star > j) - \PP(y^\star_{t_L \to t} > j, y_{t_L}^\star > s_{t_L}^\star)  \big] D_{tj}^\star.
\end{align*}
For $l < L$,  we introduce the following Lemma:
\begin{lemma}\label{lem: key-equation-induction} It holds that for every $1< l \leq L,$
\begin{equation}\label{eq: key-equation-induction}
    \begin{aligned}
    &\PP(y_{t_l\to t}^\star > j, y_{t_l}^\star = s^\star_{t_l} , y_{t_\ell}^\star > s_{t_\ell}^\star \; \forall \ell >l) - \PP( y_{ t_\ell }^\star > s_{t_\ell}^\star  \; \forall \ell > l, y^\star_{t_\ell \to t} > j\;  \exists \ell > l)\\
&\leq \PP(y_{t_l\to t}^\star > j, y_{t_\ell\to t}^\star \leq j\; \forall \ell > l  ) - \PP( y_{ t_\ell }^\star > s_{t_\ell}^\star  \; \forall \ell > l-1, y^\star_{t_\ell \to t} > j \;  \exists \ell > l-1)
\end{aligned}
\end{equation}
    \end{lemma}

By applying Lemma~\ref{lem: key-equation-induction} recursively for $1< l \leq L,$ we can get
\begin{align*}
         & \sum_{l=1}^L \sum_{j= s_{t}^\star}^{M-1}  \PP\bigg(y^\star_{t_l \to t} >  j, y_{t_l}^{\star} = s_{t_l}^\star , y^\star_{t_\ell} > s_{t_\ell}^\star  \; \forall \ell > l \bigg)   D_{tj}^\star\\
&\leq  \sum_{j \geq s_t^\star}  \bigg[\sum_{l = 1}^L \PP(y_{t_l\to t}^\star > j, y_{t_\ell\to t}^\star \leq j\; \forall \ell > l  ) - \PP( y_{ t_\ell }^\star > s_{t_\ell}^\star \; \forall \ell \geq 1, y^\star_{t_\ell \to t} > j   \; \exists \ell \geq 1)\bigg] D_{tj}^\star\\
& \leq \sum_{j \geq s_t^\star}  \sum_{l = 1}^L \PP(y_{t_l\to t}^\star > j, y_{t_\ell\to t}^\star \leq j\; \forall \ell > l  ) D_{tj}^\star,
\end{align*}
where the last line is by $D_{tj}^\star \geq 0,\; \forall j \geq s_{t}^\star.$ 

So it suffices to control $ \sum_{j \geq s_t^\star}  \sum_{l = 1}^L \PP(y_{t_l\to t}^\star > j, y_{t_\ell\to t}^\star \leq j\; \forall \ell > l  ) D_{tj}^\star$. 
For this purpose, noticing that for each $1\leq l < L$,
\begin{equation}\label{eq: appendix-tmp4}
\begin{aligned}
&W^\star_{t_l}(s_{t_l}^\star) \leq W^\star_{t_l}(s_{t_{l+1}}^\star) \implies \langle \be(s_{t_{l+1}}^\star,s^\star_{t_{l}}), \bD_{t_l}^\star\rangle \leq 0\\
&\implies \langle \be(s_{t_{l+1}}^\star,s^\star_{t_{l}}), \bA_{t_l: t-1} \bD_{t}^\star\rangle \leq -\langle  \be(s_{t_{l+1}}^\star,s^\star_{t_{l}}), \bS_{t_l:t-1}^\star \rangle.
\end{aligned}
\end{equation}

Then by Proposition~\ref{remark: A-interpretation},
\begin{equation}
    \begin{aligned}
  &\big[ \bA_{t_l:t-1}^\top \be (s_{t}^\star,s_{t_l}^\star)\big]_j =\sum_{r = s_{t}^\star}^{s_{t_l}^\star-1} \big[\bA_{t_l:t-1}\big]_{rj} = \PP(s_t^\star \leq Z_{t\to t_l}^j <s_{t_l}^\star, Z^j_{t\to \tau} \geq s_{\tau}^\star \quad \forall t_l < \tau \leq t) \\
  &= \begin{cases}
       \PP\bigg(s_{t}^\star\leq Z^j_{t \to t_l} <s_{t_l}^\star, Z_{t \to \tau}^j \geq s_\tau^\star \quad \forall t_l < \tau < t\bigg) & \text{ if } j\geq s_{t}^\star,\\
       0 & \text{otherwise.}
  \end{cases}
\end{aligned}
\end{equation}
For $j\geq s_t^\star$, the survival event in the preceding probability implies
\begin{align*}
Z^j_{t\to t_l}
=Z^j_{t\to t_{l+1}}+\sum\nolimits_{k=t_l}^{t_{l+1}-1}d_k
\geq Z^j_{t\to t_{l+1}}
\geq s_{t_{l+1}}^\star
\geq s_t^\star.
\end{align*}
As a result,
$\big[\bA_{t_l:t-1}^\top\be(s_t^\star,s_{t_l}^\star)\big]_j
=\big[\bA_{t_l:t-1}^\top\be(s_{t_{l+1}}^\star,s_{t_l}^\star)\big]_j
$.
Moreover, $Z^j_{t\to t_l}\geq j\geq s_t^\star$, while the survival condition at $\tau=t$ is exactly $j\geq s_t^\star$. Therefore, for any $j\geq s_t^\star$,
\begin{align*}
    &\PP\bigg(s_{t}^\star\leq Z^j_{t \to t_l} <s_{t_l}^\star, Z_{t \to \tau}^j \geq s_\tau^\star \quad \forall t_l < \tau < t\bigg) = \PP\bigg( Z^j_{t \to t_l} <s_{t_l}^\star, Z_{t \to \tau}^j \geq s_\tau^\star \quad \forall t_l < \tau \leq t\bigg)\\
    &= \PP\bigg( j + \sum_{k = t_l}^{t-1} d_k <s_{t_l}^\star, j + \sum_{k = \tau}^{t-1} d_k \geq s_\tau^\star \quad \forall t_l < \tau \leq t\bigg)\\
    &\overset{\text{(a)}}{=} \PP\bigg( \max_{t_l\leq \tau \leq t}\big(s_{\tau}^\star- \sum_{k = \tau}^{t-1} d_k) > j,  j + \sum_{k = \tau}^{t-1} d_k \geq s_\tau^\star \quad \forall t_l < \tau \leq t\bigg)\\
    &\overset{\text{(b)}}{=} \PP\bigg( y^\star_{t_l \to t}>j ,  j + \sum_{k = \tau}^{t-1} d_k \geq s_\tau^\star \quad \forall t_l < \tau \leq t\bigg)\\
&\overset{\text{(c)}}{=} \PP\bigg( y^\star_{t_l \to t}>j ,  y_{\tau \to t}^\star \leq j \quad \forall t_l < \tau \leq t\bigg)\overset{\text{(d)}}{=} \PP\bigg( y^\star_{t_l \to t}>j ,  y_{t_\ell \to t}^\star \leq j \quad \forall  \ell > l \bigg).
\end{align*}
Where in (a), given $j + \sum_{k = \tau}^{t-1} d_k \geq s_\tau^\star,\; \forall t_l < \tau \leq t$, it holds that
$    j + \sum_{k = t_l}^{t-1} d_k <s_{t_l}^\star \iff  \max_{t_l\leq \tau \leq t}\big(s_{\tau}^\star- \sum_{k = \tau}^{t-1} d_k) > j$.
(b) and (c) are by Lemma~\ref{lem: J-closed-form}, (d) is by \begin{align*}
    s^\star_\tau \leq s_{t_\ell}^\star, \forall t_{\ell - 1}< \tau \leq t_\ell    \implies y_{\tau \to t}^\star \leq y_{t_\ell \to t}^\star,  \forall  t_{\ell - 1}< \tau \leq t_\ell. 
\end{align*}
Thus~\eqref{eq: appendix-tmp4} gives  $\sum_{j\geq s_t^\star} \PP\bigg( y^\star_{t_l \to t}>j ,  y_{t_\ell \to t}^\star \leq j \;  \forall  \ell > l \bigg) D_{tj}^\star \leq -\langle  \be(s_{t_{l+1}}^\star,s^\star_{t_{l}}), \bS_{t_l:t-1}^\star \rangle$.
Since $t_L=t$, the $l=L$ term vanishes because $y_{t_L\to t}^\star=y_{t\to t}^\star=s_t^\star\leq j$. Taking summation over $1\leq l\leq L -1$, we arrive at
\begin{align*}
& \sum_{l=1}^{L-1} \sum_{j= s_{t}^\star}^{M-1}  \PP\bigg(y^\star_{t_l \to t} >  j, y_{t_l}^{\star} = s_{t_l}^\star , y^\star_{t_\ell} > s_{t_\ell}^\star  \; \forall \ell > l \bigg)   D_{tj}^\star\\
& \leq  \sum_{l = 1}^{L-1}  \sum_{j \geq s_t^\star} \PP(y_{t_l\to t}^\star > j, y_{t_\ell\to t}^\star \leq j\; \forall \ell > l  ) D_{tj}^\star\leq -\sum_{l = 1}^{L-1} \langle \be(s_{t_{l+1}}^\star, s_{t_l}^\star) , \bS^\star_{t_l:t-1} \rangle \\
&\leq \sum_{l = 1}^{L-1} \langle \be(s_{t_{l+1}}^\star, s_{t_l}^\star) , \sum_{k = t_l}^{t-1}  b_k \bA_{t_l:k-1} (\bm{1}-\bF_k) \rangle 
\end{align*}
Finally, by $\bF_k = \bA_k(0) \bm{1}$ and $\bA_{j:k-1} \leq \tilde{\bA}_{j:k-1}:= \prod_{i=j}^{k-1} \bA_i(0)$ entry-wisely,   \begin{align*}
    &\sum_{l = 1}^{L-1} \langle \be(s_{t_{l+1}}^\star, s_{t_l}^\star) , \sum_{k = t_l}^{t-1}  b_k \bA_{t_l:k-1} (\bm{1}-\bF_k) \rangle \leq b_\infty \sum_{l=1}^{L-1}\sum_{k = t_l}^{t-1} \langle \be(s_{t_{l+1}}^\star, s_{t_l}^\star), \tilde{\bA}_{t_l:k-1}(\bm{1}-\bF_k) \rangle\\
    &= b_\infty \sum_{l=1}^{L-1}\sum_{k = t_l}^{t-1} \langle \be(s_{t_{l+1}}^\star, s_{t_l}^\star), ( \tilde{\bA}_{t_l:k-1} -\tilde{\bA}_{t_l:k})\bm{1} \rangle  \leq b_\infty \sum_{l = 1}^{L-1} \lVert \be(s_{t_{l+1}}^\star,s_{t_l}^\star)\rVert_1 \leq b_\infty M,
\end{align*}
we get the desired $\cO(b_\infty M)$ bound.

\subsection{Proof of Lemma~\ref{lem: key-equation-induction}}

First noticing that by 
$    y_{t_l \to t}^\star > j, y_{t_\ell \to t}^\star \leq j  \; \forall \ell > l \implies  y_{t_l \to t}^\star > j, y_{t_\ell}^\star > s_{t_\ell}^\star  \; \forall \ell > l$, we have \begin{align*}
    & \PP( y_{t_l \to t}^\star > j, y_{t_l}^\star = s_{t_l}^\star, y_{t_\ell}^\star > s_{t_\ell}^\star  \; \forall \ell > l)   - \PP(y_{t_l \to t}^\star > j, y_{t_\ell \to t}^\star \leq j  \; \forall \ell > l) \\
    & = \PP( y_{t_l \to t}^\star > j, y_{t_l}^\star = s_{t_l}^\star, y_{t_\ell}^\star > s_{t_\ell}^\star  \; \forall \ell > l, y_{t_\ell \to t}^\star > j  \exists \ell > l)- \PP(y_{t_l \to t}^\star > j, y_{t_l}^\star > s_{t_l}^\star, y_{t_\ell \to t}^\star \leq j  \; \forall \ell > l)
\end{align*}
thus~\eqref{eq: key-equation-induction} is equivalent to 
\begin{align*}
&\PP( y_{t_l \to t}^\star > j, y_{t_l}^\star = s_{t_l}^\star, y_{t_\ell}^\star > s_{t_\ell}^\star  \; \forall \ell > l, y_{t_\ell \to t}^\star > j  \exists \ell > l)- \PP(y_{t_l \to t}^\star > j, y_{t_l}^\star > s_{t_l}^\star, y_{t_\ell \to t}^\star \leq j  \; \forall \ell > l)\\
&\leq  \PP( y_{ t_\ell }^\star > s_{t_\ell}^\star  \; \forall \ell > l, y^\star_{t_\ell \to t} > j   \exists \ell > l)- \PP( y_{ t_\ell }^\star > s_{t_\ell}^\star  \; \forall \ell > l-1, y^\star_{t_\ell \to t} > j   \exists \ell > l-1).
\end{align*}
On the other hand, by 
\begin{align*}
    &\PP( y_{t_l \to t}^\star > j, y_{t_l}^\star = s_{t_l}^\star, y_{t_\ell}^\star > s_{t_\ell}^\star  \; \forall \ell > l, y_{t_\ell \to t}^\star > j  \exists \ell > l) - \PP( y_{ t_\ell }^\star > s_{t_\ell}^\star  \; \forall \ell > l, y^\star_{t_\ell \to t} > j   \exists \ell > l)\\
    &=- \PP(\{ y_{t_l \to t}^\star \leq j \text{ or } y_{t_l}^\star >s_{t_l}^\star \}, y_{t_\ell}^\star > s_{t_\ell}^\star  \; \forall \ell > l, y_{t_\ell \to t}^\star > j  \exists \ell > l).
\end{align*}
It suffices to show that \begin{align*}
   &\PP( y_{ t_\ell }^\star > s_{t_\ell}^\star  \; \forall \ell > l-1, y^\star_{t_\ell \to t} > j   \exists \ell > l-1) \\
   &\leq  \PP(\underbrace{\{ y_{t_l \to t}^\star \leq j \text{ or } y_{t_l}^\star >s_{t_l}^\star \}, y_{t_\ell}^\star > s_{t_\ell}^\star  \; \forall \ell > l, y_{t_\ell \to t}^\star > j  \exists \ell > l}_{:= A_1}) + \PP(\underbrace{y_{t_l \to t}^\star > j, y_{t_l}^\star > s_{t_l}^\star, y_{t_\ell \to t}^\star \leq j  \; \forall \ell > l}_{:= A_2}).
\end{align*}
To see this, just noticing that 
\begin{align*}
    &y_{ t_\ell }^\star > s_{t_\ell}^\star  \; \forall \ell > l-1, y^\star_{t_\ell \to t} > j   \exists \ell > l-1\\
    &\implies \underbrace{\{ y_{ t_\ell }^\star > s_{t_\ell}^\star  \; \forall \ell > l-1, y_{t_l \to t}^\star \leq j, y^\star_{t_\ell \to t} > j   \exists \ell > l \}}_{:= B_1} \text{ or }  \underbrace{\{ y_{ t_\ell }^\star > s_{t_\ell}^\star  \; \forall \ell > l-1, y_{t_l \to t}^\star > j \}}_{:= B_2}.
\end{align*}
Then by $B_1 \subset A_1, B_2 \subset A_1 \cup A_2$, we have 
\begin{align*}
    &\PP( y_{ t_\ell }^\star > s_{t_\ell}^\star  \; \forall \ell > l-1, y^\star_{t_\ell \to t} > j   \; \exists \ell > l-1) = \PP(B_1 \cup B_2) \leq \PP(A_1 \cup A_2) \leq \PP(A_1)+\PP(A_2),
\end{align*}
as desired.

\section{Proof of Results in Section~\ref{sec: censored-demand}}

In this section, we provide the detailed proof of results in Section~\ref{sec: censored-demand}. For this purpose, we first provide several technical facts for later reference.

The first result is an equivalence relation between $\tilde{\cC}_{j}$ and $\cC_{j}$, which ensures using the empirical radius $\tilde{\cC}_{j}$ to build upper/lower confidence estimators is a valid choice.
\begin{lemma}\label{lem: C-hat-to-C-gap}
    Under the event~\eqref{eq: partial-CDF-confidence-bound-general} and suppose $n_j \geq 1$, define
     $\hat\cC_{j} :=  4 \sqrt{\frac{\tilde F(j)\big(1 - \tilde F(j)\big) \log(Mn/\delta)}{n_{j}}} + \dfrac{4 \log(Mn/\delta)}{n_{j}},$ it holds that
    $\lvert \hat{\cC}_{j} - \cC_{j}\rvert \leq 4\cC_{j} \leq 8\hat{\cC}_{j} + \frac{64\log(Mn/\delta)}{n_{j}} ,\quad \forall j \in [M]_+$.
\end{lemma}
In particular, noticing that
$\tilde\cC_{j} := 2\hat\cC_{j} + \dfrac{16\log(Mn/\delta)}{n_{j}}$,
Lemma~\ref{lem: C-hat-to-C-gap} implies \begin{equation}\label{eq: C-hat-to-C-gap-corollary}
     \hat{\cC}_{j}  \leq 5\cC_{j},\quad \cC_{j} \leq 40 \hat{\cC}_{j}, \quad \cC_{j} \leq \tilde{\cC}_{j} \leq 14\cC_{j},
\end{equation}

Now we establish the second result to control the bias incurred by the sequential UCB construction process when selecting an earlier $j_0 < j.$
\begin{lemma}\label{lem: C-self-bounding}
For $0\le j_0\le j\le M-1$. Suppose
$F(j)-F(j_0)\le A\cC_{j_0}$
for some  constant $A>0$. Then
$\cC_{j_0}\le (2+4A)\cC_{j}$. 
\end{lemma}

Finally, we state the following deterministic inequality for the later usage.
\begin{lemma}\label{lem: ucb-variance-transfer}
Let $0\le p\le u\le 1$ and $\xi\ge 0$, and define $d:=u-p$.
For any constants $a,b>0$, the following two statements hold:
\begin{align}
d
&\le
a\sqrt{p(1-p)\xi}+b\xi
\implies
d
\le
2a\sqrt{u(1-u)\xi}+(a^2+2b)\xi,
\label{eq: p-to-u-variance-transfer}
\\
d
&\le
a\sqrt{u(1-u)\xi}+b\xi
\implies
d
\le
2a\sqrt{p(1-p)\xi}+(a^2+2b)\xi.
\label{eq: u-to-p-variance-transfer}
\end{align}
Consequently,
$
u-p\lesssim \sqrt{p(1-p)\xi}+\xi
\iff 
u-p\lesssim \sqrt{u(1-u)\xi}+\xi.
$
\end{lemma}

\subsection{Proof of Proposition~\ref{prop: biased-F}}

\begin{proof}[Proof of Proposition~\ref{prop: biased-F}]

When the maximum $n_j \vee 1$ in both inequalities are taken at $1$, the results hold directly by $0\leq F^\mathrm{UCB}(j) \leq 1, 0\leq F^\mathrm{LCB}(j) \leq 1,$ so in the followed proof, we focus on the case $n_j > 1.$

We first prove the first inequality. Under the event~\eqref{eq: partial-CDF-confidence-bound-general}, ${F}^\mathrm{UCB}(j) \geq F(j)$ holds by induction on $j$: assuming $F^\mathrm{UCB}(j-1)\geq F(j-1)$, we have $\tilde F(j) + \tilde{\cC}_{j} \geq \tilde F(j) + \cC_{j} \geq F(j)$, where the first inequality uses Lemma~\ref{lem: C-hat-to-C-gap}'s corollary~\eqref{eq: C-hat-to-C-gap-corollary} ($\cC_{j}\leq\tilde{\cC}_{j}$) and the second uses~\eqref{eq: partial-CDF-confidence-bound-general}; hence $F^\mathrm{UCB}(j) = \min\{1, \max\{F^\mathrm{UCB}(j-1), \tilde F(j) + \tilde{\cC}_{j}\}\}\geq F(j)$.

For the upper bound, noticing the monotone nature $n_{1}\geq n_{2} \geq \dots \geq n_{M-1},$ if the maximum in~\eqref{eq: def-UCB} is attained at some $j_0 \leq j,$ then \begin{align*}
   &{F}^\mathrm{UCB}(j) \leq \tilde{F}(j_0) + \tilde{\cC}_{j_0} \overset{\text{(a)}}{\leq} F(j_0)+ 15\cC_{j_0}  \\
&\overset{\text{(b)}}{\leq} F(j) +  60 \sqrt{\frac{{F}(j) (1-{F}(j_0)) \log(Mn/\delta)}{n_{j}}}+ \frac{60 \log(Mn/\delta)}{n_{j}}\\
&\overset{\text{(c)}}{\leq} F(j) + c_0' \cC_{j},
\end{align*}
for some absolute constant $c_0'$. 
Here:
(a) follows from~\eqref{eq: partial-CDF-confidence-bound-general} together with~\eqref{eq: C-hat-to-C-gap-corollary}, giving $$\tilde F(j_0) + \tilde{\cC}_{j_0} \leq F(j_0) + \cC_{j_0} + 14\cC_{j_0} = F(j_0) + 15\cC_{j_0}.$$
(b) uses $F(j_0)\leq F(j)$ for the first factor and $n_{j_0}\geq n_{j}$ for the denominator.
To see (c), noticing that
\begin{align*}
  F(j_0) + 15\cC_{j_0}  \geq \tilde{F}(j_0) + \tilde{\cC}_{j_0} \geq F^\mathrm{UCB}(j) \geq F(j) \implies F(j) - F(j_0) \leq 15\cC_{j_0}.
\end{align*}
Then by Lemma~\ref{lem: C-self-bounding} with $A = 15$, we have $\cC_{j_0} \leq \bar{c}_0\cC_{j}$
for some absolute constant $\bar{c}_0.$
Therefore, $$F(j)(1-F(j_0)) \leq F(j)(1-F(j)) + \bar{c}_0\cC_{j}.$$
Finally, by the elementary inequalities $\sqrt{x+y}\leq\sqrt{x}+\sqrt{y},\sqrt{xy}\leq(x+y)/2, \forall x,y \geq 0$,
\begin{align*}
    \sqrt{\frac{F(j)(1-F(j_0))\log(Mn/\delta)}{n_{j}}} &\leq \sqrt{\frac{F(j)(1-F(j))\log(Mn/\delta)}{n_{j}}} + \frac{\bar{c}_0\cC_{j}}{2} + \frac{\log(Mn/\delta)}{2n_{j}}\leq (\bar{c}_0 + 1)\cC_{j}.
\end{align*}
Selecting $c_0' = 60(\bar{c}_0 +1)$ then concludes the proof for the first inequality.

For the second inequality, $F^\mathrm{LCB}(j) \leq F(j)$ holds by induction on $j$ in parallel to the UCB direction: $\tilde F(j) - \tilde{\cC}_{j} \leq \tilde F(j) - \cC_{j} \leq F(j)$, where the first inequality uses $\cC_{j}\leq\tilde{\cC}_{j}$ and the second uses~\eqref{eq: partial-CDF-confidence-bound-general}; hence $F^\mathrm{LCB}(j) = \max\{F^\mathrm{LCB}(j-1), \tilde F(j) - \tilde{\cC}_{j}\} \leq F(j)$.

For the lower bound, we have
\begin{align*}
F^\mathrm{LCB}(j) \geq \tilde{F}(j) - \tilde{\cC}_{j}
\overset{\text{(d)}}{\geq} F(j) - \cC_{j} - \tilde{\cC}_{j}
\overset{\text{(e)}}{\geq} F(j) - \cC_{j} - 14\cC_{j} = F(j) - 15\cC_{j},
\end{align*}
where (d) follows from~\eqref{eq: partial-CDF-confidence-bound-general}, and (e) uses $\tilde{\cC}_{j} \leq 14\cC_{j}$ from~\eqref{eq: C-hat-to-C-gap-corollary}. Thus selecting $c_0 = \max\{c_0',15\} $ finishes the proof.
\end{proof}

\subsection{Proof of Lemma~\ref{lem: policy-ordering-UCB}}

Recall that the base-stock levels obey
$$s_t = \min\{y: D_t^\mathrm{UCB}(y)\ge 0\} ,\quad s^\star_t = \min\{y: D^\star_t(y)\ge 0\}, \forall t\in [T],$$ it suffices to show that $\Delta \bD_{t}:= \bD^\star_{t}-\bD_{t}^\mathrm{UCB} \leq \bm{0} , \forall t\in [T].$

We now use the induction to show this fact.
First, when $t = T$, the result holds directly by
$$\Delta \bD_T = (h_T+b_T)(\bF_T-\bF_T^\mathrm{UCB}) \leq \bm{0}.$$
Now suppose for any $k>t$ we have $\Delta \bD_{k} \leq \bm{0}$. Applying~\eqref{eq: recursion-D-diff} with $\hat{\bF} = \bF^\mathrm{UCB}$ and $\hat{\bF}_t = \bF_t^{\mathrm{UCB}}$ gives \begin{align*}
    \Delta \bD_{t} = \bA_{t}(s^\star_{t+1})\Delta \bD_{t+1} + [\bA_t(s_{t+1}^\star) - \bA^\mathrm{UCB}_t(s_{t+1})]\bD_{t+1}^\mathrm{UCB} + (h_t+b_t)\Delta \bF_t,
\end{align*}
each of the three terms is non-positive: (i) $\bA_t(s_{t+1}^\star)\Delta \bD_{t+1} \leq \bm{0}$ by the induction hypothesis and $\bA_t(s_{t+1}^\star) \geq \bm{0}$ entry-wise; (ii) $(h_t+b_t)\Delta \bF_t \leq \bm{0}$ since $\Delta \bF_t = \bF_t - \bF_t^\mathrm{UCB} \leq \bm{0}$; (iii) for the middle term, \begin{align*}
    [\bA_t(s_{t+1}^\star) - \bA^\mathrm{UCB}_t(s_{t+1})]\bD_{t+1}^\mathrm{UCB} = \underbrace{[\bA_t(s_{t+1}^\star) - \bA_t(s_{t+1})]\bD_{t+1}^\mathrm{UCB}}_{\leq \bm{0}} + \underbrace{\Delta \bA_t(s_{t+1})\bD_{t+1}^\mathrm{UCB} = \bG_t \bH_t^\mathrm{UCB}}_{\leq \bm{0}},
\end{align*}
where,
\begin{enumerate}
    \item To see the sign of the first term: by the induction hypothesis, $s_{t+1} \leq s_{t+1}^\star,$ then for $j \in [M-1]_+$,
    $$ \big[[\bA_t(s_{t+1}^\star) - \bA_t(s_{t+1})]\bD_{t+1}^\mathrm{UCB}\big]_j = -\sum_{k = s_{t+1}}^{\min\{ s_{t+1}^\star-1, j\} } \mu_{t,j-k} [\bD^\mathrm{UCB}_{t+1}]_k \leq 0$$
    by  $[\bD_{t+1}^{\mathrm{UCB}}]_k \geq 0$ for $k \geq s_{t+1},$ this gives the non-positivity of the first term.
    \item To see the sign of the second term: Applying the summation-by-parts argument in~\eqref{eq: AD-to-GH}, we obtain \begin{align*}
        \Delta \bA_t(s_{t+1})\bD_{t+1}^\mathrm{UCB} = \bG_t\bH_t^\mathrm{UCB}  
    \end{align*} with $\bG_t \leq \bm{0}$ (as $\bF_t \leq \bF_t^\mathrm{UCB}$) and $\bH_t^\mathrm{UCB} \geq \bm{0}$ (as $D_{t+1}^{\mathrm{UCB}}(s_{t+1})\geq 0$ and  $\bD_{t+1}^{\mathrm{UCB}}$ is increasing) entry-wisely. 
\end{enumerate}
This then finishes the induction.

\subsection{Proof of Theorem~\ref{thm: censored-sample-complexity} and Theorem~\ref{thm: censored-sample-complexity-identical}}\label{appendix: censored-demand-proofs}

In this section, we provide the proofs of Theorem~\ref{thm: censored-sample-complexity} and Theorem~\ref{thm: censored-sample-complexity-identical}. The two proofs nearly follow the same arguments; the only differences arise from the distinct confidence bounds and demand estimators used by the two algorithms, which depend on the product-level count $N_{t,j}$ or its aggregated counterpart $N_{\mathsf{agg},j} = \sum_{t=1}^T N_{t,j}$.
In the following, we focus on the proof of Theorem~\ref{thm: censored-sample-complexity}, and substituting $N_{t,j}$ by $N_{\mathsf{agg},j}$, with the count upper bound $N$ replaced by $NT$ accordingly, gives the proof of Theorem~\ref{thm: censored-sample-complexity-identical}.

Throughout this proof, quantities carrying the superscript $\mathrm{UCB}$
are computed from the upper-biased CDFs
$\{\bF_t^{\mathrm{UCB}}\}_{t=1}^T$. In particular,
$\bD_t^{\mathrm{UCB}}$, $\bA_t^{\mathrm{UCB}}(\cdot)$, and
$\bH_t^{\mathrm{UCB}}$ are defined as the analogues of
$\hat{\bD}_t$, $\hat{\bA}_t(\cdot)$, and $\bH_t$ in
Appendix~\ref{appendix-sec: proof-of-uniform-convergence-section},
with $\hat{\bF}_t$ replaced by $\bF_t^{\mathrm{UCB}}$.
We keep $s_t$ to denote the UCB base-stock policy and $y_t$ to denote
the trajectory induced by this policy. Throughout this proof, the plain $\bD_t \equiv \bD_t^\star$ always denotes the derivative
under the \emph{true} CDF sequence $\{\bF_t\}_{t=1}^T$ and the
\emph{optimal} base-stock levels $\{s_{t+1}^\star\}$ (cf.~\eqref{eq: D-forward-equation}),
to be distinguished from its optimistic counterpart $\bD_t^{\mathrm{UCB}}$,
which is built from $\{\bF_t^{\mathrm{UCB}}\}_{t=1}^T$ and the UCB levels $\{s_{t+1}\}$.

For the given $\delta > 0,$ we also assume the following concentration bound holds throughout the proof:
\begin{align}\label{eq: UCB-version-bernstein}
&0 \leq  F_t^\mathrm{UCB}(j)-F_t(j)  \lesssim \sqrt{\frac{F_t^\mathrm{UCB}(j)(1-F_t^\mathrm{UCB}(j)) \log(MTN/\delta) }{N_{t,j}\vee 1}} + \frac{\log(MTN/\delta)}{N_{t,j}\vee 1}.
\end{align}
Actually, \eqref{eq: UCB-version-bernstein} holds with probability at least $1-\delta$ uniformly for all $t\in [T],j \in [M]_+$ via combining the deterministic fact in Lemma~\ref{lem: ucb-variance-transfer} with $
    u = F_t^\mathrm{UCB}(j), p = F_t(j)$ and~\eqref{eq: partial-CDF-confidence-bound-UCB}. This is the source of high-probability statement in both Theorems.

\begin{proof}[Proof of Theorem~\ref{thm: censored-sample-complexity}]
    Applying the cost decomposition in Theorem~\ref{thm: cost-decomposition} under $\{\bF_t^\mathrm{UCB}\}_{t=1}^T$ gives
\begin{align*}
    \max_{x\in [M]_+}\Delta(x;\pi) \leq \underbrace{\sum_{t = 1}^T \big\langle \bv_{t}, \big[\bA_t(s_{t+1}) - {\bA}^\mathrm{UCB}_t(s_{t+1})  \big]\bD^\mathrm{UCB}_{t+1} + \Delta \bc_t\big\rangle}_{:= \mathsf{Err}_T}.
\end{align*}
We may assume $\cN^\star<\infty$, since otherwise $\PP(y_t^\star>j\lvert x_1=0)=0$ for all $t,j$ (each count $N_{t,j}\vee1$ being finite), so $\bv_t=\bm0$ and the right-hand side above vanishes; in particular $a_\star>0$, as required by Lemma~\ref{lem: peeling-bound}.
By~\eqref{eq: UCB-version-bernstein}, we have with probability at least $1-\delta,$
\begin{align*}
   -c_0(h_\infty + b_\infty) \sqrt{\frac{\log(MTN/\delta)}{N_{t,j}\vee 1}}  \leq  \Delta{c}_{t,j} \leq 0,
\end{align*}
for some absolute constant $c_0$. Since $s_t \leq s_t^\star$ for all $t\in [T],$
\begin{align}\label{eq: ucb-analysis-v-bound}
  -\PP(y_t^\star > j\lvert x_1 = 0) \leq v_{t,j} = \PP(y_t^\star \leq j\lvert x_1 = 0) - \PP(y_t \leq j\lvert x_1 = 0) \leq 0.
\end{align}
As a consequence, it holds that
\begin{align}\label{eq: ucb-analysis-tmp1}
   0 \leq \langle \bv_t, \Delta\bc_t \rangle &\lesssim (h_\infty + b_\infty)  \sum_{j=0}^{M-1} \PP(y_t^\star > j\lvert x_1 = 0)\sqrt{\frac{\log(MTN/\delta)}{N_{t,j}\vee 1}}
\end{align}
Summing over $t$ and applying the Cauchy--Schwarz inequality, with $\sum_{t,j}\PP(y_t^\star>j\lvert x_1=0)\le MT$ , gives $$\sum_{t=1}^T\langle \bv_t,\Delta\bc_t\rangle\lesssim (h_\infty+b_\infty)MT\sqrt{\log(MTN/\delta)/\cN^\star}.$$

It remains to control $\sum_{t=1}^T \langle \bv_t, [\bA_t(s_{t+1}) - \bA^\mathrm{UCB}_t(s_{t+1})]\bD^\mathrm{UCB}_{t+1}\rangle$. Revisiting the proof of Theorem~\ref{thm: cost-decomposition}, this can be equivalently written as $\sum_{t=1}^T \sum_{j=0}^{M-1} v_{tj} \big[ \bG_{t} \bH^\mathrm{UCB}_t\big]_j,$
and it holds by Proposition~\ref{prop: biased-F} and~\eqref{eq: ucb-analysis-v-bound} that
\begin{equation}\label{eq: ucb-analysis-tmp2}
    \bG_{t}  \leq \bm 0 \implies \langle \bv_t, \big[\bA_t(s_{t+1}) - \bA^\mathrm{UCB}_t(s_{t+1}) \big]\bD^\mathrm{UCB}_{t+1}\rangle = \langle \bv_t, \bG_t \bH^\mathrm{UCB}_t\rangle \geq 0.
\end{equation}
\allowdisplaybreaks
Now for each $t$, denoting $\eta_{t,j}:=  \sqrt{\frac{\log(MTN/\delta)}{N_{t,j}\vee 1}},
    \zeta_{t,j}:=  \frac{\log(MTN/\delta)}{N_{t,j}\vee 1},$
we can use the error bound in~\eqref{eq: UCB-version-bernstein} to obtain
\begin{align}
    &\sum_{t=1}^T \sum_{j = 0}^{M-1} v_{tj} \big[\bG_{t}\bH^\mathrm{UCB}_t \big]_j= -\sum_{t=1}^T \sum_{j = 0}^{M-1} \sum_{m = s_{t+1}}^j v_{tj} \big( {F}^{\mathrm{UCB}}_t(j-m) - F_t(j-m)  \big) [\bH^\mathrm{UCB}_t]_m \\
    &\overset{\mathrm{(a)}}\lesssim -\sum_{t=1}^T \sum_{j = 0}^{M-1} \sum_{m = s_{t+1}}^jv_{tj} \bigg(\sqrt{{F}^\mathrm{UCB}_t(j-m)(1-{F}^\mathrm{UCB}_t(j-m) )} \eta_{t,j}  + \zeta_{t,j}\bigg) [\bH^\mathrm{UCB}_t]_m  \label{eq: ucb-analysis-tmp3}
\end{align}
where in (a), we have used $N_{t,j-m}\geq N_{t,j}$ and the error bound~\eqref{eq: UCB-version-bernstein}.

To control the variance term, we introduce the following elementary fact through the peeling technique \citep{bartlett2005local}
\begin{lemma}\label{lem: peeling-bound}.
Given non-negative sequences $\{p_m\},\{h_m\}$ over a index set $\cM\subset \mathbb{N}$ satisfying the monotonicity condition $1\geq p_m \geq p_{m'}\geq 0$ for all $m<m'$, it holds for any $0<a<1/4$ that
\begin{align*}
  \sum_{m\in \cM}  \sqrt{p_{m}(1-p_m)} h_{m} \lesssim \sqrt{\log(e/a)}  \sqrt{\sum_{m\in \cM} h_{m}^2 p_{m}(1-p_m) + 2 \sum_{m< m'} h_{m}h_{m'} p_{m'}(1-p_m)} + \sqrt{a}\sum_{m\in \cM} h_m.
\end{align*}    
\end{lemma}

For each $t,j$ denoting \begin{align*}
    \cV_{t,j}:= &\sum_{s_{t+1}\leq m\leq j} [\bH_t^{\mathrm{UCB}}]_{m}^2 F_{t}^{\mathrm{UCB}}(j-m)\big(1-F_{t}^{\mathrm{UCB}}(j-m)\big) \\
    &+ 2 \sum_{s_{t+1}\leq m< m'\leq j} [\bH_t^{\mathrm{UCB}}]_{m}[\bH_t^{\mathrm{UCB}}]_{m'} F_{t}^{\mathrm{UCB}}(j-m')\big(1-F_{t}^{\mathrm{UCB}}(j-m)\big), 
\end{align*}
it holds that by selecting $a = a_\star := \frac{1}{8}\min\{1,\,\log(MTN/\delta)/(\cN^\star \wedge N)\}$ in Lemma~\ref{lem: peeling-bound},
\begin{align*}
    \sum_{m = s_{t+1}}^j \sqrt{{F}^\mathrm{UCB}_t(j-m)(1-{F}^\mathrm{UCB}_t(j-m) )}   [\bH^\mathrm{UCB}_t]_m\lesssim \sqrt{\log(e/a_\star)\, \cV_{t,j}} + \sqrt{a_\star}\,[(\bD_{t+1}^\mathrm{UCB})_+]_j.
\end{align*}
Together with~\eqref{eq: ucb-analysis-tmp3}, using $-v_{tj}\ge0$ and $\sum_{m=s_{t+1}}^j[\bH^\mathrm{UCB}_t]_m = [(\bD^\mathrm{UCB}_{t+1})_+]_j$, we get \begin{align*}
      &\sum_{t=1}^T \sum_{j = 0}^{M-1} v_{tj} \big[\bG_{t}\bH^\mathrm{UCB}_t \big]_j \lesssim -\sum_{t=1}^{T}\sum_{j=0}^{M-1} v_{tj}\eta_{t,j}\sqrt{\log(e/a_\star) \cV_{t,j}} - \sum_{t=1}^{T}\sum_{j= 0}^{M-1} v_{tj}\zeta_{t,j} [(\bD_{t+1}^{\mathrm{UCB}})_+]_j\\
      &\quad + \underbrace{\sqrt{a_\star}\sum_{t=1}^T\sum_{j=0}^{M-1}(-v_{tj})\eta_{t,j}[(\bD_{t+1}^\mathrm{UCB})_+]_j}_{:=\cR_{\mathrm{peel}}}\\
      &\lesssim \sqrt{\frac{MT\log(MTN/\delta)\log(e/a_\star)}{\cN^\star}}\,\sqrt{-\sum\nolimits_{t=1}^{T}\sum\nolimits_{j=0}^{M-1} v_{tj}\cV_{t,j}} - \sum_{t=1}^{T}\sum_{j= 0}^{M-1} v_{tj}\zeta_{t,j} [(\bD_{t+1}^{\mathrm{UCB}})_+]_j + \cR_{\mathrm{peel}}.
\end{align*}
The Cauchy--Schwarz step in the first term uses $\eta_{t,j}^2 = \zeta_{t,j}$ together with the coverage bound \begin{align*}
-\sum_{t=1}^{T}\sum_{j=0}^{M-1} v_{tj}\,\zeta_{t,j} \le \log(MTN/\delta)\sum_{t=1}^{T}\sum_{j=0}^{M-1} \frac{\PP(y_t^\star>j\lvert x_1=0)}{N_{t,j}\vee 1} \le \frac{MT\log(MTN/\delta)}{\cN^\star},
\end{align*}
which follows from~\eqref{eq: ucb-analysis-v-bound} and the definition~\eqref{eq: def-effective-sample-size} of $\cN^\star$.
By~\eqref{eq: D-recursion}, $[(\bD_{t+1}^{\mathrm{UCB}})_+]_j= \cO\big((h_\infty + b_\infty)T\big)$; hence both $\cR_{\mathrm{peel}}$ and the $\zeta_{t,j}$ term are of lower order:
\begin{align*}
      \cR_{\mathrm{peel}} \lesssim \sqrt{a_\star}\,(h_\infty +b_\infty)T \sum_{t=1}^T &\sum_{j=0}^{M-1}\PP(y_t^\star>j\lvert x_1=0)\,\eta_{t,j} \lesssim \frac{(h_\infty +b_\infty)MT^2\log(MTN/\delta)}{\cN^\star},\\
      -\sum_{t=1}^{T}\sum_{j= 0}^{M-1} v_{tj}\zeta_{t,j} [(\bD_{t+1}^{\mathrm{UCB}})_+]_j &\lesssim (h_\infty +b_\infty)T\log(MTN/\delta) \sum_{t=1}^T \sum_{j=0}^{M-1} \frac{\PP(y_t^\star >j \lvert x_1 = 0)}{N_{t,j}\vee 1} \\
      &\lesssim \frac{(h_\infty +b_\infty)MT^2\log(MTN/\delta)}{\cN^\star},
\end{align*}
{where for $\cR_{\mathrm{peel}}$, with $p_{t,j}^\star:=\PP(y_t^\star>j\lvert x_1=0)$, we used $\sqrt{a_\star}\le\sqrt{\tfrac18\log(MTN/\delta)/(\cN^\star\wedge N)}$ together with $\sum_{t,j}p_{t,j}^\star\eta_{t,j}=\sqrt{\log(MTN/\delta)}\sum_{t,j}\frac{p_{t,j}^\star}{\sqrt{N_{t,j}\vee1}}$: when $\cN^\star\le N$, Cauchy--Schwarz gives $\sum_{t,j}\frac{p_{t,j}^\star}{\sqrt{N_{t,j}\vee1}}\le\sqrt{MT}\sqrt{MT/\cN^\star}$, and when $\cN^\star>N$, the bound $N_{t,j}\vee1\le N$ gives $\sum_{t,j}\frac{p_{t,j}^\star}{\sqrt{N_{t,j}\vee1}}\le\sqrt N\sum_{t,j}\frac{p_{t,j}^\star}{N_{t,j}\vee1}=\sqrt N\,MT/\cN^\star$; either way $\sqrt{a_\star}\sum_{t,j}p_{t,j}^\star\eta_{t,j}\lesssim MT\log(MTN/\delta)/\cN^\star$.}
Now it remains to control the $\sum_{t,j} v_{tj}\cV_{t,j}$ term, for this purpose, we introduce the following identity:
\begin{lemma}\label{lem: V-identity}
 For every $t\in [T],j \geq s_{t+1}$, with  $\bgamma_{t+1} \in \RR^M$ defined as
       $ [\bgamma_{t+1}]_{i}:=  [(\bD^\mathrm{UCB}_{t+1})_+]^2_i$ for $ i \in [M-1]_+$,
 it holds that
$    \cV_{t,j} = \big[\bA_t^{\mathrm{UCB}}(s_{t+1})\bgamma_{t+1}\big]_j - \big(\big[\bA_t^{\mathrm{UCB}}(s_{t+1})(\bD_{t+1}^\mathrm{UCB})_+\big]_j\big)^2$.
\end{lemma}
\noindent Moreover, write $m_{t,j}:=\big[\bA_t^{\mathrm{UCB}}(s_{t+1})(\bD_{t+1}^\mathrm{UCB})_+\big]_j\geq 0$ and $c_{t,j}:=[\bc_t^{\mathrm{UCB}}]_j$. By~\eqref{eq: D-forward-equation}, $[\bD_t^\mathrm{UCB}]_j=m_{t,j}+c_{t,j}$, so that $[(\bD_t^\mathrm{UCB})_+]_j=[m_{t,j}+c_{t,j}]_+$, while $c_{t,j}=(h_t+b_t)F_t^{\mathrm{UCB}}(j)-b_t\leq h_t$.

Now we can show that $m_{t,j}^2\geq[(\bD_t^\mathrm{UCB})_+]_j^2-2h_t[(\bD_t^\mathrm{UCB})_+]_j$. Actually, if $c_{t,j}\leq 0$ then $[(\bD_t^\mathrm{UCB})_+]_j=[m_{t,j}+c_{t,j}]_+\leq m_{t,j}$ and the claim holds, otherwise if $c_{t,j}>0$ then $[(\bD_t^\mathrm{UCB})_+]_j=m_{t,j}+c_{t,j}$, and
\begin{align*}
    m_{t,j}^2=\big([(\bD_t^\mathrm{UCB})_+]_j-c_{t,j}\big)^2\geq[(\bD_t^\mathrm{UCB})_+]_j^2-2c_{t,j}[(\bD_t^\mathrm{UCB})_+]_j\geq[(\bD_t^\mathrm{UCB})_+]_j^2-2h_t[(\bD_t^\mathrm{UCB})_+]_j,
\end{align*}
using $c_{t,j}\leq h_t$ and $[(\bD_t^\mathrm{UCB})_+]_j\geq 0$, as desired. 

By $[\bgamma_t]_j=[(\bD_t^\mathrm{UCB})_+]_j^2$ and $h_t\leq h_\infty$, we have
    $\big(\big[\bA_t^{\mathrm{UCB}}(s_{t+1})(\bD_{t+1}^\mathrm{UCB})_+\big]_j\big)^2 \geq  [\bgamma_t]_j - 2h_\infty[(\bD_{t}^\mathrm{UCB})_+]_j$.
As a result,
$\cV_{t,j} 
\leq  \big[\bA_t^{\mathrm{UCB}}(s_{t+1})\bgamma_{t+1}\big]_j - [\bgamma_{t}]_{j} + 2 h_\infty \big[(\bD_{t}^\mathrm{UCB})_+-\bD_{t}\big]_j  +2h_\infty [\bD_{t}]_j$,
this gives \begin{align*}
    -\sum_{j=0}^{M-1} v_{tj}\cV_{t,j} \overset{(\mathrm{a})}
    {\le}& -\langle \bv_t, \bA_t^{\mathrm{UCB}}(s_{t+1})\bgamma_{t+1} - \bgamma_{t} \rangle -2h_\infty \langle \bv_t, \bD_t^\mathrm{UCB}-\bD_t \rangle -2h_\infty \langle \bv_t, \bD_t \rangle \\
    =&\underbrace{\langle \bv_t, \big[\bA_t(s_{t+1})-\bA^\mathrm{UCB}_t(s_{t+1})\big]\bgamma_{t+1} \rangle -2h_\infty \langle \bv_t, \bD_t^\mathrm{UCB}-\bD_t \rangle}_{:= \cK_{t,1}} \\
    &+\underbrace{\langle \bv_t,\bgamma_{t}- \bA_t(s_{t+1})\bgamma_{t+1}  \rangle }_{:= \cK_{t,2}}+ \underbrace{2h_\infty \langle -\bv_t, \bD_t \rangle }_{:=\cK_{t,3}},
\end{align*} 
Where in the $\cK_{t,1}$ part we have used $\langle \bv_t, (\bD_t^{\mathrm{UCB}})_+\rangle = \langle \bv_t, \bD_t^{\mathrm{UCB}}\rangle $ by $v_{tj}= 0,\;\forall j < s_{t}\leq s_{t}^\star.$
Now it remains to bound the summation of $\cK_t$ terms.

\noindent\textbf{Bounding the summation of $\cK_{t,1}$.} By the summation-by-parts identity, we have
\begin{align*}
   &\big( \big[\bA^\mathrm{UCB}_t(s_{t+1}) - \bA_t(s_{t+1})\big]\bgamma_{t+1}\big)_j = \sum_{m=s_{t+1}}^j \big[F_t^{\mathrm{UCB}}(j-m) - F_t(j-m)\big]  \big(\big[(\bD^\mathrm{UCB}_{t+1})_+\big]_m^2 - \big[(\bD^\mathrm{UCB}_{t+1})_+\big]_{m-1}^2\big)\\
   &\leq 2\max_{m} \big[(\bD^\mathrm{UCB}_{t+1})_+\big]_{m}  \cdot \sum_{m=s_{t+1}}^j \big[F_t^{\mathrm{UCB}}(j-m) - F_t(j-m)\big]  \big(\big[(\bD^\mathrm{UCB}_{t+1})_+\big]_m - \big[(\bD^\mathrm{UCB}_{t+1})_+\big]_{m-1}\big)\\
   & \lesssim (h_\infty +b_\infty)T  \cdot \sum_{m=s_{t+1}}^j \big[F_t^{\mathrm{UCB}}(j-m) - F_t(j-m)\big]  \big(\big[(\bD^\mathrm{UCB}_{t+1})_+\big]_m - \big[(\bD^\mathrm{UCB}_{t+1})_+\big]_{m-1}\big)\\
   &= (h_\infty +b_\infty)T  \cdot \big( \big[\bA^\mathrm{UCB}_t(s_{t+1}) - \bA_t(s_{t+1})\big]\bD^\mathrm{UCB}_{t+1}\big)_j.
\end{align*}
Since $-\bv_t\ge\bm 0$, this entry-wise bound leads to
\begin{align*}
    \langle \bv_t, [\bA_t(s_{t+1})-\bA^\mathrm{UCB}_t(s_{t+1})]\bgamma_{t+1}\rangle \lesssim (h_\infty+b_\infty)T\,\langle \bv_t, [\bA_t(s_{t+1})-\bA^\mathrm{UCB}_t(s_{t+1})]\bD^\mathrm{UCB}_{t+1}\rangle.
\end{align*}
On the other hand, applying the one-step identity in the proof of Theorem~\ref{thm: cost-decomposition} with $\hat{\bF}=\bF^\mathrm{UCB}$ and writing $\br_\ell := \big[\bA_\ell(s_{\ell+1})-\bA^\mathrm{UCB}_\ell(s_{\ell+1})\big]\bD^\mathrm{UCB}_{\ell+1}+\Delta\bc_\ell$, we have
\begin{align*}
    \langle \bv_t, \bD_t - \bD^\mathrm{UCB}_t\rangle = \langle \bv_{t+1}, \bD_{t+1}-\bD^\mathrm{UCB}_{t+1}\rangle + \langle \bv_t, \br_t\rangle + \underbrace{\langle \bu_{t+1}, \bD^\mathrm{UCB}_{t+1}\rangle}_{\le 0} - \underbrace{\langle \bq_{t+1}, \bD_{t+1}\rangle}_{\ge 0}.
\end{align*}
Here $\langle \bu_{t+1}, \bD^\mathrm{UCB}_{t+1}\rangle\le0$ since $\bu_{t+1}\le\bm0$ is supported on $[s_{t+1},s_{t+1}^\star)$ where $\bD^\mathrm{UCB}_{t+1}\ge\bm0$ (as $j\ge s_{t+1}$); and $\langle \bq_{t+1}, \bD_{t+1}\rangle\ge0$ since $\bq_{t+1}\le\bm0$ is supported on the same interval where $\bD_{t+1}=\bD^\star_{t+1}<\bm0$ (as $j<s_{t+1}^\star$). Iterating from $t$ to $T$ with $\bD_{T+1}=\bD^\mathrm{UCB}_{T+1}=\bm0$ and dropping these two non-positive contributions gives
\begin{align*}
    \langle \bv_t, \bD_t - \bD^\mathrm{UCB}_t \rangle & \leq \sum_{\ell = t}^{T} \langle \bv_\ell, \br_\ell \rangle \leq \sum_{\ell = 1}^{T} \langle \bv_\ell, \br_\ell \rangle,
\end{align*}
where the second inequality holds since every term is non-negative, as shown in \eqref{eq: ucb-analysis-tmp1} and \eqref{eq: ucb-analysis-tmp2}.
Combining these two bounds together, we obtain \begin{align*}
    \sum_{t=1}^T\cK_{t,1}\lesssim (h_\infty+b_\infty) T \cdot \sum_{\ell = 1}^{T} \langle \bv_\ell, \big(\bA_\ell(s_{\ell+1}) - \bA^\mathrm{UCB}_\ell(s_{\ell+1}) \big) \bD^\mathrm{UCB}_{\ell+1} + \Delta \bc_\ell \rangle
\end{align*}
\noindent\textbf{Bounding the summation of $\cK_{t,2}$.} Set $\bgamma_{T+1}:=\bm 0$. By the recursion $\bA_t(s_{t+1})^\top\bv_t = \bv_{t+1}-\bu_{t+1}$ of Proposition~\ref{prop: v-propagation}, we have \begin{align*}
    \sum_{t=1}^T \cK_{t,2} &= \sum_{t=1}^T \langle \bv_t, \bgamma_t - \bA_t(s_{t+1}) \bgamma_{t+1}\rangle = \sum_{t=1}^{T-1} \big[ \langle \bv_t ,\bgamma_t \rangle - \langle \bv_{t+1} ,\bgamma_{t+1} \rangle + \langle \bu_{t+1}, \bgamma_{t+1}\rangle \big] + \langle \bv_T, \bgamma_T\rangle\\
    &= \langle \bv_1, \bgamma_1\rangle + \sum_{t=2}^T \langle \bu_t, \bgamma_t\rangle  \leq 0,
\end{align*}
where the telescoping uses $\bgamma_{T+1}=\bm 0$, and the inequality holds since $\bgamma_t\ge\bm 0$ while $\bv_1=\bu_1\le\bm 0$ and $\bu_t\le\bm 0$ (by $\sigma_t\le 0$ in the UCB setting).

\noindent\textbf{Bounding the summation of $\cK_{t,3}$.} For $\cK_{t,3}$, we first bound $-\langle\bv_t,\bD_t\rangle$ by splitting at $s_t^\star$: \begin{align*}
-\langle \bv_t, \bD_t\rangle &= \sum_{j<s_t^\star}(-v_{tj})[\bD_t]_j + \sum_{j\ge s_t^\star}(-v_{tj})[\bD_t]_j \le \sum_{j=s_t^\star}^{M-1}(-v_{tj})[\bD_t]_j\\
&\leq  \sum_{j=s_{t}^\star}^{M-1} \PP(y_t^\star > j\lvert x_1 = 0)[\bD_t]_j \lesssim (h_\infty +b_\infty)M,
\end{align*}
where the first inequality drops the indices $j<s_t^\star$, for which $-v_{tj}\ge0$ and $[\bD_t]_j\leq 0$ make each term non-positive; the second applies $-v_{tj}\le\PP(y_t^\star>j\lvert x_1=0)$ from~\eqref{eq: ucb-analysis-v-bound} on $j\ge s_t^\star$, where $[\bD_t]_j\ge0$; and the last uses Proposition~\ref{prop: over-shooting-price}. Since $\cK_{t,3}=2h_\infty\langle-\bv_t,\bD_t\rangle$, this gives $\sum_{t=1}^T \cK_{t,3} \lesssim h_\infty(h_\infty +b_\infty)MT \le (h_\infty+b_\infty)^2 MT.$

\noindent\textbf{Putting all together.} Combining the bounds on $\cK_{t,1},\cK_{t,2},\cK_{t,3}$ leads to
\begin{align*}
    -\sum_{t=1}^T\sum_{j=0}^{M-1} v_{tj}\cV_{t,j} \le \sum_{t=1}^T\big(\cK_{t,1}+\cK_{t,2}+\cK_{t,3}\big) \lesssim (h_\infty+b_\infty)T\,\mathsf{Err}_T + (h_\infty +b_\infty)^2 MT.
\end{align*}
Substituting this into the bound on $\sum_{t,j} v_{tj}[\bG_t\bH^\mathrm{UCB}_t]_j$, and adding the immediate-cost term~\eqref{eq: ucb-analysis-tmp1} together with the lower-order terms $\cR_{\mathrm{peel}}$ and the $\zeta_{t,j}$ contribution, we obtain
\begin{align*}
   \mathsf{Err}_T &\lesssim \sqrt{\frac{MT\log(MTN/\delta)\log(e/a_\star)}{\cN^\star}}\cdot\sqrt{(h_\infty+b_\infty)T\,\mathsf{Err}_T + (h_\infty +b_\infty)^2 MT } \\
   &\quad+ \frac{(h_\infty + b_\infty)MT^2\log(MTN/\delta)}{\cN^\star} + (h_\infty+b_\infty)MT\sqrt{\frac{\log(MTN/\delta)}{\cN^\star}}.
\end{align*}
Using $\sqrt{x+y}\le\sqrt x+\sqrt y$ to split the first term and Young's inequality $C\sqrt{\alpha\,\mathsf{Err}_T}\le \frac{1}{2}\mathsf{Err}_T + C^2\alpha$ with $\alpha:=\dfrac{(h_\infty+b_\infty)MT^2\log(MTN/\delta)\log(e/a_\star)}{\cN^\star}$ to absorb the self-referential term, we have then
\begin{align*}
\mathsf{Err}_T \lesssim  (h_\infty + b_\infty) M \cdot \bigg(T \sqrt{\frac{\log(MTN/\delta)\log(e/a_\star)}{\cN^\star}} + \frac{T^2\log(MTN/\delta)\log(e/a_\star)}{\cN^\star}\bigg).
\end{align*}
Finally, since $a_\star\ge 1/(8{(\cN^\star\wedge N)})$, we have $\log(e/a_\star)\le\log(8e{(\cN^\star\wedge N)})\le 4\log\big(e(\cN^\star\wedge N)\big) \le 4\log(eN)$, which yields the bound stated in Theorem~\ref{thm: censored-sample-complexity}.
\end{proof}

\section{Proof of Corollary~\ref{corollary: C-star-sample-complexity}}

For each $t,j$, noticing that by the Chernoff's bound, we have the following dichotomy result:
\begin{lemma}\label{lem: N-bound-offline}
With probability at least $1-\delta,$ for every $t,j$ at least one of the following statements holds:
\begin{enumerate}[nosep]
    \item  $ \PP(y_t^{(b)} >j)\leq 8 \log(MT/\delta)/N,$
    \item $N_{t,j}  \geq   \PP(y_t^{(b)} >j)N/2.$
\end{enumerate}
\end{lemma}

With Lemma~\ref{lem: N-bound-offline}, we can divide the index set  to $[M-1]_+  = \cI_t \cup \cI_t^c$  accordingly as \begin{align*}
    \cI_{t}:= \{ j\in [M-1]_+: \PP(y_t^{(b)} > j) \leq 8\log(MT/\delta)/N \}
\end{align*} 
and relate the effective sample size $\cN^\star$ to $\cC^\star/N$ the following:
\begin{align*}
    \frac{MT}{\cN^\star} &= \sum_{t=1}^T \sum_{j=0}^{M-1} \frac{\PP(y_t^\star > j \lvert x_1 = 0)}{N_{t,j}\vee 1}\\
     &=   \sum_{t=1}^T \bigg[\sum_{j\in \cI_t} \frac{\PP(y_t^\star > j \lvert x_1 = 0)}{N_{t,j}\vee 1} + \sum_{j\notin \cI_t} \frac{\PP(y_t^\star > j \lvert x_1 = 0)}{N_{t,j}\vee 1}  \bigg]\\
    &\leq \sum_{t=1}^T \bigg[\sum_{j\in \cI_t} \frac{8\PP(y_t^\star > j \lvert x_1 = 0)\log (MT/\delta) }{\PP(y_t^{(b)}> j) N} + \sum_{j\notin \cI_t} \frac{2\PP(y_t^\star > j \lvert x_1 = 0)}{\PP(y_t^{(b)}> j) N}  \bigg]\\
    &\leq 8MT\log(MT/\delta) \frac{\cC^\star}{N}.
\end{align*}
Bringing this bound to Theorem~\ref{thm: censored-sample-complexity}, with the Lemma and the Theorem each applied at confidence level $\delta/2$ and a union bound (the enlarged $\log(2MTN/\delta) \leq 2\log(MTN/\delta)$ being absorbed into $c_0$), leads to the desired result.

\subsection{Proof of auxiliary results}

\subsubsection{Proof of Lemma~\ref{lem: C-hat-to-C-gap}}
\begin{proof}[Proof of Lemma~\ref{lem: C-hat-to-C-gap}]
    Under~\eqref{eq: partial-CDF-confidence-bound-general}, for any $j\in [M]_+$ \begin{align}\label{eq: C-gap-proof-tmp1}
          |\tilde{F}(j)(1-\tilde{F}(j)) - F(j)(1-F(j))| = |\tilde{F}(j) - F(j)|\cdot|1-\tilde{F}(j) - F(j)| \leq \cC_{j},
    \end{align}
    this then gives
    \begin{align*}
          &\sqrt{\frac{\tilde{F}(j) \big(1 - \tilde{F}(j) \big)\log(Mn/\delta) }{n_{j}}}\leq \sqrt{\frac{F(j)(1-F(j))\log (Mn/\delta)}{n_{j}}}  + \sqrt{\frac{\cC_{j} \log(Mn/\delta)}{n_{j}}}\\
      &\leq \sqrt{\frac{F(j)(1-F(j))\log (Mn/\delta)}{n_{j}}} + \frac{\cC_{j}}{2} + \frac{\log(Mn/\delta)}{2n_{j}} \leq \cC_{j}.
    \end{align*}
Similarly, by~\eqref{eq: C-gap-proof-tmp1},
\begin{align}\label{eq: C-gap-proof-tmp2}
          &\sqrt{\frac{{F}(j) \big(1 - {F}(j) \big)\log(Mn/\delta) }{n_{j}}}\leq \sqrt{\frac{\tilde{F}(j)(1-\tilde{F}(j))\log (Mn/\delta)}{n_{j}}}  + \sqrt{\frac{\cC_{j} \log(Mn/\delta)}{n_{j}}}\\
      &\leq \sqrt{\frac{\tilde{F}(j)(1-\tilde{F}(j))\log (Mn/\delta)}{n_{j}}} + \frac{\cC_{j}}{2} + \frac{\log(Mn/\delta)}{2n_{j}} \leq \cC_{j}.
    \end{align}
    As a result, \begin{align*}
        \lvert \hat{\cC}_{j} - \cC_{j}\rvert &= 4\left\lvert\sqrt{\frac{\tilde{F}(j) \big(1 - \tilde{F}(j) \big)\log(Mn/\delta) }{n_{j}}} - \sqrt{\frac{{F}(j) \big(1 - {F}(j) \big)\log(Mn/\delta) }{n_{j}}}\right\rvert\leq 4\cC_{j}.
     \end{align*}
     This gives the first inequality. To see the second inequality, noticing that~\eqref{eq: C-gap-proof-tmp2} also implies \begin{align*}
        \cC_{j} &=4\sqrt{\frac{{F}(j) \big(1 - {F}(j) \big)\log(Mn/\delta) }{n_{j}}} + 4\frac{\log(Mn/\delta)}{n_{j}} \\
        &\leq  4\sqrt{\frac{\tilde{F}(j)(1-\tilde{F}(j))\log (Mn/\delta)}{n_{j}}}  + 4 \sqrt{\frac{\cC_{j} \log(Mn/\delta)}{n_{j}}} + \frac{4\log(Mn/\delta)}{n_{j}}\\
        & \leq \hat{\cC}_{j} + \frac{\cC_{j}}{2} + \frac{8\log(Mn/\delta)}{n_{j}},
     \end{align*}
     where the last step uses $4\sqrt{\cC_{j}\log(Mn/\delta)/n_{j}} \leq \cC_{j}/2 + 8\log(Mn/\delta)/n_{j}$ (AM--GM $\sqrt{xy}\leq \alpha x + y/(4\alpha)$ with $\alpha=1/8$). Moving $\cC_{j}/2$ term to left and multiplying both sides by $8$ gives the desired result.
\end{proof}

\subsubsection{Proof of Lemma~\ref{lem: C-self-bounding}}

\begin{proof}[Proof of Lemma~\ref{lem: C-self-bounding}]

We have by $F(j)\geq F(j_0),$ \begin{align*}
    F(j_0)(1-F(j_0))- F(j)(1-F(j))  = (F(j)-F(j_0)) (F(j)+F(j_0)-1) \leq  (F(j)-F(j_0)).
\end{align*}
As a result, 
\begin{align*}
    \cC_{j_0} &\leq 4 \sqrt{\frac{F(j_0)(1-F(j_0))\log(Mn/\delta)}{n_{j_0}}} + \frac{4\log(Mn/\delta)}{n_{j_0}}\\
    &\leq 4 \sqrt{\frac{F(j)(1-F(j))\log(Mn/\delta)}{n_{j_0}}} + \frac{4\log(Mn/\delta)}{n_{j_0}} + 4\sqrt{\frac{(F(j) - F(j_0))\log(Mn/\delta) }{n_{j_0}}}\\
    &\overset{\text{(a)}}{\leq} 4 \sqrt{\frac{F(j)(1-F(j))\log(Mn/\delta)}{n_{j}}} + \frac{4\log(Mn/\delta)}{n_{j}} + 4\sqrt{\frac{(F(j) - F(j_0))\log(Mn/\delta) }{n_{j}}}\\
    &\overset{\text{(b)}}{\leq} \cC_{j} + \sqrt{\frac{16A\cC_{j_0}\log(Mn/\delta)}{n_{j}}} \leq \cC_{j} + \frac{\cC_{j_0}}{2} + \frac{8A \log(Mn/\delta)}{n_{j}}.
\end{align*}
Then moving the $\cC_{j_0}/2$ to left-hand-side and multiplying both sides by $2$ gives $$\cC_{j_0} \leq 2\cC_{j} + 16A\log(Mn/\delta)/n_{j} \leq (2 + 4A)\cC_{j},$$
as desired. 

In the above arguments, (a) uses $n_{j_0} \geq n_{j}$; (b) uses the condition $F(j)-F(j_0) \leq A \cC_{j_0}$  and the weighted AM-GM inequality $$\sqrt{16A\cC_{j_0}\log/n_{j}}\leq 16A\alpha\cC_{j_0} + \log/(4\alpha n_{j})$$ with $\alpha=1/(32A)$.
\end{proof}

\subsubsection{Proof of Lemma~\ref{lem: ucb-variance-transfer}}

\begin{proof}[Proof of Lemma~\ref{lem: ucb-variance-transfer}]
Let
$
\varphi(x):=x(1-x).
$
Since $\varphi$ is $1$-Lipschitz on $[0,1]$, we have
$
|\varphi(p)-\varphi(u)|\le |p-u|=d.
$
As a result,
\begin{align}
p(1-p)&\le u(1-u)+d,
\label{eq: p-var-to-u-var}\\
u(1-u)&\le p(1-p)+d.
\label{eq: u-var-to-p-var}
\end{align}
To see~\eqref{eq: p-to-u-variance-transfer}, suppose
$
d\le a\sqrt{p(1-p)\xi}+b\xi
$ holds, using~\eqref{eq: p-var-to-u-var}, we obtain
\begin{align*}
d \le
a\sqrt{p(1-p)\xi}+b\xi \le
a\sqrt{\big(u(1-u)+d\big)\xi}+b\xi \le
a\sqrt{u(1-u)\xi}+ a\sqrt{d\xi} + b\xi.
\end{align*}
By AM--GM inequality,
$$
a\sqrt{d\xi}
\le
\frac{1}{2} d+\frac{a^2}{2}\xi \implies d \le a\sqrt{u(1-u)\xi} + \frac{1}{2} d
+
\left(\frac{a^2}{2}+b\right)\xi.
$$
Absorbing $ d$ into the left-hand side gives
$d\le 2a\sqrt{u(1-u)\xi} + (a^2+2b)\xi. $
This proves~\eqref{eq: p-to-u-variance-transfer}.
And~\eqref{eq: u-to-p-variance-transfer} follows the same argument with~\eqref{eq: p-var-to-u-var} replaced by~\eqref{eq: u-var-to-p-var}.
\end{proof}

\subsubsection{Proof of Lemma~\ref{lem: peeling-bound}}

\begin{proof}[Proof of Lemma~\ref{lem: peeling-bound}]
Split the index set into an extreme part and a middle part,
\begin{align*}
    \cM_{\mathrm{ext}}:=\{m\in\cM:p_m(1-p_m)\leq a\},\qquad \cM_{\mathrm{mid}}:=\cM\setminus\cM_{\mathrm{ext}}.
\end{align*}
On $\cM_{\mathrm{ext}}$ we have $\sqrt{p_m(1-p_m)}\leq\sqrt a$, so the extreme part is bounded by the second term on the right-hand side,
$\sum_{m\in\cM_{\mathrm{ext}}}\sqrt{p_m(1-p_m)}\,h_m\leq\sqrt a\sum_{m\in\cM}h_m$.
It remains to control the middle part.

For $m\in\cM_{\mathrm{mid}}$ we have $p_m(1-p_m)>a>0$, hence $p_m\in(0,1)$ and the odds ratio $\theta_m:=p_m/(1-p_m)$ is well-defined. Since $p_m\leq 1$ gives $1-p_m\geq p_m(1-p_m)>a$ and $1-p_m\leq 1$ gives $p_m\geq p_m(1-p_m)>a$, we obtain $a<\theta_m<1/a$ on $\cM_{\mathrm{mid}}$. For each integer $\ell$ define the dyadic block $\cM_\ell:=\{m\in\cM_{\mathrm{mid}}:2^\ell\leq\theta_m<2^{\ell+1}\}$ and let $\cL:=\{\ell:\cM_\ell\neq\emptyset\}$; the two-sided bound on $\theta_m$ forces $|\cL|\lesssim\log(e/a)$.

We control the cross terms within a single block through the kernel $p_{m'}(1-p_m)$. Writing $p_m=\theta_m/(1+\theta_m)$ and $1-p_m=1/(1+\theta_m)$,
\begin{align*}
    p_{m'}(1-p_m)=\frac{\theta_{m'}}{(1+\theta_m)(1+\theta_{m'})},\qquad \sqrt{p_m(1-p_m)\,p_{m'}(1-p_{m'})}=\frac{\sqrt{\theta_m\theta_{m'}}}{(1+\theta_m)(1+\theta_{m'})},
\end{align*}
so that $\sqrt{p_m(1-p_m)\,p_{m'}(1-p_{m'})}=\sqrt{\theta_m/\theta_{m'}}\,p_{m'}(1-p_m)$. For $m<m'$ with $m,m'\in\cM_\ell$, the monotonicity assumption gives $p_m\geq p_{m'}$, hence $\theta_m\geq\theta_{m'}$, while the definition gives $\theta_m/\theta_{m'}<2$; together $1\leq\theta_m/\theta_{m'}\leq 2$ and therefore $\sqrt{p_m(1-p_m)\,p_{m'}(1-p_{m'})}\leq\sqrt2\,p_{m'}(1-p_m)$.
Consequently, setting $B_\ell:=\sum_{m\in\cM_\ell}\sqrt{p_m(1-p_m)}\,h_m$ and expanding the square, the non-negativity of $\{h_m\}$ yields
\begin{align*}
    B_\ell^2&=\sum\nolimits_{m\in\cM_\ell}h_m^2\,p_m(1-p_m)+2\sum\nolimits_{m<m', m,m'\in\cM_\ell}h_m h_{m'}\sqrt{p_m(1-p_m)\,p_{m'}(1-p_{m'})}\\
    &\lesssim\sum\nolimits_{m\in\cM_\ell}h_m^2\,p_m(1-p_m)+2\sum\nolimits_{m<m', m,m'\in\cM_\ell}h_m h_{m'}\,p_{m'}(1-p_m).
\end{align*}
Summing $\ell$, we have then
    $\sum_{m\in\cM_{\mathrm{mid}}}\sqrt{p_m(1-p_m)}\,h_m=\sum_{\ell\in\cL}B_\ell\leq\sqrt{|\cL|}\Big(\sum_{\ell\in\cL}B_\ell^2\Big)^{1/2}$.
Since each term involved is non-negative, so restricting the two sums to within-block pairs only decreases them; hence
\begin{align*}
    \sum_{\ell\in\cL}B_\ell^2\lesssim\sum_{m\in\cM}h_m^2\,p_m(1-p_m)+2\sum_{m<m'}h_mh_{m'}\,p_{m'}(1-p_m).
\end{align*}
Combining this with $|\cL|\lesssim\log(e/a)$ and the extreme part proves the lemma.
\end{proof}

\subsubsection{Proof of Lemma~\ref{lem: V-identity}}

Fix $t\in[T]$ and $j\geq s_{t+1}$, and denote $g_m:=[(\bD_{t+1}^{\mathrm{UCB}})_+]_m$ with the convention $g_{s_{t+1}-1}:=0$. By the definition of $\bH_t^{\mathrm{UCB}}$ in~\eqref{eq: H-def} and $[\bD_{t+1}^{\mathrm{UCB}}]_m\geq 0$ for $m\geq s_{t+1}$, it holds that $g_r=\sum_{m=s_{t+1}}^r[\bH_t^{\mathrm{UCB}}]_m,\forall r\geq s_{t+1}$.
Let $\tilde d$ be a random variable with CDF $F_t^{\mathrm{UCB}}$ and set
$Z_{t,j}:=\sum_{m=s_{t+1}}^j[\bH_t^{\mathrm{UCB}}]_m\,\bm{1}\{\tilde d\leq j-m\}$.
Since $\bm{1}\{\tilde d\leq j-m\}=\bm{1}\{m\leq j-\tilde d\}$ and $\tilde d\geq 0$, we have $Z_{t,j}=g_{j-\tilde d}\,\bm{1}\{\tilde d\leq j-s_{t+1}\}$.
By $\big[\bA_t^{\mathrm{UCB}}(s_{t+1})\bx\big]_j=\sum_{r=s_{t+1}}^j\PP(\tilde d=j-r)\,x_r$, the substitution $r=j-\tilde d$ gives
    $\EE[Z_{t,j}]=\sum_{r=s_{t+1}}^j\PP(\tilde d=j-r)\,g_r=\big[\bA_t^{\mathrm{UCB}}(s_{t+1})(\bD_{t+1}^{\mathrm{UCB}})_+\big]_j,$ and $  \EE[Z_{t,j}^2]=\sum_{r=s_{t+1}}^j\PP(\tilde d=j-r)\,g_r^2=\big[\bA_t^{\mathrm{UCB}}(s_{t+1})\bgamma_{t+1}\big]_j$.
Therefore, $\big[\bA_t^{\mathrm{UCB}}(s_{t+1})\bgamma_{t+1}\big]_j-\big(\big[\bA_t^{\mathrm{UCB}}(s_{t+1})(\bD_{t+1}^{\mathrm{UCB}})_+\big]_j\big)^2=\Var(Z_{t,j})$.

It remains to expand this variance from the original representation of $Z_{t,j}$:
\begin{align*}
    \Var(Z_{t,j})=&\sum_{m=s_{t+1}}^j[\bH_t^{\mathrm{UCB}}]_m^2\,\Var\big(\bm{1}\{\tilde d\leq j-m\}\big)\\
&+2\sum\nolimits_{s_{t+1}\leq m<m'\leq j}[\bH_t^{\mathrm{UCB}}]_m[\bH_t^{\mathrm{UCB}}]_{m'}\,\Cov\big(\bm{1}\{\tilde d\leq j-m\},\bm{1}\{\tilde d\leq j-m'\}\big).
\end{align*}
The diagonal terms equal $F_t^{\mathrm{UCB}}(j-m)\big(1-F_t^{\mathrm{UCB}}(j-m)\big)$. For $m<m'$ the events are nested, $\{\tilde d\leq j-m'\}\subseteq\{\tilde d\leq j-m\}$, so
\begin{align*}
    \Cov\big(\bm{1}\{\tilde d\leq j-m\},\bm{1}\{\tilde d\leq j-m'\}\big)&=\PP(\tilde d\leq j-m')-\PP(\tilde d\leq j-m)\,\PP(\tilde d\leq j-m')\\
    &=F_t^{\mathrm{UCB}}(j-m')\big(1-F_t^{\mathrm{UCB}}(j-m)\big).
\end{align*}
Bringing these two equations back gives the desired identity.

\section{Proof of Results in Section~\ref{sec: extensions-online}}\label{appendix: online-censored-proofs}

\subsection{Proof of Lemma~\ref{lem: online-policy-monotonicity}}

To prove Lemma~\ref{lem: online-policy-monotonicity}, we first record the following monotone result for the LCB, which follows by the same DP-monotonicity argument as Lemma~\ref{lem: policy-ordering-UCB}.

\begin{lemma}\label{lem: policy-ordering-LCB} Let $\pi = \{s_t\}_{t=1}^T$ be the base-stock policy obtained by solving the DP with $\{\bF_t^\mathrm{LCB}\}_{t=1}^T$. If the event in~\eqref{eq: partial-CDF-confidence-bound-general} holds for every $t\in [T]$, then $s_t \geq s_t^\star$ for all $t \in [T].$
\end{lemma}

\begin{proof}[Proof of Lemma~\ref{lem: policy-ordering-LCB}]
The argument mirrors that of Lemma~\ref{lem: policy-ordering-UCB}, with all inequalities reversed. Recall that the base-stock levels obey
$s_t = \min\{y: D_t^\mathrm{LCB}(y)\ge 0\},\quad s^\star_t = \min\{y: D^\star_t(y)\ge 0\}, \forall t\in [T].$
Since both $\bD_t^\mathrm{LCB}$ and $\bD_t^\star$ are non-decreasing, $\Delta \bD_t := \bD^\star_t - \bD_t^\mathrm{LCB} \geq \bm{0}$ for all $t$ forces the lower-biased derivative to cross zero no earlier than the optimal one, i.e.\ $s_t \geq s_t^\star$; it therefore suffices to establish $\Delta \bD_t \geq \bm{0}$.
We proceed by induction. When $t = T$,
$\Delta \bD_T = (h_T+b_T)(\bF_T-\bF_T^\mathrm{LCB}) \geq \bm{0}$
since $\bF_T^\mathrm{LCB} \leq \bF_T$. Now suppose $\Delta \bD_{k} \geq \bm{0}$ for all $k>t$. Applying~\eqref{eq: recursion-D-diff} with $\hat{\bF} = \bF^\mathrm{LCB}$ gives
\begin{align*}
    \Delta \bD_{t} = \bA_{t}(s_{t+1})\Delta \bD_{t+1} + \Delta \bA_t(s_{t+1})\bD_{t+1}^\mathrm{LCB} + (h_t+b_t)\Delta \bF_t + \big[\bA_t(s_{t+1}^\star) - \bA_t(s_{t+1})\big]\bD_{t+1}^\star,
\end{align*}
and each of the four terms is non-negative:
\begin{enumerate}
    \item $\bA_t(s_{t+1})\Delta \bD_{t+1} \geq \bm{0}$ by the induction hypothesis and $\bA_t(s_{t+1}) \geq \bm{0}$ entry-wise;
    \item $(h_t+b_t)\Delta \bF_t \geq \bm{0}$ since $\Delta \bF_t = \bF_t - \bF_t^\mathrm{LCB} \geq \bm{0}$;
    \item by the summation-by-parts identity~\eqref{eq: AD-to-GH}, $\Delta \bA_t(s_{t+1})\bD_{t+1}^\mathrm{LCB} = \bG_t\bH_t^\mathrm{LCB} \geq \bm{0}$, with $\bG_t \geq \bm{0}$ (as $\bF_t \geq \bF_t^\mathrm{LCB}$) and $\bH_t^\mathrm{LCB} \geq \bm{0}$ (as $\bD_{t+1}^{\mathrm{LCB}}$ is increasing) entry-wise;
    \item by the induction hypothesis $s_{t+1} \geq s_{t+1}^\star$, so the matrix difference $\bA_t(s_{t+1}^\star) - \bA_t(s_{t+1})$ is supported on $s_{t+1}^\star \leq k \leq s_{t+1}-1$, on which $[\bD_{t+1}^\star]_k \geq 0$; explicitly, for $j \in [M-1]_+$,
    $$ \big[[\bA_t(s_{t+1}^\star) - \bA_t(s_{t+1})]\bD_{t+1}^\star\big]_j = \sum_{k = s_{t+1}^\star}^{\min\{ s_{t+1}-1,\, j\} } \mu_{t,j-k} [\bD^\star_{t+1}]_k \geq 0.$$
\end{enumerate}
This then finishes the induction.
\end{proof}

\begin{proof}[Proof of Lemma~\ref{lem: online-policy-monotonicity}]
The lower bound $\hat{\bF}_t^{{(1)}} = \bm{0} \leq \hat{\bF}_t^{(k)}$ and the across-episode monotonicity $\hat{F}_t^{(k-1)}(j) \leq \hat{F}_t^{(k)}(j)$ are immediate from the maximum update in Step~4 of Algorithm~\ref{alg: censored-online}. For the upper bound, Proposition~\ref{prop: biased-F} gives $F_t^{\mathrm{LCB},(k)}(j) \leq F_t(j)$ on $\cE^{(K)}$; since $\hat{F}_t^{{(1)}} \equiv 0 \leq F_t(j)$ and the maximum of two quantities bounded by $F_t(j)$ is again bounded by $F_t(j)$, induction on $k$ yields $\hat{F}_t^{(k)}(j) \leq F_t(j)$ for all $k$. This proves the CDF chain. The base-stock ordering follows from the DP-monotonicity argument of Lemma~\ref{lem: policy-ordering-LCB}: a coordinate-wise larger CDF can only lower the optimal base-stock level, so the non-decreasing sequence $\{\hat{\bF}_t^{(k)}\}_k$ induces the non-increasing sequence $s_t^{(1)} \geq \cdots \geq s_t^{(K)}$, while $\hat{\bF}_t^{(k)} \leq \bF_t$ gives $s_t^{(k)} \geq s_t^\star$.
\end{proof}

In particular, Lemma~\ref{lem: online-policy-monotonicity} ensures the following monotone property:
\begin{align}\label{eq: J-monotone-online}
    q^{(k)}_{t,j}\leq q^{(k-1)}_{t,j}\leq \dots \leq q^{(1)}_{t,j}.
\end{align}

\subsection{Proof of Proposition~\ref{prop: online-k-th-episode-bound}}

The proof of Proposition~\ref{prop: online-k-th-episode-bound} follows a similar argument as in the proof of Theorem~\ref{thm: censored-sample-complexity}, we provide the detailed arguments here for completeness. Throughout the proof, we fix an episode index $k$ and proceed analysis conditional on $\cF_k$, underwhich both $ q_{t,j}^{(k)},\cN^{(k)}_{\mathsf{cov}},\cE^{(k)}$ are measurable. 

We work on the event $\cE^{(k)}$ and condition on $\cF_k$ throughout. The argument runs parallel to the proof of Theorem~\ref{thm: censored-sample-complexity}: once the offline (UCB) quantities are translated into their episode-$k$ (LCB) analogues, the only structural change is that the optimistic bias reverses the signs that pessimism induced. We first introduce the corresponding notation and then detail only the steps where the reversal matters. 

Compared with the proof of Theorem~\ref{thm: censored-sample-complexity}, every quantity carrying the superscript $\mathrm{UCB}$ in the proof of Theorem~\ref{thm: censored-sample-complexity} is replaced by the episode-$k$ quantity built from the lower-biased CDFs $\hat\bF_t^{(k)}$. More precisely, we write $\hat\bD_t^{(k)},\bA_t^{(k)}(\cdot),\bH_t^{(k)},\bG_t^{(k)},\bc_t^{(k)}, \bv_t^{(k)}$ for the induced quantities under $\{\hat{\bF}_t^{(k)}\}$ as introduced during the proof of Theorem~\ref{thm: censored-sample-complexity} or in Table~\ref{tab: section3-notation}. We also define the confidence radius related notations
$    \eta_{t,j}:=\sqrt{\frac{\log(KMT/\delta)}{N_{t,j}^{(k)}\vee1}}, \zeta_{t,j}:=\frac{\log(KMT/\delta)}{N_{t,j}^{(k)}\vee1}$.

\begin{proof}[Proof of Proposition~\ref{prop: online-k-th-episode-bound}]

Applying Theorem~\ref{thm: cost-decomposition} to $\pi^{(k)}$ under $\{\hat\bF_t^{(k)}\}_{t=1}^T$, we have
\begin{align*}
   \max_{x\in [M]_+} \Delta(x;\pi^{(k)})\leq  \mathsf{Err}_T^{(k)}:= \sum\nolimits_{t=1}^T\big\langle\bv_t^{(k)},\big[\Delta \bA_t^{(k)}(s_{t+1}^{(k)})\big]\hat\bD_{t+1}^{(k)}+\Delta\bc_t^{(k)}\big\rangle.
\end{align*}
We may assume $\cN^{(k)}_{\mathsf{cov}}<\infty$, since otherwise $q_{t,j}^{(k)}=0$ for all $t,j$, hence $\bv_t^{(k)}=\bm0$.

For the second term in $\mathsf{Err}_T^{(k)}$, under~$\cE^{(k)}$, we have
 $0\le\langle\bv_t^{(k)},\Delta\bc_t^{(k)}\rangle\lesssim(h_\infty+b_\infty)\sum_{j}q_{t,j}^{(k)}\eta_{t,j}$, summing over $t$ and applying Cauchy--Schwarz with $\sum_{t,j}q_{t,j}^{(k)}\le MT$, we arrive at
\begin{align*}
    \sum_{t=1}^T\langle\bv_t^{(k)},\Delta\bc_t^{(k)}\rangle\lesssim(h_\infty+b_\infty)MT\sqrt{\log(KMT/\delta)/\cN^{(k)}_{\mathsf{cov}}}.
\end{align*}

It remains to control the first term, for which applying summation-by-parts gives \begin{align}
    &\sum_{t=1}^T\big\langle\bv_t^{(k)},\Delta \bA_t^{(k)}(s_{t+1}^{(k)})\hat\bD_{t+1}^{(k)}\rangle  = \sum_{t=1}^T \sum_{j = 0}^{M-1} v_{tj}^{(k)} \big[\bG^{(k)}_{t}\bH^{(k)}_t \big]_j\\
    &= \sum_{t=1}^T \sum_{j = 0}^{M-1} \sum_{m = s_{t+1}^{(k)}}^j v_{tj}^{(k)} \big( F_t(j-m) - \hat{F}^{(k)}_t(j-m)   \big) [\bH^{(k)}_t]_m \\
    &\overset{\mathrm{(a)}}\lesssim \sum_{t=1}^T \sum_{j = 0}^{M-1} \sum_{m = s_{t+1}^{(k)}}^j v_{tj}^{(k)} \bigg(\sqrt{\hat{F}^{(k)}_t(j-m)(1-\hat{F}^{(k)}_t(j-m) )} \eta_{t,j}  + \zeta_{t,j}\bigg) [\bH^{(k)}_t]_m  \label{eq: lcb-analysis-tmp3}
\end{align}
where in (a), we have used $N_{t,j-m}^{(k)}\ge N_{t,j}^{(k)}$ and~Lemma~\ref{lem: ucb-variance-transfer} with $p = \hat{F}_t^{(k)}(j-m), u = F_t(j-m)$ for each $t,j,m$.
Now, applying Lemma~\ref{lem: peeling-bound} with $p_m=\hat F_t^{(k)}(j-m)$, $h_m=[\bH_t^{(k)}]_m,$
\begin{align*}
    a=a_\star^{(k)}:=\frac18\min\big\{1,\ \log(KMT/\delta)/(\cN^{(k)}_{\mathsf{cov}}\wedge K)\big\},
\end{align*}
and defining $\cV_{t,j}^{(k)}$ exactly as in the proof of Theorem~\ref{thm: censored-sample-complexity} with $\bF^{\mathrm{UCB}}_t,\bH^{\mathrm{UCB}}_t$ replaced by $\hat\bF_t^{(k)},\bH_t^{(k)}$, we obtain, by $\sum_m[\bH_t^{(k)}]_m=[(\hat\bD_{t+1}^{(k)})_+]_j$, 
\begin{align*}
    \sum_{t,j}v_{tj}^{(k)}[\bG_t^{(k)}\bH_t^{(k)}]_j\lesssim\sqrt{\frac{MT\log(KMT/\delta)\log(e/a_\star^{(k)})}{\cN^{(k)}_{\mathsf{cov}}}}\sqrt{\sum_{t,j}v_{tj}^{(k)}\cV_{t,j}^{(k)}}+\sum_{t,j}v_{tj}^{(k)}\zeta_{t,j}[(\hat\bD_{t+1}^{(k)})_+]_j+\cR_{\mathrm{peel}}^{(k)},
\end{align*}
with $\cR_{\mathrm{peel}}^{(k)}:=\sqrt{a_\star^{(k)}}\sum_{t,j}v_{tj}^{(k)}\eta_{t,j}[(\hat\bD_{t+1}^{(k)})_+]_j$.

For the last two terms, by~\eqref{eq: D-recursion}, $[(\hat\bD_{t+1}^{(k)})_+]_j=\cO((h_\infty+b_\infty)T)$, so, exactly as in the offline proof, we get 
$    \cR_{\mathrm{peel}}^{(k)}+\sum_{t,j}v_{tj}^{(k)}\zeta_{t,j}[(\hat\bD_{t+1}^{(k)})_+]_j\lesssim{(h_\infty+b_\infty)MT^2\log(KMT/\delta)}/{\cN^{(k)}_{\mathsf{cov}}}$.

It remains to control $\sum_{t,j}v_{tj}^{(k)}\cV_{t,j}^{(k)}$. Applying the identity in Lemma~\ref{lem: V-identity} obtains \begin{align*}
    \cV_{t,j}^{(k)} = \big[\bA^{(k)}_t(s_{t+1}^{(k)}) \bgamma_{t+1}^{(k)} \big]_j -\big( \big[\bA^{(k)}_t(s_{t+1}^{(k)}) (\hat\bD_{t+1}^{(k)})_+ \big]_j\big)^2,\quad \bgamma_t^{(k)}:=[(\hat\bD_t^{(k)})_+]^2.
\end{align*}
By the elementary bound
$\cV_{t,j}^{(k)}\le[\bA_t^{(k)}(s_{t+1}^{(k)})\bgamma_{t+1}^{(k)}]_j-[\bgamma_t^{(k)}]_j+2h_\infty[(\hat\bD_t^{(k)})_+]_j,$
we arrive at the following decomposition as in the offline proof:
\begin{align*}
    \sum_{t=1}^T\sum_{j=0}^{M-1} v_{tj}^{(k)}\cV_{t,j}^{(k)}&\le \sum_{t=1}^T\underbrace{\langle\bv_t^{(k)},\big[\bA^{(k)}_t(s_{t+1}^{(k)}) - \bA_t(s_{t+1}^{(k)}) \big] \bgamma_{t+1}^{(k)}\rangle}_{:= \cK_{t,1}^{(k)}} +\sum_{t=1}^T\underbrace{\langle\bv_t^{(k)},\bA_t(s_{t+1}^{(k)}) \bgamma_{t+1}^{(k)}-\bgamma_t^{(k)}\rangle}_{:= \cK_{t,2}^{(k)}}\\
    &+\sum_{t=1}^T\underbrace{2h_\infty\langle\bv_t^{(k)},(\hat\bD_t^{(k)})_+\rangle}_{:= \cK_{t,3}^{(k)}}.
\end{align*}
It remains to bound the summation over $\cK_{t}^{(k)}$ terms as in the offline setting. Notably, due to the lower-biased CDF estimator, the control of $\cK_{t,1}^{(k)}$ is much easier, while the control of $\cK_{t,2}^{(k)}$ follows the same recursive argument and the control of $\cK_{t,3}^{(k)}$ has additional difficulty.

\noindent\textbf{Bounding the summation of $\cK_{t,1}^{(k)}$.} For each $j\in [M-1]_+,$ by summation-by-parts, we have \begin{align*}
 \big(\big[ \bA^{(k)}_t(s_{t+1}^{(k)}) - \bA_t(s_{t+1}^{(k)})\big]\bgamma_{t+1}^{(k)}\big)_j = \sum_{m=s_{t+1}^{(k)}}^j \big[\underbrace{\hat{F}_t^{(k)}(j-m)  - F_t(j-m)}_{\leq 0}\big] \big([\underbrace{(\hat{D}_{t+1}^{(k)})_+]_m^2 - [(\hat{D}_{t+1}^{(k)})_+]_{m-1}^2 }_{\geq 0}\big) \leq 0,
\end{align*}
thus $\sum_{t=1}^T \cK_{t,1}^{(k)}\leq 0.$

\noindent\textbf{Bounding the summation of $\cK_{t,2}^{(k)}$.} Set $\bgamma_{T+1}^{(k)}:=\bm0$. By the recursion $\bA_t(s_{t+1}^{(k)})^\top\bv_t^{(k)}=\bv_{t+1}^{(k)}-\bu_{t+1}^{(k)}$ of Proposition~\ref{prop: v-propagation}, 
    $\sum_{t=1}^T\cK_{t,2}^{(k)}=\sum_{t=1}^T\big\langle\bv_t^{(k)},\bA_t(s_{t+1}^{(k)})\bgamma_{t+1}^{(k)}-\bgamma_t^{(k)}\big\rangle=-\langle\bv_1^{(k)},\bgamma_1^{(k)}\rangle-\sum_{t=2}^T\langle\bu_t^{(k)},\bgamma_t^{(k)}\rangle\leq 0$,
where the inequality holds since $\bgamma_t^{(k)}\geq\bm0$ while $\bv_1^{(k)}=\bu_1^{(k)}\geq\bm0$ and $\bu_t^{(k)}\geq\bm0$, by $\sigma_t\geq 0$.

\noindent\textbf{Bounding the summation of $\cK_{t,3}^{(k)}$.}
Denote $\hat{q}_{t,j}^{(k)}:= \hat{\PP}^{(k)}(y_t^{(k)}> j \lvert x_1 = 0 )$,
where $\hat{\PP}^{(k)}(\cdot)$ is taken with respect to the demand distributions under $\{\hat{\bF}_t^{(k)}\}_{t=1}^T$, we have then for each $t\in [T].$
\begin{align*}
    \langle \bv_t^{(k)}, (\hat{\bD}_t^{(k)})_+\rangle
    &\leq \sum_{j = s_{t}^{(k)}}^{M-1} q_{t,j}^{(k)} \big[(\hat{\bD}_t^{(k)})_+\big]_j = \sum_{j = s_{t}^{(k)}}^{M-1} \hat{q}_{t,j}^{(k)} \big[(\hat{\bD}_t^{(k)})_+\big]_j + \sum_{j = s_{t}^{(k)}}^{M-1} \big(q_{t,j}^{(k)} - \hat{q}_{t,j}^{(k)}\big) \big[(\hat{\bD}_t^{(k)})_+\big]_j.
\end{align*}
For the first term, by $\pi^{(k)}$ is the optimal base-stock policy under the environment induced by $\{\hat{\bF}_t^{(k)}\}_{t=1}^T$, Proposition~\ref{prop: over-shooting-price} shows it is of order $\cO\big(M(h_\infty + b_\infty)\big)$. To control the second term, if we introduced the vector notations $\bw^{(k)}_t, \hat{\bw}^{(k)}_t , \bm{b}_t \in \RR^{M}$ as in the proof of Proposition~\ref{coro: uniform-convergence}: \begin{align*}
    [\bw_t^{(k)}]_j:= q_{t,j}^{(k)} ,\quad [\hat{\bw}_t^{(k)}]_j:= \hat{q}_{t,j}^{(k)},\quad [\bm{b}^{(k)}_t]_j:= \bm{1}\{j < s_{t}^{(k)}\}, \quad {\bw_{1}^{(k)} =\hat{\bw}_{1}^{(k)}= \bm{b}_{1}^{(k)}},
\end{align*}
then
$\sum_{j = s_{t}^{(k)}}^{M-1} \big(q_{t,j}^{(k)} - \hat{q}_{t,j}^{(k)}\big) \big[(\hat{\bD}_t^{(k)})_+\big]_j \leq \lVert  \bw_{t}^{(k)} - \hat{\bw}_{t}^{(k)}\rVert_1 \lVert \hat{\bD}_t^{(k)}\rVert_\infty\lesssim (h_\infty + b_\infty) T \lVert  \bw_{t}^{(k)} - \hat{\bw}_{t}^{(k)}\rVert_1$.
To control the $\lVert  \bw_{t}^{(k)} - \hat{\bw}_{t}^{(k)}\rVert_1$ term, noticing that the following recursion formula\begin{equation}
    \bw^{(k)}_{t+1} = \bA_{t}(s_{t+1}^{(k)})^\top \bw^{(k)}_{t}+ \bm{b}_{t+1}^{(k)},\quad \hat{\bw}^{(k)}_{t+1} = \bA^{(k)}_{t}(s_{t+1}^{(k)})^\top \hat{\bw}^{(k)}_{t} + \bm{b}_{t+1}^{(k)},\quad \forall t\in [T-1],
\end{equation}
gives $\bw_{t+1}^{(k)} - \hat{\bw}_{t+1}^{(k)} =\bA_t(s_{t+1}^{(k)})^\top( \bw_{t}^{(k)} - \hat{\bw}_{t}^{(k)})+\big(\bA_t(s_{t+1}^{(k)})-\bA_t^{(k)}(s_{t+1}^{(k)})\big)^\top\hat\bw_t^{(k)},\quad\forall t \in [T-1]$.

Iterating the error propagation with ${\bw_1^{(k)}-\hat\bw_1^{(k)}}=\bm0$ and the $\ell_1$ non-expansiveness of $\bA_t(s)^\top$ on $\RR_+^M$ (its row sums obey $\sum_j[\bA_t(s)]_{ij}=\sum_{m=0}^{i-s}\mu_{t,m}\le1$) gives, for every $t$,
\begin{align*}
    \big\lVert\bw_t^{(k)}-\hat\bw_t^{(k)}\big\rVert_1\le\sum_{\tau=1}^{t-1}\big\lVert\big(\bA_\tau(s_{\tau+1}^{(k)})-\bA_\tau^{(k)}(s_{\tau+1}^{(k)})\big)^\top\hat\bw_\tau^{(k)}\big\rVert_1.
\end{align*}
By the summation-by-parts argument in the proof of Corollary~\ref{coro: uniform-convergence}, $\big[\big(\bA_\tau(s_{\tau+1}^{(k)})-\bA_\tau^{(k)}(s_{\tau+1}^{(k)})\big)\bm1\big]_i=F_\tau(i-s_{\tau+1}^{(k)})-\hat F_\tau^{(k)}(i-s_{\tau+1}^{(k)})\ge0$, and the fact that $\hat\bw_\tau^{(k)}$ is non-negative and non-increasing(thus, then $\big(\bA_\tau(s_{\tau+1}^{(k)})-\bA_\tau^{(k)}(s_{\tau+1}^{(k)})\big)^\top\hat\bw_\tau^{(k)}\ge\bm0$ entry-wise). Using $\hat q_{\tau,i}^{(k)}\le q_{\tau,i}^{(k)}$, $N_{\tau,i-s}^{(k)}\ge N_{\tau,i}^{(k)}$ and Lemma~\ref{lem: ucb-variance-transfer},
\begin{align*}
    \big\lVert\big(\bA_\tau(s_{\tau+1}^{(k)})-\bA_\tau^{(k)}(s_{\tau+1}^{(k)})\big)^\top\hat\bw_\tau^{(k)}\big\rVert_1=\sum_{i}\hat q_{\tau,i}^{(k)}\big[F_\tau(i-s_{\tau+1}^{(k)})-\hat F_\tau^{(k)}(i-s_{\tau+1}^{(k)})\big]\lesssim\sum_{i}q_{\tau,i}^{(k)}\big(\eta_{\tau,i}+\zeta_{\tau,i}\big).
\end{align*}
Hence $\lVert\bw_t^{(k)}-\hat\bw_t^{(k)}\rVert_1\lesssim\sum_{\tau=1}^T\sum_i q_{\tau,i}^{(k)}(\eta_{\tau,i}+\zeta_{\tau,i})$ for every $t$, and by Cauchy--Schwarz inequality, 
\begin{align*}
    \sum_{\tau=1}^T\sum_{i=0}^{M-1}q_{\tau,i}^{(k)}\big(\eta_{\tau,i}+\zeta_{\tau,i}\big)\lesssim MT\Bigg(\sqrt{\frac{\log(KMT/\delta)}{\cN^{(k)}_{\mathsf{cov}}}}+\frac{\log(KMT/\delta)}{\cN^{(k)}_{\mathsf{cov}}}\Bigg).
\end{align*}
Summing the second-term bound over $t$ and adding the first term $\sum_t\cO\big(M(h_\infty+b_\infty)\big)=\cO\big(M(h_\infty+b_\infty)T\big)$, we conclude
\begin{align*}
    \sum_{t=1}^T\cK_{t,3}^{(k)}=2h_\infty\sum_{t=1}^T\langle\bv_t^{(k)},(\hat\bD_t^{(k)})_+\rangle\lesssim(h_\infty+b_\infty)^2MT\Bigg(1+T^2\bigg(\sqrt{\frac{\log(KMT/\delta)}{\cN^{(k)}_{\mathsf{cov}}}}+\frac{\log(KMT/\delta)}{\cN^{(k)}_{\mathsf{cov}}}\bigg)\Bigg).
\end{align*}

\noindent\textbf{Putting all together.} Since $\sum_{t=1}^T\cK_{t,1}^{(k)}\le0$ and $\sum_{t=1}^T\cK_{t,2}^{(k)}\le0$, we have
\begin{align*}
    \sum_{t,j}v_{tj}^{(k)}\cV_{t,j}^{(k)}\le\sum_{t=1}^T\cK_{t,3}^{(k)}\lesssim (h_\infty+b_\infty)^2MT\bigg(1+T^2\Big(\sqrt{\frac{\log(KMT/\delta)}{\cN^{(k)}_{\mathsf{cov}}}}+\frac{\log(KMT/\delta)}{\cN^{(k)}_{\mathsf{cov}}}\Big)\bigg).
\end{align*}
 Substituting the above bound on $\sum_{t,j}v_{tj}^{(k)}\cV_{t,j}^{(k)}$ into the estimate for $\sum_{t,j}v_{tj}^{(k)}[\bG_t^{(k)}\bH_t^{(k)}]_j$ and applying the elementary inequality
$T^2\rho^{3/4} \lesssim T\sqrt{\rho} + T^3\rho$
for $\rho = \log(e/a_\star^{(k)})\log(KMT/\delta)/\cN^{(k)}_{\mathsf{cov}},$ 
we obtain
\begin{align*}
    \mathsf{Err}_T^{(k)}\lesssim(h_\infty+b_\infty)M\bigg(T\sqrt{\frac{\log(KMT/\delta)\log(e/a_\star^{(k)})}{\cN^{(k)}_{\mathsf{cov}}}}+\frac{T^3\log(KMT/\delta)\log(e/a_\star^{(k)})}{\cN^{(k)}_{\mathsf{cov}}}\bigg).
\end{align*}
Finally, since $a_\star^{(k)}\ge1/(8(\cN^{(k)}_{\mathsf{cov}}\wedge K))$ we have $\log(e/a_\star^{(k)})\le\log(8e(\cN^{(k)}_{\mathsf{cov}}\wedge K))\le\log(8eK)$, so that
\begin{align*}
    \mathsf{Err}_T^{(k)}\lesssim(h_\infty+b_\infty)M\bigg(T\sqrt{\frac{\log(KMT/\delta){\log (eK)}}{\cN^{(k)}_{\mathsf{cov}}}}+\frac{T^3\log(KMT/\delta){\log (eK)}}{\cN^{(k)}_{\mathsf{cov}}}\bigg),
\end{align*}
which is the bound stated in Proposition~\ref{prop: online-k-th-episode-bound}.
\end{proof}

\subsection{Proof of Proposition~\ref{prop: online-self-coverage} and Theorem~\ref{thm: online-regret}}

In this section, we provide the proof of Proposition~\ref{prop: online-self-coverage} and Theorem~\ref{thm: online-regret}.
To prove Proposition~\ref{prop: online-self-coverage}, we first introduce the following dichotomy lemma,

\begin{lemma}[Online coverage dichotomy]\label{lem: N-bound-online}
With probability at least $1-\delta$, for every $k=2,\dots,K$, $t\in[T], j\in[M-1]_+,$ at least one of the following holds:
\begin{align}\label{eq: online-dichotomy}
    q_{t,j}^{(k)} \leq \frac{100\log(KMT/\delta)}{k-1}, \qquad\text{or}\qquad N_{t,j}^{(k)} \geq \frac{k-1}{2}\,q_{t,j}^{(k)}.
\end{align}
\end{lemma}

\begin{proof}[Proof of Lemma~\ref{lem: N-bound-online}]
Define $\cF_\ell$-measurable random variable  $\widetilde q_{t,j}^{(\ell)}:=\PP(y_t^{(\ell)} > j \lvert x_1 = x_1^{(\ell)})$ and the binary variable $Z_{t,\ell}:= \bm{1}\{y_t^{(\ell)} > j\}$. Then, 
$N_{t,j}^{(k)} = \sum_{\ell=1}^{k-1} Z_{t,\ell},
\EE [Z_{t,\ell}\lvert \cF_{\ell}] = \widetilde q_{t,j}^{(\ell)},
\operatorname{Var}(Z_{t,\ell}\lvert \cF_{\ell})\leq \widetilde q_{t,j}^{(\ell)}.
$
Applying Freedman's inequality {with confidence level $\delta/2$}, it holds with probability at least {$1-\delta/2$} that
\begin{align}\label{eq: N-concentration}
  \big\lvert N_{t,j}^{(k)} - \widetilde Q_{t,j}^{(k)} \big\rvert
  \leq 4\sqrt{\widetilde Q_{t,j}^{(k)}\log(KMT/\delta) } +2\log (KMT/\delta),
  \qquad
  \widetilde Q_{t,j}^{(k)}:=\sum_{\ell = 1}^{k-1} \widetilde q_{t,j}^{(\ell)},
\end{align}
uniformly for all $t\in [T], j \in [M-1]_+, k \in [K].$

Denote $\bar{\cE}$ the event that~\eqref{eq: N-concentration} holds. Applying Proposition~\ref{prop: biased-F} at confidence level $\delta/2$ so that $\cE^{(K)}$ holds with probability at least $1-\delta/2$, a union bound gives $\PP(\bar{\cE} \cap \cE^{(K)} ) \geq 1-\delta$. Under $\cE^{(K)}\cap \bar{\cE}$, we have for any $k\geq 2, t\in [T], j \in [M-1]_+:$

\noindent Case~(a): If $q_{t,j}^{(k)} \leq 100\log(KMT/\delta)/(k-1),$ the first alternative holds in the dichotomy.

\noindent Case~(b): Otherwise, by $\widetilde q_{t,j}^{(\ell)}\ge q_{t,j}^{(\ell)}$ and~\eqref{eq: J-monotone-online}, 
$\widetilde Q_{t,j}^{(k)}\geq\sum_{\ell=1}^{k-1}q_{t,j}^{(\ell)}\geq (k-1)q_{t,j}^{(k)} \geq 100 \log(KMT/\delta),$
this gives 
$    N_{t,j}^{(k)}
    \geq \widetilde Q_{t,j}^{(k)}-4\sqrt{\widetilde Q_{t,j}^{(k)}\log(KMT/\delta)}-2\log(KMT/\delta)
    \geq \frac{\widetilde Q_{t,j}^{(k)}}2
    \geq \frac{k-1}{2}q_{t,j}^{(k)},$ 
    as desired.
\end{proof}

\begin{proof}[Proof of Proposition~\ref{prop: online-self-coverage}]
By Lemma~\ref{lem: N-bound-online}, it holds with probability at least $1-\delta,$ for all $k=2,\dots,K,\,t\in [T],j\in [M-1]_+,$ \begin{align*}
    \frac{q^{(k)}_{t,j}}{N_{t,j}^{(k)} \vee 1 } \leq \begin{cases}
        \dfrac{100\log(KMT/\delta)}{k-1}, &  \text{ if } q^{(k)}_{t,j}\leq \dfrac{100\log(KMT/\delta)}{k-1},\\
        \dfrac{2}{k-1}, &  \text{ otherwise by Lemma~\ref{lem: N-bound-online} }.
    \end{cases}
\end{align*} 
Taking average over such element-wise upper bound then gives 
   $ \cN_{\mathsf{cov}}^{(k)} \geq \frac{k-1}{100\log(KMT/\delta)},$
as desired.
\end{proof}

\begin{proof}[Proof of Theorem~\ref{thm: online-regret}]

Combining Proposition~\ref{prop: online-k-th-episode-bound} and Proposition~\ref{prop: online-self-coverage}, we have with probability at least $1-\delta,$ for every $k\geq 2$ the per-episode cost gap satisfies \begin{align*}
    \max_{x\in [M]_+} \Delta(x;\pi^{(k)}) \lesssim (h_\infty + b_\infty) M \bigg({\frac{T\log(KMT/\delta) \sqrt{{\log (K)}}}{\sqrt {k-1}}} + {\frac{T^3\log^2(KMT/\delta) {\log (K)}}{k-1}}\bigg).
\end{align*}
On the other hand, for $k = 1$ we use the trivial per-episode bound $\max_{x\in[M]_+}\Delta(x;\pi^{(1)})\leq (h_\infty+b_\infty)MT$. 
Summing over episodes via $\sum_{k=2}^K (k-1)^{-1/2}\leq 2\sqrt K$ and $\sum_{k=2}^K (k-1)^{-1}\leq 1+\log K$, and absorbing the trivial upper bound $(h_\infty+b_\infty)MT$ for $k=1$ case yields the desired result.
\end{proof}

\subsection{Proof of Theorem~\ref{thm: online-regret-stationary}}

The proof of Theorem~\ref{thm: online-regret-stationary} relies on an aggregated version of Proposition~\ref{prop: online-k-th-episode-bound} and Lemma~\ref{lem: N-bound-online}. More precisely, define the aggregated per-episode sample size 
    $N_{\mathsf{agg},j}^{(k)}:= \sum_{t=1}^T N_{t,j}^{(k)},\quad \forall j \in [M-1]_+$
and the per-episode effective sample size 
       $ \cN^{(k)}_{\mathsf{agg}} := \bigg(\dfrac{1}{{MT}}{ \sum_{t=1}^T\sum_{j=0}^{M-1} \dfrac{\PP(y_t^{(k)} > j \lvert x_1 = 0)}{N_{\mathsf{agg},j}^{(k)}\vee 1}}\bigg)^{-1}$ for $k\in [K].$
We have the following two results:

\begin{proposition}\label{prop: online-k-th-episode-bound-stationary}
For every $k\in [K]$, suppose the event $\cE^{(k)}$ holds, then there exists an absolute constant $c_0$ so that the $k$-th episode cost gap in aggregated version of Algorithm~\ref{alg: censored-online} satisfies \begin{align*}
    \max_{x\in [M]_+} \Delta(x;\pi^{(k)})\leq   c_0 (h_\infty + b_\infty) M \cdot \bigg[T \sqrt{\frac{\log(KMT/\delta){\log(KT)}}{\cN^{(k)}_{\mathsf{agg}}}} + \frac{T^3\log(KMT/\delta){\log(KT)}}{\cN^{(k)}_{\mathsf{agg}}}\bigg].
\end{align*}
\end{proposition}

\begin{proposition}[Self-generated coverage, aggregated]\label{prop: online-self-coverage-stationary}
There is an absolute constant $c_0>0$ such that, with probability at least $1-\delta$, the dataset generated by the aggregated version of Algorithm~\ref{alg: censored-online} satisfies
$    \cN^{(k)}_{\mathsf{agg}} \geq \frac{c_0\,{(k-1)}T}{\log(KMT/\delta)},  \forall k \in [K]$.

\end{proposition}
Proposition~\ref{prop: online-k-th-episode-bound-stationary} holds directly by substituting $N_{\mathsf{agg},j}^{(k)}$ to $N_{t,j}^{(k)}$, with the count upper bound $K$ replaced by $KT$ accordingly, in the proof of Proposition~\ref{prop: online-k-th-episode-bound}, so we only provide the proof of Proposition~\ref{prop: online-self-coverage-stationary} here.

\begin{proof}[Proof of Proposition~\ref{prop: online-self-coverage-stationary}]

At $k=1$ the claimed bound is vacuous, as its right-hand side is $0$; we therefore fix $k\geq 2$ in what follows and write $Q_j^{(k)}:=\sum_{t=1}^T q_{t,j}^{(k)}$ for the aggregated coverage probability of coordinate $j$ in episode $k$. Under this notation, the aggregated effective sample size satisfies
$    \big(\cN^{(k)}_{\mathsf{agg}}\big)^{-1} = \frac{1}{MT}\sum_{j=0}^{M-1} \frac{Q_j^{(k)}}{N_{\mathsf{agg},j}^{(k)}\vee 1}$.

For each episode $\ell$ and coordinate $j$, define $X_{\ell,j}:= \sum_{t=1}^T \bm{1}\{y_t^{(\ell)} > j\}$, so that $N_{\mathsf{agg},j}^{(k)} = \sum_{\ell=1}^{k-1} X_{\ell,j}$. Conditional on $\cF_{\ell}$, we have $X_{\ell,j} \in [0,T]$ and
    $\widetilde Q_j^{(\ell)}
    :=\EE[X_{\ell,j}\lvert \cF_{\ell}]
    \geq Q_j^{(\ell)},
    \Var(X_{\ell,j}\lvert\cF_{\ell})
    \leq \EE[X_{\ell,j}^2\lvert\cF_{\ell}]
    \leq T\,\widetilde Q_j^{(\ell)}$.
Applying Freedman's inequality to the martingale difference sequence $\{(X_{\ell,j}-\widetilde Q_j^{(\ell)})\}_{\ell\in [k-1]}$ {with confidence level $\delta/2$} and taking a union bound over $j\in[M-1]_+, k\in[K]$, it holds with probability at least {$1-\delta/2$} that
\begin{align}\label{eq: N-concentration-agg}
  \Big\lvert N_{\mathsf{agg},j}^{(k)} - \widetilde{\cQ}_j^{(k)} \Big\rvert
  &\leq 4\sqrt{T\widetilde{\cQ}_j^{(k)}\log(KMT/\delta) } +2T\log (KMT/\delta),
   \forall j \in [M-1]_+, k \in [K]
\end{align}
for  $\widetilde{\cQ}_j^{(k)}
  :=\sum_{\ell = 1}^{k-1}\widetilde Q_j^{(\ell)}$.
Denote by $\bar{\cE}$ the event that~\eqref{eq: N-concentration-agg} holds. {Applying Proposition~\ref{prop: biased-F} at confidence level $\delta/2$ so that $\cE^{(K)}$ holds with probability at least $1-\delta/2$, a union bound gives $\PP(\bar{\cE} \cap \cE^{(K)} ) \geq 1-\delta.$} We then proceed the analysis under $\cE^{(K)}\cap \bar{\cE}$ in the remained proof.

According to the value of $Q_j^{(k)}$, we partition the set of inventory level coordinates as
\begin{align*}
    \cJ_{1}^{(k)} := \Big\{ j\in[M-1]_+ : Q_j^{(k)}\leq \tfrac{100\,T\log(KMT/\delta)}{k-1} \Big\},\qquad \cJ_{2}^{(k)} := [M-1]_+\setminus \cJ_{1}^{(k)}.
\end{align*}
For $j\in\cJ_{1}^{(k)}$, the first episode is run under $\hat{\bF}_t^{(1)}\equiv\bm0$ and hence orders up to $M$ in every period, so $\bm{1}\{y_t^{(1)}>j\}=1$ for all $t\in[T]$; thus $X_{1,j}=T$ and $N_{\mathsf{agg},j}^{(k)}\geq T$ for every $k\geq 2$. Consequently
\begin{align*}
    \sum_{j\in\cJ_{1}^{(k)}} \frac{Q_j^{(k)}}{N_{\mathsf{agg},j}^{(k)}\vee 1} \leq \sum_{j\in\cJ_{1}^{(k)}} \frac{Q_j^{(k)}}{T}\leq \frac{100\,M\log(KMT/\delta)}{k-1}.
\end{align*}
For $j\in\cJ_{2}^{(k)}$, we have $Q_j^{(k)}> 100\,T\log(KMT/\delta)/(k-1)$, so by the monotonicity~\eqref{eq: J-monotone-online},
\begin{align*}
    \widetilde{\cQ}_j^{(k)}
    \geq\sum_{\ell=1}^{k-1} Q_j^{(\ell)}
    \geq (k-1)Q_j^{(k)}
    \geq 100\,T\log(KMT/\delta).
\end{align*}
Plugging this into~\eqref{eq: N-concentration-agg}  yields
$N_{\mathsf{agg},j}^{(k)}
    \geq \widetilde{\cQ}_j^{(k)}-4\sqrt{T\widetilde{\cQ}_j^{(k)}\log(KMT/\delta)}-2T\log(KMT/\delta)
    \geq\frac{\widetilde{\cQ}_j^{(k)}}2
    \geq \frac{k-1}{2}\,Q_j^{(k)}$,
and therefore $\sum_{j\in\cJ_{2}^{(k)}} \frac{Q_j^{(k)}}{N_{\mathsf{agg},j}^{(k)}\vee 1} \leq \sum_{j\in\cJ_{2}^{(k)}} \frac{2}{k-1} \leq \frac{2M}{k-1}$.

Combining above bounds together, we have then
\begin{align*}
    \big(\cN^{(k)}_{\mathsf{agg}}\big)^{-1} = \frac{1}{MT}\sum_{j=0}^{M-1} \frac{Q_j^{(k)}}{N_{\mathsf{agg},j}^{(k)}\vee 1} \leq \frac{1}{MT}\Big(\frac{100\,M\log(KMT/\delta)}{k-1}+\frac{2M}{k-1}\Big)\lesssim \frac{\log(KMT/\delta)}{(k-1)T},
\end{align*}
as desired.
\end{proof}

\begin{proof}[Proof of Theorem~\ref{thm: online-regret-stationary}]

Combining Proposition~\ref{prop: online-k-th-episode-bound-stationary} and Proposition~\ref{prop: online-self-coverage-stationary}, we have with probability at least $1-\delta,$ for every $k\geq 2$ the per-episode cost gap satisfies \begin{align*}
    \max_{x\in [M]_+} \Delta(x;\pi^{(k)}) \lesssim (h_\infty + b_\infty) M \bigg({\frac{\sqrt T\log(KMT/\delta) \sqrt{\log (KT)}}{\sqrt {k-1}}} + {\frac{T^2\log^2(KMT/\delta) \log(KT)}{k-1}}\bigg).
\end{align*}
As in the proof of Theorem~\ref{thm: online-regret}, the first episode carries no historical data, for which we use the trivial bound $\max_{x\in[M]_+}\Delta(x;\pi^{(1)})\leq (h_\infty+b_\infty)MT$. Summing the per-episode bound over $k\geq 2$ via $\sum_{k=2}^K (k-1)^{-1/2}\leq 2\sqrt K$ and $\sum_{k=2}^K (k-1)^{-1}\leq 1+\log K$, and absorbing the first-episode term into the second term, yields the desired result.
\end{proof}

\section{Proof of Lower Bound Results}

In this section, we collect the proof of lower bound results, including Lemma~\ref{lem: policy-evaluation}, Theorem~\ref{thm: lower-bound-independent}, Theorem~\ref{thm: lower-bound-independent-identical}. In particular, the online lower bound results Theorem~\ref{thm: online-lower-bound} are direct corollary of Theorem~\ref{thm: lower-bound-independent} and Theorem~\ref{thm: lower-bound-independent-identical} in uncensored setting thus are omitted.

\noindent\textbf{Reduction to an arbitrary coverage fraction.}
All information-theoretic proofs below are carried out for the fully informative case $p=1$. To obtain a coverage fraction $p\in(0,1]$, set the censoring levels to $M$ (fully revealing) on $n_p:=\lceil Np\rceil$ trajectories and to $0$ on the rest; each padded observation equals $\bar d=\min\{0,d\}=0$ deterministically under every instance, thus contributing neither KL divergence nor coverage counts. The experiment is therefore the $p=1$ experiment with $N$ replaced by $n_p$, with $N_{t,j} = n_p \geq \lfloor Np\rfloor$ and $N_{\mathsf{agg},j} = n_p T \geq \lfloor NTp\rfloor$ for every $t,j$, so $\cN^\star \geq \lfloor Np\rfloor$ and $\cN_{\mathsf{agg}}^\star \geq \lfloor NTp\rfloor$.

\noindent\textbf{Notations and information theoretic tools.} Most of our arguments in this section can be seen as a problem-specific construction of problem instances to meet the requirement of general information theoretic limit results \citep{yu1997assouad,Tsybakov2009lowerbound,wainwright2019high, huang2025lower}. More precisely, in the proof of all results, we use the following notations: 
\begin{enumerate}
    \item We use $\Theta$ to denote the space of instances, for each $\theta \in \Theta$, we let $P_\theta$ be the corresponding joint distribution of all observations when the underlying environment is given by $\theta$.
    \item We use $\Ab$ to denote the space of possible decisions taken by the learner.
    \item For each $a\in \Ab$ and $\theta \in \Theta$, we set $L(\theta,a)$ as the expected cost incurred by taking decision $a$ under environment $\theta$.
\end{enumerate}
It is worth noting that some later proofs require modifications to this typical choice due to technical reasons, which we specify on a case-by-case basis.

With such general notations, we provide the basic information theoretic tools for controlling the minimax lower bound $\min_{\cA}\max_{\theta\in \Theta}\EE_{\cD \sim \theta}[L(\theta,\cA(\cD))],$ where the minimum is taken over all possible data-driven algorithms taking $\cD$ as an input and outputs the decision variables.

The first result is a standard lower bound of the binary testing problem:
\begin{lemma}\label{lem: LeCam}
Let $P$ and $Q$ be two probability measures on a common measurable space, and let $\phi$ be any measurable test taking values in $\{0,1\}$. Then
$    P(\phi = 1) + Q(\phi = 0) \geq 1 - \mathrm{TV}(P,Q)$.
Consequently, by Pinsker's inequality,
    $\inf_{\phi} \max \big\{P(\phi = 1), Q(\phi = 0)\big\} \geq \frac{1-\sqrt{D(P\lVert Q)/2}}{2}$.
\end{lemma}

\begin{lemma}[Assouad's lemma]\label{lem: assouad} 
Given a $m$-Hamming cube $\cZ^m:= \{-1,1\}^m$ and a class of problem instances $\Theta = \{\theta_z: z\in \cZ^m\}$ indexed by $\cZ^m$, suppose it holds that \begin{align}\label{eq: assouad-separation-appendix}
    L(\theta_{z}, a) + L(\theta_{z'}, a)\geq \sum_{j=1}^m \Delta_j \bm{1}\{z_j \neq z_j'\},\quad \forall z,z' \in \cZ^m,
\end{align}
for some $\{\Delta_j\}_{j=1}^m$ sequence.
Then for
  $  P_{j,+}:= \frac{1}{2^{m-1}}\sum_{z\in \cZ^m: z_j = 1} P_{\theta_z},     P_{j,-}:= \frac{1}{2^{m-1}}\sum_{z\in \cZ^m: z_j = -1} P_{\theta_z},$ \begin{align}\label{eq: assouad-lower-bound-appendix}
    \min_{\cA}\max_{\theta \in \Theta} \E_{\cD \sim \theta}[L(\theta,\cA(\cD))] \geq \frac{1}{4}\sum_{j=1}^m \Delta_j \big( 1 - \lVert P_{j,+} - P_{j,-}\rVert_{\mathrm{TV}} \big) .
\end{align}
\end{lemma}

\begin{lemma}\label{lem: Bernoulli-KL}
    For two Bernoulli distributions $\mathrm{Bern}(p_1),\mathrm{Bern}(p_2)$ with $0< p_1,p_2 <1$, it holds that $\text{D}(\mathrm{Bern}(p_1)\lVert \mathrm{Bern}(p_2)) \leq \frac{(p_1-p_2)^2}{p_2(1-p_2)}$.
\end{lemma}

\subsection{Proof of Lemma~\ref{lem: policy-evaluation}}\label{appendix-sec: evaluation-lb}

In the proof of Lemma~\ref{lem: policy-evaluation}, we replicate the construction of the single-action MDP instances in Appendix~B.1 of \citet{ren2021nearly} and present it in the language of a two-support demand inventory problem, with the general framework introduced before.

\noindent\textbf{Construction of the instance collection $\Theta$.} With the shared cost coefficients $h_t \equiv c_\infty, b_t \equiv 0,$ we construct two instances with different demand distributions $P_t^+ \equiv P^+, P_t^- \equiv P^-$, where \begin{align*}
    & P^\pm (d = 0) = p_\pm,  P^\pm(d = M) = 1-p_\pm,  p_+ = 1-\frac{c_+}{T}, p_- = 1-\frac{c_-}{T}, c_- = c_+-\sqrt{\frac{c_+(T-c_+)}{2NT}},
\end{align*}
for some $c_+>0$ to be determined later. This corresponds to the general case $\Theta = \{p_+,p_-\}$.

\noindent\textbf{Action $\Ab$ and loss $L$.} In this policy evaluation setting, the "action" here is simply an estimation of the optimal policy value, thus $\Ab = \mathbb{R}$ and if we denote the optimal cost as $C^{\star,+},C^{\star,-}$ under the environment induced by $p_+,p_-$ respectively, we then set $
    L(p_\pm,a) = \lvert C^{\star,\pm}(M) - a\rvert.$

\noindent\textbf{Separation condition.} It is easy to see that in both instances, the optimal policy is given by $s_t^{\star,\pm} \equiv 0$. As a consequence, 
$    C^{\star,\pm}(M) = c_\infty M\sum_{k=1}^T p_\pm^k = c_\infty M p_{\pm} \frac{1 - p_\pm^{T}}{1-p_\pm}$.
By our selection of $p_\pm$ and $c_\pm$, we have then for any $a\in\Ab$,
\begin{align*}
    L(p_+,a)+L(p_-,a)
    &\geq C^{\star,-}(M)-C^{\star,+}(M) \geq c_\infty M\frac{1-(1+c_+)e^{-c_+}}{(1-p_+)^2}(p_--p_+)\\
    &=\frac{1-(1+c_+)e^{-c_+}}{c_+^2}\,c_\infty MT
    \sqrt{\frac{c_+(T-c_+)}{2NT}}.
\end{align*}
The second inequality follows from the mean value theorem and the monotonicity of the derivative, with the derivative at $p_+$ calculated as
$    \frac{\mathrm{d}}{\mathrm{d}p}\sum_{k=1}^Tp^k
    =\frac{1-(1+c_+)p_+^T}{(1-p_+)^2}
    \geq\frac{1-(1+c_+)e^{-c_+}}{(1-p_+)^2}$.

This verifies~\eqref{eq: assouad-separation-appendix} with
$    \Delta_1=\frac{1-(1+c_+)e^{-c_+}}{c_+^2}\,c_\infty MT
    \sqrt{\frac{c_+(T-c_+)}{2NT}}$.

\noindent\textbf{Putting all together.} Finally, noticing that the observation $\cD=\{d_t^k\}_{t,k=1}^{T,N}$ follows $(P^\pm)^{\otimes NT}$ under the environment determined by $p_\pm$, applying Lemma~\ref{lem: assouad} with $m=1$ leads to
\begin{align*}
    &\min_{\cA}\max\Big\{
    \EE_{\cD\sim(P^+)^{\otimes NT}}\big|C^{\star,+}(M)-\cA(\cD)\big|,
    \EE_{\cD\sim(P^-)^{\otimes NT}}\big|C^{\star,-}(M)-\cA(\cD)\big|
    \Big\}\\
    &\geq\frac{\Delta_1}{4}\left(1-
    \sqrt{\frac{NT}{2}D\big(\mathrm{Bern}(p_-)\lVert\mathrm{Bern}(p_+)\big)}\right).
\end{align*}
On the other hand, by Lemma~\ref{lem: Bernoulli-KL},
$    NT\cdot D\big(\mathrm{Bern}(p_-)\lVert\mathrm{Bern}(p_+)\big)
    \leq\frac{NT(c_+-c_-)^2}{c_+(T-c_+)}=\frac{1}{2}$,
selecting $c_+=1/2$ gives the desired $\Omega(c_\infty MT/\sqrt N)$  result.

\subsection{Proof of Theorem~\ref{thm: lower-bound-independent}}

In this section, we present the proof of Theorem~\ref{thm: lower-bound-independent} by Lemma~\ref{lem: assouad}.

\noindent\textbf{Construction of the instance collection $\Theta$.} To present our construction of the $T$-period instance, we first present our construction of a $3$-period sub-problem as the following:

Fix $\epsilon \in (0,\frac{1}{2T})$ to be determined later and consider two three-period instances $B_+$ and $B_-$ with shared cost 
$    b_1 = h_2 = h_3 = c_\infty, h_1 = b_2 = b_3 = 0,$
and different 3-period demand distributions $\bP^+,\bP^-$ over $\{0,M\}$, defined as:
\begin{align*}
    & P^{+}_1(d = M) = \frac{1+\epsilon}{2}, \quad P_1^{+}(d = 0)=\frac{1-\epsilon}{2}, \quad P^+_2( d = 0) = 1,\quad P_3^+(d = M) = 1,\\
    & P^{-}_1(d = M) = \frac{1-\epsilon}{2}, \quad P_1^{-}(d = 0)=\frac{1+\epsilon}{2}, \quad P^-_2( d = 0) = 1,\quad P_3^-(d = M) = 1.
\end{align*}

To construct a $T$-period instance from above, we write $T=3T'+r$ with $r\in\{0,1,2\}$, use the first $3T'$ periods to concatenate $T'$ three-period block above, and set the remaining $r$ periods to have deterministic zero demand $P_t(d=0)=1$ with cost coefficients $h_t = 0, b_t = c_\infty$. Since $h_t = 0$ and $d_t \equiv 0$, each such period incurs zero cost under \emph{every} policy ($h_t(y_t-d_t)_+ + b_t(d_t - y_t)_+ = c_\infty(0-y_t)_+ = 0$ as $y_t \geq 0$), so it leaves the sub-optimality gap unchanged while keeping $\min\{h_t,b_t\}=0, \max\{h_t,b_t\}=c_\infty$ for all $t\in [T]$. This gives the same $h_t,b_t$ pattern for all instances. And all possible combinations of $B_+,B_-$ across $T$ can be represented as a Hamming sequence $\bz=(z_0,\dots,z_{T'-1})\in \cZ^{T'}$, with
$$z_k = \begin{cases}
    +1 & \text{ if the $k$-th block in the sequence is } B_+,\\
    -1 & \text{ if the $k$-th block in the sequence is } B_-.
\end{cases} 
$$ 
This gives a natural index set $\Theta$ via $\cZ^{T'}.$

\noindent\textbf{Action space $\Ab$ and loss $L$.} Due to the multi-period nature of the problem formulation, we set the possible action space $\Ab$ as all policies $\pi$ that determine $T$-period order-up-to levels sequentially, possibly depending on all historical demand realizations $d_1,\dots,d_{t-1}$ at each $t$. Then for $\bz \in \cZ^{T'}$ and $\pi \in \Ab$, we set 
    $L(\pi, \bz):= C_{1}^{\pi, \bz} (0) - C_1^{\star, \bz}(0)$.

\noindent\textbf{Separation condition.} Given any fixed $\pi \in \Ab$, if we denote $q_t$ its order-up-to level at time $t$, we have then by construction,
\begin{align*}
    C_{1}^{\pi,\bz}(0) &= \EE_{\bz}\big[\sum_{t=1}^T h_t (q_t-d_t)_+ + b_t(d_t-q_t)_+  \big]\\
    &=\EE_{\bz}\bigg[\sum_{k=0}^{T'-1}  c_{\infty}\big[(d_{3k+1}-q_{3k+1})_+ + (q_{3k+2} - d_{3k+2})_+ + (q_{3k+3}-d_{3k+3})_+\big]  \bigg]\\
    &=\EE_{\bz}\bigg[\sum_{k=0}^{T'-1}  c_{\infty}\big[(d_{3k+1}-q_{3k+1})_+ + q_{3k+2}  + (q_{3k+3}-M)_+\big]  \bigg]\\
    &\overset{\text{(i)}}{=} \EE_{\bz}\bigg[\sum_{k=0}^{T'-1}  c_{\infty}\big[\frac{1+z_k \epsilon}{2}(M-q_{3k+1})_+ + q_{3k+2}  + (q_{3k+3}-M)_+\big]  \bigg]\\
    &\overset{\text{(ii)}}{\geq} \EE_{\bz}\bigg[\sum_{k=0}^{T'-1}  c_{\infty}\big[\frac{1+z_k \epsilon}{2}(M-q_{3k+1})_+ + (q_{3k+1} - d_{3k+1})_+  + (q_{3k+1}-d_{3k+1}-M)_+\big]  \bigg]\\
    &\overset{\text{(iii)}}{\geq} \EE_{\bz}\bigg[\sum_{k=0}^{T'-1}  c_{\infty}\big[\frac{1+z_k \epsilon}{2}(M-q_{3k+1})_+ + \frac{1+z_k \epsilon}{2}(q_{3k+1} - M)_+ + \frac{1-z_k\epsilon}{2} q_{3k+1}\big]\bigg].
\end{align*}
Here, (i) is by the definition of $P^\pm,$ (ii) is by $q_{3k+3}\geq q_{3k+2} \geq (q_{3k+1}-d_{3k+1})_+,$ and in (iii) the last period cost is dropped. In particular, the right-hand-side of inequality (ii) can be achieved by an auxiliary policy $\tilde{\pi}$ induced by $\pi$ as the following: At each time $t$, with the historical demand realizations $d_1,\dots,d_{t-1},$ the $\tilde{\pi}$ policy first calculate the order-up-to level $q_t$ given by $\pi$, then modify it to \begin{align*}
    \tilde{q}_{3k+1}:= \begin{cases}
        \min\{M,q_t\}, &\text{ if } t = 3k+1 \text{ for some } 0 \leq k \leq T'-1,\\
        x_t, &\text{ if } t = 3k+2 \text{ for some } 0 \leq k \leq T'-1,\\
        M, &\text{ if } t = 3k+3 \text{ for some } 0 \leq k \leq T'-1.
    \end{cases}
\end{align*}
It can be checked directly by calculation that 
\begin{align*}
    C_{1}^{\tilde{\pi},\bz}(0)  &= \EE_{\bz}\bigg[\sum_{k=0}^{T'-1}  c_{\infty}\big[\frac{1+z_k \epsilon}{2}(M-\tilde{q}_{3k+1})  + \frac{1-z_k\epsilon}{2} \tilde{q}_{3k+1}\big]\bigg]\\
    &\leq \EE_{\bz}\bigg[\sum_{k=0}^{T'-1}  c_{\infty}\big[\frac{1+z_k \epsilon}{2}(M-q_{3k+1})_+ + \frac{1+z_k \epsilon}{2}(q_{3k+1} - M)_+ + \frac{1-z_k\epsilon}{2} q_{3k+1}\big]\bigg] \leq  C_{1}^{\pi,\bz}(0).
\end{align*}

On the other hand, by construction, we have an optimal policy under $\bz$ given by the base-stock policy $\{s^{\star,\bz}_{t}\}_{t = 1}^T$ with
\begin{align*}
   s_{3k+1}^{\star,\bz} = \begin{cases}
    M, &\text{if } z_{k} = 1,\\
    0, &\text{if } z_{k} = -1
   \end{cases} = \frac{1+z_k}{2}M, \quad s^{\star,\bz}_{3k+2} = 0,\quad  s^{\star,\bz }_{3k+3} = M, \quad \forall 0\leq k \leq T'-1.
\end{align*}
And then,
   $ C^{\star, \bz} = \EE_{\bz}\big[\sum_{t=1}^T h_t(q_t^\star - d_t)_+ +b_t(d_t-q_t^\star)_+\big]= \sum_{k=0}^{T'-1} c_\infty\big[ \frac{1+z_k\epsilon}{2} \frac{1- z_k}{2} M + \frac{1-z_k\epsilon}{2} \frac{1+z_k}{2} M  \big]$.

As a result, \begin{align*}
    & L(\pi, \bz) =  C_{1}^{\pi,\bz}(0) - C_{1}^{\star,\bz}(0) \geq C_{1}^{\tilde{\pi},\bz}(0) - C_{1}^{\star,\bz}(0)\\
    &\geq \EE_{\bz}\bigg[\sum_{k=0}^{T'-1}  c_{\infty}\bigg(\big[\frac{1+z_k \epsilon}{2}(M-\tilde{q}_{3k+1})  + \frac{1-z_k\epsilon}{2} \tilde{q}_{3k+1}\big] - \big[ \frac{1+z_k\epsilon}{2} \frac{1- z_k}{2} M + \frac{1-z_k\epsilon}{2} \frac{1+z_k}{2} M  \big]\bigg)\bigg]\\
    &\geq \EE_{\bz}\bigg[\sum_{k=0}^{T'-1}  c_{\infty}\bigg(\big[\frac{1+z_k \epsilon}{2}( \frac{1+z_k}{2} M -\tilde{q}_{3k+1})  + \frac{1-z_k\epsilon}{2} ( \tilde{q}_{3k+1} - \frac{1+z_k}{2} M)\big] \bigg]\\
    &= \EE_{\bz}\bigg[\sum_{k=0}^{T'-1} {c_\infty z_k\epsilon} (\frac{1+z_k}{2} M - \tilde{q}_{3k+1}) \bigg]. 
\end{align*}
Now, for any distinct $\bz,\bz'$, adding and subtracting the same expectation gives
\begin{align*}
    L(\pi,\bz)+L(\pi,\bz')
    &\geq \EE_{\bz}\bigg[\sum_{k=0}^{T'-1}c_\infty z_k\epsilon
    \left(\frac{1+z_k}{2}M-\tilde q_{3k+1}\right)\bigg]
    +\EE_{\bz'}\bigg[\sum_{k=0}^{T'-1}c_\infty z'_k\epsilon
    \left(\frac{1+z'_k}{2}M-\tilde q_{3k+1}\right)\bigg]\\
    &=\EE_{\bz}\bigg[\sum_{k=0}^{T'-1}c_\infty\epsilon
    \bigg\{z_k\left(\frac{1+z_k}{2}M-\tilde q_{3k+1}\right)
    +z'_k\left(\frac{1+z'_k}{2}M-\tilde q_{3k+1}\right)\bigg\}\bigg]\\
    &\qquad+\EE_{\bz'}\bigg[\sum_{k=0}^{T'-1}c_\infty z'_k\epsilon
    \left(\frac{1+z'_k}{2}M-\tilde q_{3k+1}\right)\bigg]
    -\EE_{\bz}\bigg[\sum_{k=0}^{T'-1}c_\infty z'_k\epsilon
    \left(\frac{1+z'_k}{2}M-\tilde q_{3k+1}\right)\bigg]\\
    &\geq c_\infty M\epsilon N_d
    -c_\infty M\epsilon T'\lVert\bP_{\bz}-\bP_{\bz'}\rVert_{\mathrm{TV}} \geq c_\infty M\epsilon N_d
    -c_\infty M\epsilon T'\sqrt{D(\bP_{\bz}\Vert\bP_{\bz'})}\\
    &\overset{\mathrm{(i)}}{\geq}c_\infty M\epsilon N_d
    -3c_\infty M\epsilon^2T'\sqrt{N_d} \overset{\mathrm{(ii)}}{\geq}c_\infty M\epsilon(1-3\epsilon T')N_d
    \geq\frac{c_\infty M\epsilon}{2}N_d,
\end{align*}
where $N_d:=\sum_{k=0}^{T'-1}\bm{1}\{z_k\neq z_k'\}$. The TV step uses $0\leq\tilde q_{3k+1}\leq M$: the paired terms contribute at least $c_\infty M\epsilon N_d$, and the remaining random variable lies in $[0,c_\infty M\epsilon T']$. Bound (i) follows from
\begin{align*}
    D(\bP_{\bz}\Vert\bP_{\bz'})
    &\leq\sum_{k=0}^{T'-1}\bm{1}\{z_k\neq z'_k\}
    \max\{D(P^+\Vert P^-),D(P^-\Vert P^+)\}
    \leq6\epsilon^2N_d,
\end{align*}
and (ii) uses $\sqrt{N_d}\leq N_d$ for $N_d\geq 1$. Finally,
$\epsilon T'\leq1/6$ follows from $\epsilon<1/(2T)$ and $T\geq3T'$, so
$1-3\epsilon T'\geq1/2$. This verifies~\eqref{eq: assouad-separation-appendix} with
$\Delta_j\equiv c_\infty M\epsilon/2$.

\noindent\textbf{Putting all together.} With the specified $\Theta,\Ab,L$ and the verified separation condition, now we are ready to apply Lemma~\ref{lem: assouad} to obtain the desired lower bound. We apply the lemma with $m = T'$, indexing the Hamming coordinates by $k \in \{0, 1, \dots, T'-1\}$ to match the 0-based block convention used throughout the construction. Noticing that in this setting, we have $\cD = \{d_t^k\}_{t=1,k=1}^{T,N}$ follows the distribution $\bP_{\bz}^{\otimes N}$ under the environment determined by $\bz$, thus~\eqref{eq: assouad-lower-bound-appendix} yields
\begin{align*}
    \min_{\cA} \max_{\bz} \EE_{\cD\sim \bP_{\bz}^{\otimes N}}[L(\bz,\cA(\cD))]&\geq \frac{c_\infty M \epsilon}{8}\sum_{k=0}^{T'-1} \big(1 - \sqrt{\frac{1}{2}D(P_{k,+}\lVert P_{k,-})}\big).
\end{align*}
Noticing that for each $k \in \{0, 1, \dots, T'-1\},$ by the joint convexity of KL divergence, we have \begin{align*}
    D(P_{k,+} \lVert P_{k,-})\leq \frac{1}{2^{T'-1}} \sum\nolimits_{\bz \in \cZ^{T'}: z_k = 1} D(P^{\bz\otimes N} \lVert P^{\bz^{(k)}\otimes N}),
\end{align*}
with $\bz^{(k)}$ defined as the flip-one element induced by $\bz$ at its $k$-th coordinate. Then, by Lemma~\ref{lem: Bernoulli-KL},
    $D(P^{\bz \otimes N} \lVert P^{\bz^{(k)}\otimes N}) \leq 8N \epsilon^2, \forall \bz \in \cZ^{T'}$.
Selecting $\epsilon = \frac{1}{8\sqrt{N}}$, we have
$\sum_{k=0}^{T'-1} \big(1 - \sqrt{\frac{1}{2}D(P_{k,+}\lVert P_{k,-})}\big) \geq T'/4$, and this gives the desired $\Omega(c_\infty M T/\sqrt{N})$ lower bound.

\subsection{Proof of Theorem~\ref{thm: lower-bound-independent-identical}}

\noindent\textbf{Construction of the instance collection $\Theta$.} For the given $c_\infty,$ set $h_t = b_t = c_\infty$ for all $t\in [T]$ and pick $\epsilon \in (0,1/4) $ to be determined later. We construct two 
distributions supported over $\{0,M\}$ as
\begin{align*}
    P^+(d=0) = \frac{1}{2}+\epsilon,\quad P^+(d=M) = \frac{1}{2}-\epsilon, \quad P^-(d=0) = \frac{1}{2}-\epsilon,\quad P^-(d=M) = \frac{1}{2}+\epsilon.
\end{align*}
and set the corresponding two instances with shared cost coefficients $h_t = b_t = c_\infty$ and stationary distributions $\bP^+, \bP^-$ so that $\bP^\pm_t \equiv P^\pm$ for all $t\in [T]$. With such construction, we can represent two instances using binary class $\Theta = \{+, -\}$, where the parameter $z\in \Theta$ corresponds to the environment with distribution $\bP^{z}.$

\noindent\textbf{Action space $\Ab$ and loss $L$.} We set $\Ab$ as all policies $\pi$ that determine $T$-period order-up-to levels sequentially, possibly depending on all historical demand realizations $d_1,\dots,d_{t-1}$ at each $t$. For $z \in \Theta$ and $\pi \in \Ab$, the loss is the sub-optimality gap $L(\pi,z) := C_1^{\pi,z}(0) - C_1^{\star,z}(0)$.

\noindent\textbf{Separation condition.} By construction, the optimal base-stock levels are $s_t^{\star,+} \equiv 0$ under $\bP^+$ and $s_t^{\star,-} \equiv M$ under $\bP^-$. For any sequential policy $\pi$ with order-up-to level $q_t$ at period $t$, a direct calculation gives
$    L(\pi,+) \geq 2\epsilon c_\infty\EE^+\Big[\sum\nolimits_{t=1}^T q_t\Big],$
$L(\pi,-) \geq 2\epsilon c_\infty\EE^-\Big[\sum\nolimits_{t=1}^T (M-q_t)_+\Big]$
(both with equality whenever $q_t \leq M$ for all $t$). Together with $q_t \geq M - (M-q_t)_+$, this yields
\begin{align*}
    & L(\pi,\bP^+) + L(\pi,\bP^-) \geq 2\epsilon c_\infty\Big(\EE^+\big[\sum\nolimits_{t=1}^T \big(M - (M-q_t)_+\big)\big] + \EE^-\big[\sum\nolimits_{t=1}^T (M-q_t)_+\big]\Big)\\
    &\geq 2\epsilon c_\infty\Big(MT - \big\lvert \EE^+\big[\sum\nolimits_{t=1}^T (M-q_t)_+\big] - \EE^-\big[\sum\nolimits_{t=1}^T (M-q_t)_+\big] \big\rvert\Big)\\
    &\geq 2\epsilon c_\infty MT\Big(1 - \big\lVert (P^+)^{\otimes T} - (P^-)^{\otimes T} \big\rVert_\mathrm{TV}\Big) \geq 2\epsilon c_\infty MT\bigg(1 - \sqrt{\frac{T D(P^+\lVert P^-)}{2}}\bigg).
\end{align*}
To bound the KL divergence term, applying Lemma~\ref{lem: Bernoulli-KL} gives
$    D(P^+\lVert P^-)
    \le  \frac{4\epsilon^2}{(\frac{1}{2}-\epsilon)(\frac{1}{2}+\epsilon)}\leq 64\epsilon^2$
by $\epsilon \le \frac{1}{4}$.
Thus the separation lower bound
$    L(\pi,\bP^+) + L(\pi,\bP^-) \geq \epsilon c_\infty MT 
$
holds when
$    \epsilon \leq  \sqrt{\frac{1}{128T}}$. This verifies~\eqref{eq: assouad-separation-appendix} with $\Delta_1 = \epsilon c_\infty MT$.

\noindent\textbf{Putting all together.} Now 
selecting
   $ \epsilon = \sqrt{\frac{1}{128NT}}$
and applying Lemma~\ref{lem: assouad} with $m = 1$ gives
\begin{align*}
    &\min_{\cA}\max_{\bP \in \{\bP^+,\bP^-\}}\EE_{\cD\sim\bP^{\otimes NT}}\big[C_1^{\cA(\cD)}(0) - C_1^\star(0)\big] \geq \frac{\Delta_1}{4}\Big(1 - \big\lVert (P^+)^{\otimes NT} - (P^-)^{\otimes NT}\big\rVert_{\mathrm{TV}}\Big)\\
    &\qquad\geq \frac{c_\infty MT\epsilon}{8}  = \Omega\left(M c_\infty \sqrt{\frac{T}{N}}\right),
\end{align*}
as desired.